\documentclass{article}
\usepackage{iclr2027_conference,times}
\usepackage{pifont}
\usepackage{amsmath,amssymb,amsthm}
\usepackage{tikz}
\usepackage{listings}
\usepackage{graphicx}
\usepackage{algpseudocode}
\usepackage{mathtools}
\usepackage{color}
\usepackage[utf8]{inputenc}
\usepackage{booktabs}
\usepackage{multirow}
\usepackage{makecell}
\usepackage[table]{xcolor}
\DeclareUnicodeCharacter{FF0C}{,}
\definecolor{grey}{gray}{0.5}
\usepackage{verbatim}
\usepackage[most]{tcolorbox}
\usepackage{multicol}
\usepackage{enumerate}
\usepackage{wrapfig}
\usepackage{caption}
\usepackage{float}
\usetikzlibrary{shapes.geometric, arrows.meta, positioning, calc, fit}
\usepackage[ruled, linesnumbered]{algorithm2e}
\usepackage{tabularx}
\usepackage{booktabs}

\newtheorem{definition}{Definition}
\newtheorem{theorem}{Theorem}

\newtheorem{lemma}{Lemma}
\newtheorem{remark}{Remark}

\newtcolorbox[auto counter, number within=section]{bluebox}{
    enhanced,
    colback=white,
    colframe=black,
    colbacktitle=black,
    coltitle=white,
    fonttitle=\bfseries\sffamily,
    attach boxed title to top left={xshift=3mm, yshift=-2mm},
    boxed title style={sharp corners, size=small},
    arc=1.5mm,
    top=1mm,
    bottom=1mm,
    before skip=1mm,
    after skip=1mm,
    before upper={
        \setlength{\parskip}{0pt}
        \setlength{\abovedisplayskip}{3pt}
        \setlength{\belowdisplayskip}{3pt}
        \setlength{\abovedisplayshortskip}{2pt}
        \setlength{\belowdisplayshortskip}{2pt}
    }
}

\newtcolorbox[auto counter, number within=section]{yellowbox}{
    enhanced,
    colback=white,
    colframe=black,
    colbacktitle=black,
    coltitle=white,
    fonttitle=\bfseries\sffamily,
    attach boxed title to top left={xshift=3mm, yshift=-2mm},
    boxed title style={sharp corners, size=small},
    arc=1.5mm,
}

\newtcolorbox[auto counter, number within=section]{orangebox}{
    enhanced,
    colback=white,
    colframe=black,
    colbacktitle=black,
    coltitle=white,
    fonttitle=\bfseries\sffamily,
    attach boxed title to top left={xshift=3mm, yshift=-2mm},
    boxed title style={sharp corners, size=small},
    arc=1.5mm,
}

\newtcolorbox{myremark}{
    blanker, breakable,
    left=3mm, before skip=10pt, after skip=10pt,
    borderline west={1.5pt}{1.5pt}{grey} 
}

\definecolor{linkblue}{RGB}{0,70,140}

\usepackage{amsmath,amsfonts,bm}

\def\eqref#1{equation~\ref{#1}}

\def\1{\bm{1}}

\DeclareMathAlphabet{\mathsfit}{\encodingdefault}{\sfdefault}{m}{sl}
\SetMathAlphabet{\mathsfit}{bold}{\encodingdefault}{\sfdefault}{bx}{n}

\usepackage{hyperref}
\hypersetup{
    colorlinks=true,
    linkcolor=blue,
    citecolor=linkblue,
    urlcolor=cyan
}
\usepackage{url}

\title{All Work And No Play Makes Jack a Dull Boy: Understanding and Preventing Catastrophic Strategy Collapse in RLVR}

\author{Qiyuan Huang\textsuperscript{\dag} \\
Peking University \\
\And
Tianshi Xu\textsuperscript{\dag} \\
Peking University \\
\And Meng Li$^*$ \\
Peking University \\
}

\newcommand{\sota}[1]{%
  \begingroup
  \setlength{\fboxsep}{1.5pt}%
  \colorbox{blue!5}{\textbf{#1}}%
  \endgroup
}

\begin{document}
\begingroup
\renewcommand{\thefootnote}{\fnsymbol{footnote}}
\footnotetext[2]{These authors contributed equally to this work. \quad $^*$Corresponding author: meng.li@pku.edu.cn}
\endgroup
\iclrfinalcopy

\maketitle
\lhead{}
\begin{abstract}
During post-training of large language models (LLMs) with Reinforcement Learning with Verifiable Rewards (RLVR), GRPO-style algorithms can exhibit severe late-stage collapse. Prompt-based probing reveals that this is not benign strategic pruning, but a harmful contraction of effective strategy capacity that makes distinct reasoning strategies increasingly inaccessible. To characterize this phenomenon, we define strategies through trajectory-level policy-update interactions and develop a unified theoretical framework combining optimization dynamics and information theory. We prove that major RLVR objectives progressively concentrate probability mass onto a single strategy, while sustaining nontrivial task accuracy requires a minimum strategy capacity. The conflict between these two results provides a mechanistic explanation for catastrophic collapse. We further derive the {Mirrored Entanglement Index (MEI)} as a lightweight online warning signal. To prevent collapse, we propose \textbf{Mesh Learning}, which exposes multiple reasoning strategies and prevents any single strategy from dominating optimization. Across AIME26, AIME25, MATH-500, GPQA, and LiveCodeBench, Mesh Learning consistently outperforms strong baselines across Qwen and Phi model families, with gains of up to 13.4 pp and 11.5 pp, respectively. These results establish strategy preservation as a key principle for stable RLVR. Code is available at \href{https://github.com/Ayanami-0123/Open-Mesh-Learning}{GitHub}.
\end{abstract}

\begin{figure}[h]
\begin{center}
\includegraphics[width=1.0\linewidth]{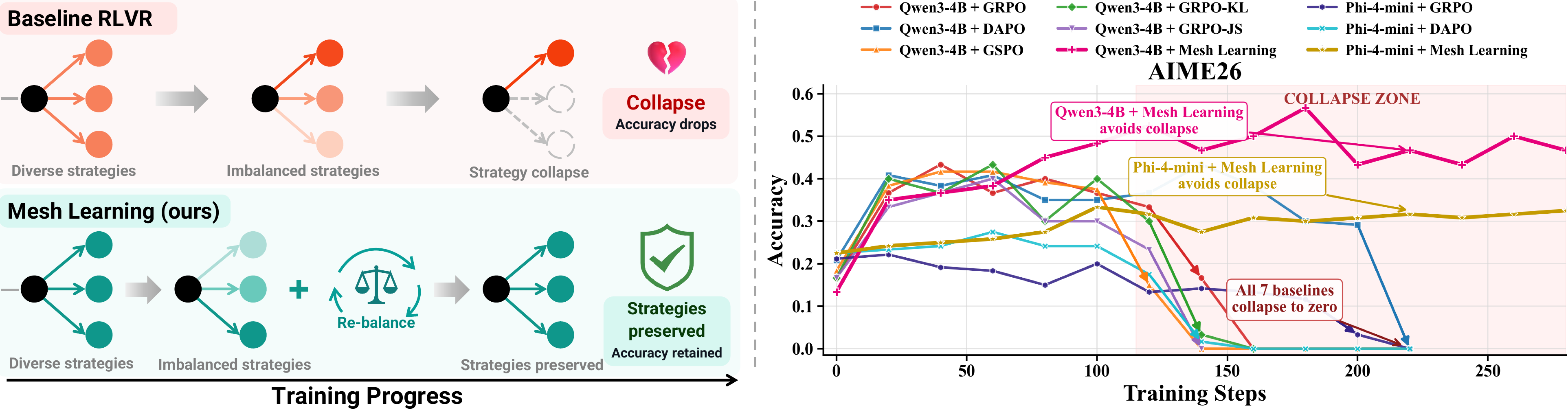}
\end{center}
\caption{
\textbf{Left:} Baseline RLVR progressively collapses to a single strategy, while \textsc{Mesh Learning} preserves strategy diversity through rebalancing.
\textbf{Right:} On AIME26, all baselines eventually collapse, whereas \textsc{Mesh Learning} remains stable and accurate across two model families.}\label{fig:cover}
\end{figure}

\section{{Introduction}}

In recent years, Reinforcement Learning with Verifiable Rewards (RLVR) has emerged as a powerful paradigm for improving the complex reasoning capabilities of Large Language Models (LLMs), particularly in mathematics and coding~\citep{guo2025deepseekr1,chen2025minimax,xu2026deepseek,team2026kimi,zeng2026glm,xu2026cleaner}. 
Driven by verifiable outcome supervision and scalable critic-free policy optimization, RLVR has been rapidly adopted in both academia and industry, giving rise to a range of optimization frameworks, including GRPO~\citep{guo2025deepseekr1}, DAPO~\citep{yu2025dapo}, and GSPO~\citep{zheng2025groupspo}, as well as regularization schemes based on KL divergence~\citep{shao2024DeepSeekMath} and JS divergence~\citep{li2026choice_js}. 
Together, these developments have established RLVR as a core post-training paradigm for modern reasoning models.

Despite its effectiveness in improving reasoning performance, RLVR can exhibit severe late-stage instability. Prior studies have reported similar collapse behaviors~\citep {guo2025deepseekr1,zheng2025groupspo,gao2025sapo,li2026choice_js}, and we observe the same phenomenon systematically in our experiments. As shown in Fig.~\ref{fig:cover}, accuracy drops abruptly after a period of seemingly successful optimization across two model families, Qwen~\citep{yang2025qwen3} and Phi~\citep{abdin2025phi}, and five representative training recipes, including GRPO, DAPO, GSPO, GRPO+KL, and GRPO+JS. We call this abrupt degradation \textbf{collapse}. Its recurrence across prior studies, model families, algorithms, and regularization schemes suggests a shared limitation of current RLVR optimization, yet its underlying mechanism remains poorly understood.

To characterize this behavior across RLVR algorithms, we develop a unified theory grounded in the optimization dynamics of policy updates. Since individual trajectories rarely recur exactly and are too fine-grained to describe reasoning behavior, we introduce \emph{strategies} as groups of trajectories that evolve coherently under parameter updates, and derive their distribution updates from the underlying trajectory-level updates. Within this framework, we explain collapse by contradiction through two central theorems: \textbf{\underline{1)}} If RLVR objectives (GRPO, DAPO, GSPO) persist in a non-collapse regime, the optimization inherently drives the policy toward a single dominant strategy with high probability; \textbf{\underline{2)}} achieving nontrivial task accuracy requires the model to retain a minimum strategy capacity. Together, these results reveal a fundamental conflict: RLVR optimization progressively contracts strategy capacity, while successful reasoning requires sufficient capacity to remain available. This contradiction provides a mechanistic explanation for the catastrophic collapse observed in practice.

This mechanistic view also suggests what to monitor before accuracy collapses: the latent concentration of strategy dynamics. Since this transition is difficult to observe from standard training signals, we derive the \textbf{Mirrored Entanglement Index (MEI)} from the geometry of RLVR updates as a lightweight online proxy for strategy collapse, requiring no auxiliary training and little computational overhead. Calibrating MEI on pre-RLVR models yields a \(3\sigma\) alarm threshold for detecting departure from the non-collapsed regime before the accuracy cliff. Across model families, RLVR algorithms, and regularization schemes, MEI rises above the threshold ahead of the eventual accuracy cliff, consistently tracking the theoretically predicted transition. Independent strategy-level analysis further confirms that this transition corresponds to the loss of distinct reasoning strategies, supporting strategy collapse as a shared mechanism underlying late-stage RLVR failure.

Building on the theoretical characterization and online measurability of strategy collapse, we propose \textbf{Mesh Learning}, a strategy-preserving RLVR framework that prevents individual strategies from dominating optimization and preserves multiple viable reasoning strategies throughout training. It consists of two complementary components: \textbf{\underline{1)} Coach Prompting (CP)}, which augments each query with strategy prefixes that elicit distinct solution approaches and ensure broad strategy exposure; and \textbf{\underline{2)} Strategy-Balancing Regularization}, which balances the learning progress across strategies to prevent any single strategy from overwhelming the others. In this way, CP constructs a diverse strategy mesh, while the regularizer preserves it throughout optimization. Across AIME26, AIME25, MATH-500~\citep{hendrycks2021measuring,lightman2023lets}, GPQA~\citep{rein2024gpqa}, and LiveCodeBench (v6)~\citep{jain2024livecodebench}, Mesh Learning consistently outperforms strong RLVR and divergence-regularized baselines across both Qwen and Phi families, with gains of up to 13.4 and 11.5 percentage points (pp), respectively. These results highlight strategy preservation as a key principle for stable and effective RLVR optimization.
In summary, our contributions are as follows:

\noindent\ding{182}\enspace
We establish a unified theory of RLVR dynamics, showing that inherent competition drives strategy concentration and explaining collapse through its conflict with the minimum strategy capacity required for successful reasoning. The analysis also yields MEI as a lightweight online indicator.

\noindent\ding{183}\enspace
We propose \textbf{Mesh Learning}, combining Coach Prompting and Strategy-Balancing Regularization to expose diverse reasoning strategies and prevent any strategy from dominating optimization.

\noindent\ding{184}\enspace
We provide extensive empirical validation across RLVR recipes and model families on AIME25/26, MATH-500, GPQA-Diamond, and LiveCodeBench, confirming the predicted collapse dynamics and the effectiveness of Mesh Learning.
\section{{From Trajectory to Strategy}}
Since the exact recurrence of an individual trajectory is exceedingly unlikely, analyzing isolated trajectories provides limited insight into training dynamics. To study these dynamics more effectively, we introduce the notion of a \textbf{\textit{strategy}}, which captures a group of trajectories that respond coherently to local policy updates, i.e., they tend to be reinforced or suppressed together.

\subsection{{The Concept of Strategy}}
Consider a prompt $Q$ and a policy $\pi_\theta$ parameterized by $\theta$. For a query $Q$, let its admissible trajectory space be $\mathcal T_Q=\{\tau_1,\tau_2,\ldots\}$. We first introduce the \textbf{Fisher Score} of a trajectory $\tau_i$:
$U_i=\nabla_\theta\log\pi_\theta(\tau_i\mid Q).$
Then consider an infinitesimal update induced by another trajectory $\tau_j$:
$\Delta\theta=\eta A_j U_j,$ where $\eta$ is the learning rate and $A_j$ denotes the scalar reward signal associated with $\tau_j$.
By first-order approximation, the log-likelihood of $\tau_i$ changes by
$\Delta\log\pi_{\theta}(\tau_i\mid Q)\approx\eta A_j U_i^\top U_j.$

Hence, we define the \textbf{trajectory coupling kernel} $K_{ij}\triangleq U_i^\top U_j$, which characterizes the first-order interaction between trajectories. For a positive update on $\tau_j$, $K_{ij}>0$ reinforces $\tau_i$, whereas $K_{ij}<0$ suppresses it, and $K_{ij}\approx0$ indicates that the two trajectories are locally decoupled.

Intuitively, trajectories belonging to the same strategy should not only be positively coupled, but also respond similarly to updates induced by other trajectories. This suggests a natural partitioning based on pairwise similarity $\sup_k |K_{ik}-K_{jk}|\le\epsilon$. However, pairwise similarity alone is not transitive and therefore does not directly induce a valid partition. To make this intuition precise, we define \textbf{strategy} in Def.~\ref{def:strategy} through the neighborhoods of the examined trajectories.

\begin{bluebox}
\begin{definition}[Strategy]\leavevmode\\
\label{def:strategy}
Given a prompt $Q$, policy $\pi_\theta$, and examined trajectories $\{\tau_1,\ldots,\tau_n\}$, let $K_{ij}=U_i^\top U_j$. 
For threshold $\epsilon$, define the closed neighborhood of $\tau_i$ as
\begin{equation}
    \mathcal{N}_i(\epsilon)\coloneqq \{\tau_i\}\cup
    \left\{\tau_j \mid \tau_j\neq\tau_i,\ K_{ij}>0,\ 
    \sup_{k}|K_{ik}-K_{jk}|\le\epsilon\cdot\sup_{a,b}|K_{ab}| \right\}.
\end{equation}
Here, $K_{ij}>0$ requires $\tau_i$ and $\tau_j$ to be positively coupled, while $\sup_{k}|K_{ik}-K_{jk}|\le\epsilon\cdot\sup_{a,b}|K_{ab}|$ further requires them to have similar coupling profiles with the examined trajectories.
We say $\tau_i$ and $\tau_j$ are dynamically equivalent, denoted by $\tau_i\sim_\epsilon\tau_j$, if
\begin{equation}
    \tau_i\sim_\epsilon\tau_j
    \iff
    \mathcal{N}_i(\epsilon)=\mathcal{N}_j(\epsilon).
\end{equation}
A \textbf{Strategy} is an equivalence class induced by $\sim_\epsilon$, with $\epsilon$ defined as the \textbf{sensitivity}.
\end{definition}
\end{bluebox}

\subsection{{The Semantic Proxy of Strategy}}
\label{subsec:2.2}
Def.~\ref{def:strategy} is deliberately optimizer-centric, grouping trajectories solely by their interactions under policy updates without semantic assumptions. We next investigate whether these strategies align with human-recognizable solution methods, enabling both interpretation and explicit control.

\begin{wrapfigure}[11]{r}[0em]{0.55\textwidth}
  \centering
  \vspace{-12pt}
  \includegraphics[width=1.0\linewidth]{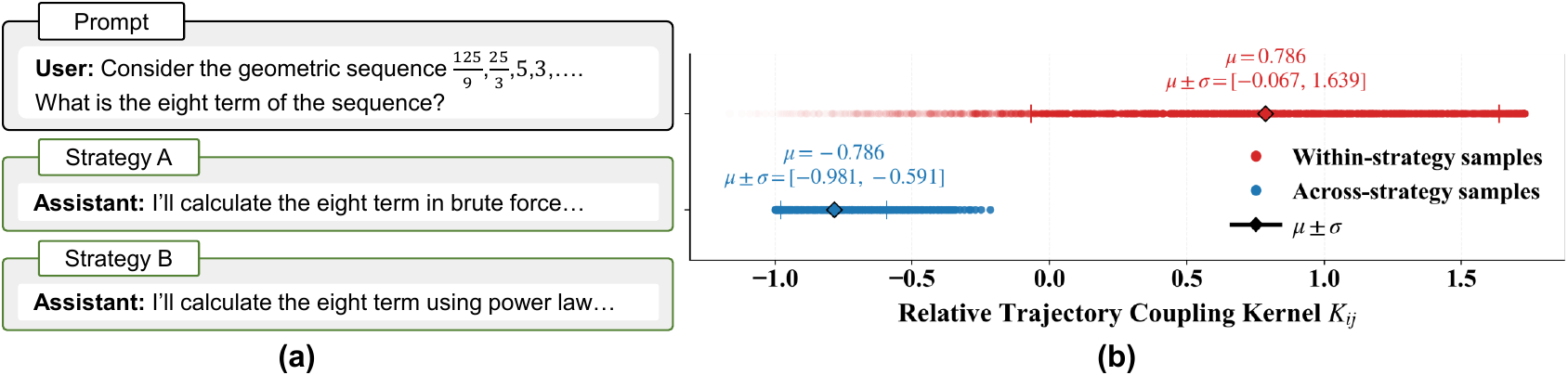}
  \caption{{Semantic interpretability of strategies.} 
(a) Distinct solution methods for the same problem. 
(b) On MATH-500, within-method trajectory pairs are substantially more coupled than cross-method pairs, supporting Def.~\ref{def:strategy}.}~\label{fig:strategy_example}
\end{wrapfigure}

Using correct trajectories generated by Qwen-2.5-7B-Instruct on MATH-500, we categorize their solution methods with DeepSeek-V4-Pro and compare the coupling kernel $K_{ij}$. With $\epsilon=0.05$, the resulting strategy partition empirically aligns with distinct solution methods. Fig.~\ref{fig:strategy_example} (a) illustrates that the same problem can admit semantically distinct strategies, while Fig.~\ref{fig:strategy_example} (b) shows at scale that within-method trajectory pairs are substantially more positively coupled than cross-method pairs, which are weakly or negatively coupled. This correspondence supports the semantic interpretability of our dynamically defined \textbf{strategies} and allows solution-method prefixes to serve as practical proxies for latent strategies. We use these prefixes for strategy elicitation in Sec.~\ref{subsec:3.2} and strategy-conditioned training in Sec.~\ref{subsec:data_orchestration}.
\section{{Theoretical Analysis }}\label{sec:3}
In this section, we prove the RLVR collapse mechanism by contradiction: \textbf{Theorem~\ref{theorem:concentration}} shows that if the non-collapse regime persists, GRPO, DAPO, and GSPO reduce strategy capacity; \textbf{Theorem~\ref{theorem:PowerLaw}} yields a minimum strategy capacity $\vert{}\mathcal S_{\epsilon,k}\vert{}$ for nontrivial task accuracy. These consequences eventually conflict, causing an accuracy cliff. Then, we derive a \textbf{lightweight online metric} that alerts collapse via a $3\sigma$ threshold from non-collapsed training. Finally, \textbf{Theorem~\ref{theorem:divergence_failure}} builds an effective strategy-level RL system, explains regularization failures, and motivates our prevention in the next section. The main notation used throughout our theoretical analysis is summarized in Table~\ref{tab:notation}.
\begin{table}[htbp]
\centering
\caption{Core notation used in Sec.~\ref{sec:3}.}
\label{tab:notation}
\footnotesize
\setlength{\tabcolsep}{3pt}
\renewcommand{\arraystretch}{0.8}
\resizebox{0.97\linewidth}{!}{%
\begin{tabularx}{\linewidth}{@{}cX|cX@{}}
\toprule
\textbf{Symbol} & \textbf{Meaning}
& \textbf{Symbol} & \textbf{Meaning} \\
\midrule

$Q,\mathcal D,\mathcal D',\mathcal H$ &
Query; test/training set; distribution &
$S_{\alpha,Q}^{(\epsilon)}$ &
Strategy $\alpha$ in $Q$'s step-0 $\epsilon$-partition\\

$k,\pi_k,\pi_{old,k}$ &
Step; step-$k$ current/old policy &
$\nu_{\alpha,Q,k}^{(\epsilon)}$ &
Probability mass of $S_{\alpha,Q}^{(\epsilon)}$ under $\pi_k$ \\
$Y(Q),\widehat{Y}(\tau)$ &
ground truth/ prediction of trajectory $\tau$ &
$P_{acc,k}$ &
Model's step-$k$ test set accuracy\\

$\epsilon,m_{\epsilon,Q}$ &
Sensitivity in Def.~\ref{def:strategy}; step-0 strategy amount for $Q$ under sensitivity $\epsilon$&
$|\mathcal{S}_{\epsilon,k}|$ &
Model's total number of usable strategies on $\mathcal D$ at step $k$
\\
\bottomrule
\end{tabularx}%
}
\end{table}
\vspace{-5pt}
\hypertarget{back:thm1}{}
\subsection{The Main Theorem of Strategy Collapse}
\textbf{If} RLVR remains in \textbf{non-collapse} regime, where test accuracy stays nontrivial, it induces competition among strategies. We show that this competition inherently concentrates their probability mass.

\begin{bluebox}
\begin{theorem}[The Main Theorem of Strategy Collapse]~\label{theorem:concentration}\leavevmode\\
Suppose $\mathcal D,\mathcal D'\overset{\mathrm{i.i.d.}}{\sim}\mathcal H$, with $|\mathcal D|,|\mathcal D'|\gg1$, and GRPO, GSPO, or DAPO is optimized
on $\mathcal D'$ by gradient descent with step-wise old-policy refresh
$\pi_{\mathrm{old},k}=\pi_{k-1}$.
For \(Q\in\mathcal{D}\), let $\eta$, $|\tau|$, and $\epsilon_{\mathrm{clip}}$
denote the learning rate, trajectory length, and clipping ratio. Define
$T_c\triangleq\lfloor\eta^{-1}\rfloor$. 
If $P_{acc,k}\in(0,1)$, for all $k\le T_c,$ and 
$\eta\ll\epsilon_{\mathrm{clip}}<1\ll|\tau|\ll\eta^{-1}\ll|\tau|^3,$
then
\begin{equation}
\begin{gathered}
    \forall\epsilon>0,\epsilon\ll1,\zeta,\delta\in(\eta,1),\zeta^2\gg\eta\quad
    \exists K\in\mathbb{Z}^{+}
    ,K\ll |\tau|<T_c,\\
    \forall k\in\mathbb{Z}^{+}
    ,K<k\le T_c,\quad
    \mathbb{P}\!\left(
        \max_{1\le\alpha\le m_{\epsilon,Q}}
        \nu_{\alpha,Q,k}^{(\epsilon)}
        \ge 1-\zeta
    \right)
    \ge 1-\delta.
\end{gathered}
\end{equation}
The probability is taken over the stochasticity of RLVR training on
$\mathcal{D}'$.
\end{theorem}
\end{bluebox}
\textbf{Theorem~\ref{theorem:concentration}} shows that, after burn-in time $K$, RLVR concentrates at least $1-\zeta$ of the strategy mass for $Q$ on a single strategy throughout $K<k\le T_c$ with probability at least $1-\delta$, thereby \textbf{reducing effective strategy capacity}. The scale condition $\eta\ll\epsilon_{\mathrm{clip}}<1\ll|\tau|\ll\eta^{-1}\ll|\tau|^3$ matches standard RLVR configurations and experiments (Sec.~\ref{subsec:exp_setup}); the proof is in Appendix~\hyperlink{proof:thm1}{\ref*{app:theorem1_proof}}. Since directly tracking $\nu_{\alpha,Q,k}^{(\epsilon)}$ is impractical, Sec.~\ref{subsec:mei} derives MEI as a lightweight observable.

\hypertarget{back:thm2}{}
\subsection{Power Law for Strategy Capacity Lower Bound}
\label{subsec:3.2}

Assuming the non-collapse regime persists, \textbf{Theorem~\ref{theorem:concentration}} implies concentration. We further derive a lower bound on usable strategy capacity for nontrivial accuracy, thus contradicting this assumption.
\begin{bluebox}
\begin{theorem}[Power Law for Strategy Capacity Lower Bound]\leavevmode\\
\label{theorem:PowerLaw}
For $\mathcal D,\mathcal D'\overset{\mathrm{i.i.d.}}{\sim}\mathcal H,$ $|\mathcal D|=N$. Suppose $
\mathbb P_{Q\sim\mathcal D}\!\left[\exists Q'\neq Q,\exists\tau\in S^{(\epsilon)}_{\cdot,Q}:\widehat Y(\tau)=Y(Q')\right]\le\gamma\ll1$, requiring a diverse answer space.
Fix $\rho,\epsilon\in(0,1)$ such that  $\min_{Q\in\mathcal D,\alpha}\nu^{(\epsilon)}_{\alpha,Q,0}\ge\rho$, (initially non-collapsed). Let $m_\epsilon\triangleq\min_{Q\in\mathcal D}m_{\epsilon,Q}$,
then for $0<k\ll|\tau|$ under RLVR optimization on $\mathcal D'$:

\begin{equation}
    |{\mathcal S}_{\epsilon,k}|
    =
    \Omega_{\rho,\gamma}\!\left(
        m_\epsilon N^{P_{acc,k}-\gamma}
    \right).
\end{equation}
\end{theorem}
\end{bluebox}
\textbf{Theorem~\ref{theorem:PowerLaw}} shows higher accuracy requires a \textbf{larger usable strategy capacity}. With \textbf{Theorem~\ref{theorem:concentration}}, this explains accuracy collapse: RLVR contracts strategy capacity until violating the bound, making performance unsustainable. Collapsing \(m_\epsilon\) strategies to one removes entropy \(\Delta H_{\mathrm{req}}\!\approx\!\log m_\epsilon\), which exceeds budget \(\Delta H_{\mathrm{ad}}=o(\log m_\epsilon)\), thus breaking the non-collapse regime. As Theorem~\ref{theorem:concentration} holds with high probability, occasional non-collapse runs are expected. See details in Appendix \hyperlink{proof:thm2}{\color{blue}{\ref*{app:theorem2_proof}}}.

\begin{wrapfigure}[11]{r}[0em]{0.49\textwidth}
    \vspace{-12pt}
    \centering
    \includegraphics[width=0.92\linewidth]{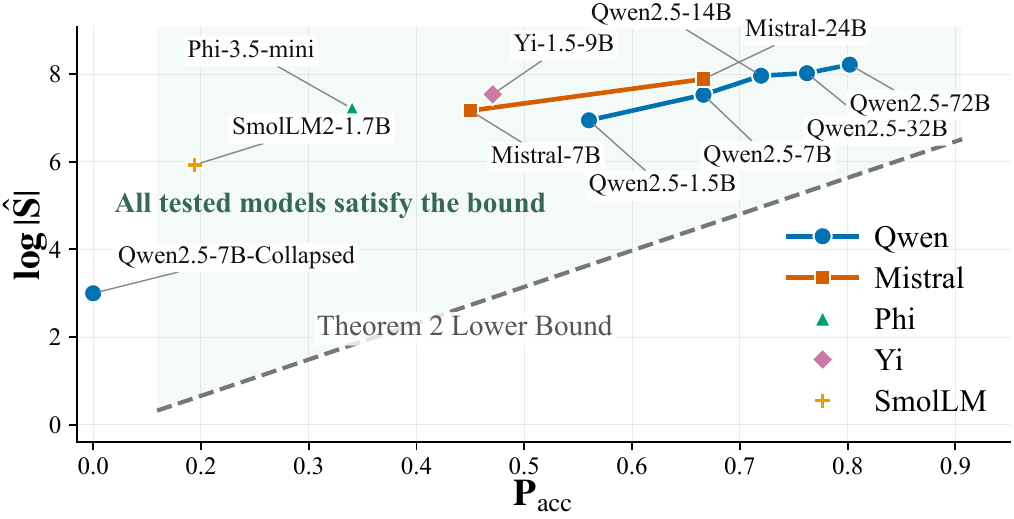}
    \caption{Empirical examination of Theorem~\ref{theorem:PowerLaw}. All tested models satisfy the lower bound.}
    \label{fig:powerlawbound}
\end{wrapfigure}

For empirical validation, we use a conservative estimate $\widehat m=4$, obtained by asking DeepSeek-V4-Pro to enumerate distinct solution methods for $1{,}000$ MATH problems and taking the minimum across prompts. This deliberately avoids overestimating $m$, ensuring that the empirical test does not rely on an inflated estimate of task strategy diversity. We then evaluate the resulting bound on models from 1.7B to 72B parameters. As shown in Fig.~\ref{fig:powerlawbound}, all tested models satisfy the predicted lower bound on strategy capacity.

\hypertarget{back:mei}{}
\subsection{{How to Measure and Forecast Collapse}}\label{subsec:mei}
The preceding analysis identifies strategy concentration as the latent transition underlying accuracy collapse, but directly tracking strategy partitions or parameter-space gradients during training is impractical. We therefore seek a lightweight observable of this transition. In Appendix~\hyperlink{proof:appB}{\ref*{app:mei}}, we prove that as collapse approaches, both the Fisher score $U_i$ and the logit gradients $v_i=\nabla_z\log\pi(\tau_i\mid Q)$ contract, i.e., $\|U_i-U_j\|\to0$ and $\|v_i-v_j\|\to0$. To reduce memory overhead, we construct our monitor from $v_i$, which can be obtained directly from rollout logits. This yields a lightweight online observable, the \textbf{Mirrored Entanglement Index (MEI)}:

\begin{bluebox}
\begin{definition}[Mirrored Entanglement Index (MEI)]
$
\text{MEI} = \| \sum_{i=1}^{G} v_i\|^2 / \sum_{i=1}^{G} \|v_i\|^2
$. 
\end{definition}
\end{bluebox}
Equivalently, $\mathrm{MEI}=G-\sum_{i<j}\|v_i-v_j\|^2/\sum_i\|v_i\|^2
=1+2\sum_{i<j}v_i^\top v_j/\sum_i\|v_i\|^2$. The first form shows that MEI increases as $\|v_i-v_j\|\to0$, while the second gives a baseline of $1$ for weakly coupled trajectories. Appendix~\hyperlink{proof:appB}{\ref*{app:mei}} establishes that increasing MEI reflects $v$-concentration, which implies $g$-concentration and departure from the non-collapsed regime. Calibrating on pre-RLVR Qwen2.5-7B-Instruct and Qwen3-4B across $3$ datasets yields the $3\sigma$ threshold $\mathrm{MEI}=1.013$.

\begin{wrapfigure}[9]{r}[0em]{0.52\textwidth}
\vspace{-12pt}
\centering
\includegraphics[width=1.0\linewidth]{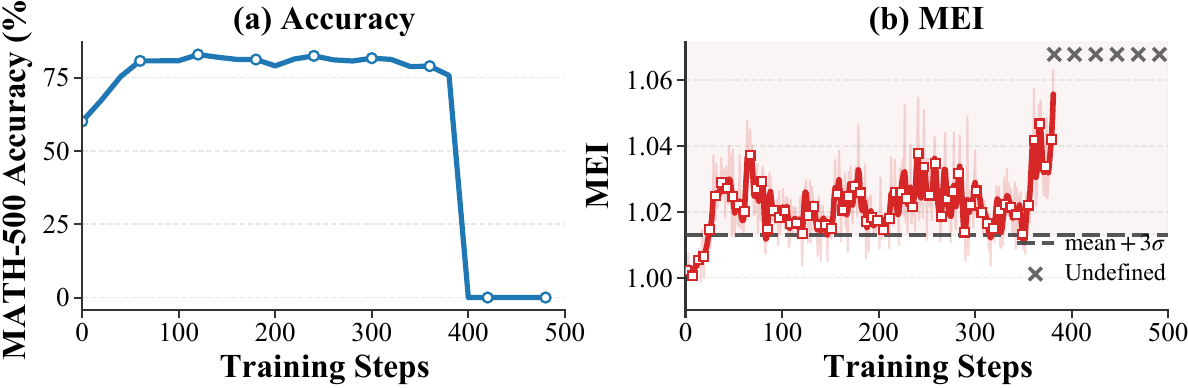}
\caption{MEI as an alert signal of collapse.}~\label{fig:mei}
\end{wrapfigure}

Fig.~\ref{fig:mei} illustrates this behavior in GRPO training. MEI crosses the \(3\sigma\) threshold early in training, before its final surge and the subsequent accuracy collapse. This provides a lightweight online signal for detecting departure from the non-collapsed regime, without explicitly tracking parameter-space gradients during optimization.
\subsection{Why Divergences Fail}
The preceding results identify collapse as intrinsic to pure RLVR. Can a fixed divergence penalty prevent it? \textbf{Theorem~\ref{theorem:divergence_failure}} says no. Using a \textbf{derived strategy-level RL system}, we show sufficiently large reward gaps can drive KL- or JS-regularized optima arbitrarily close to a single strategy.

\begin{bluebox}
\begin{theorem}[Failure of Divergence Regularization]\leavevmode\\
\label{theorem:divergence_failure}
Under \textbf{Theorem~\ref{theorem:concentration}}, within step $k\ll|\tau|$, RLVR admits a static strategy-level optimization
$\nu^D\triangleq\arg\max_\nu\{\mathbb E_\nu[R]-\beta D\},$
where $R(S)$ is the effective static reward, $\beta>0$ the regularization strength, and
$D$ be KL or JS divergence with full-support $\nu^{\rm ref}$. Let
$\Delta_R=R_{(1)}-R_{(2)}$ be the best-to-second-best reward gap,
 then:
\begin{equation}
    \forall\epsilon>0,
    \exists C_D<\infty \text{ such that } \Delta_R/\beta\ge C_D
    \implies \max_i\nu^D_i\ge1-\epsilon.
\end{equation}
\end{theorem}
\end{bluebox}
\hypertarget{back:exp2}{}
Theorem~\ref{theorem:divergence_failure} shows that if $\Delta_R/\beta$ exceeds a finite threshold $C_D$, the optimal distribution assigns at least $1-\epsilon$ mass to one best strategy for any $\epsilon>0$. Thus, fixed divergence penalties cannot universally prevent strategy concentration when reward dominates regularization. See details in Appendix~\hyperlink{proof:exp2}{\ref*{app:kl_js_fail}}.
\section{Methodology: Mesh Learning}

Sec.~\ref{sec:3} shows that preventing RLVR collapse requires preserving sufficient strategy capacity. This creates two challenges: training rollouts must elicit distinct strategies, and no single strategy should dominate optimization. We therefore propose \textbf{Mesh Learning} with two components: \textbf{\underline{1)} Coach Prompting}, which elicits multiple reasoning strategies for each query, and \textbf{\underline{2)} Strategy-Balancing Regularization}, which prevents any strategy from dominating optimization. Under the conditions analyzed in Appendix~\ref{app:mesh_learning}, we further prove that the resulting dynamics converge to equilibrium retaining nonzero probability mass for every strategy. The pipeline of Mesh Learning is shown in Fig.~\ref{fig:pipeline}, it replaces homogeneous GRPO rollouts with multi-strategy exploration and jointly optimizes GRPO with Strategy-Balancing Regularization.

\begin{figure}[!tb]
\centering
\includegraphics[width=1.0\linewidth]{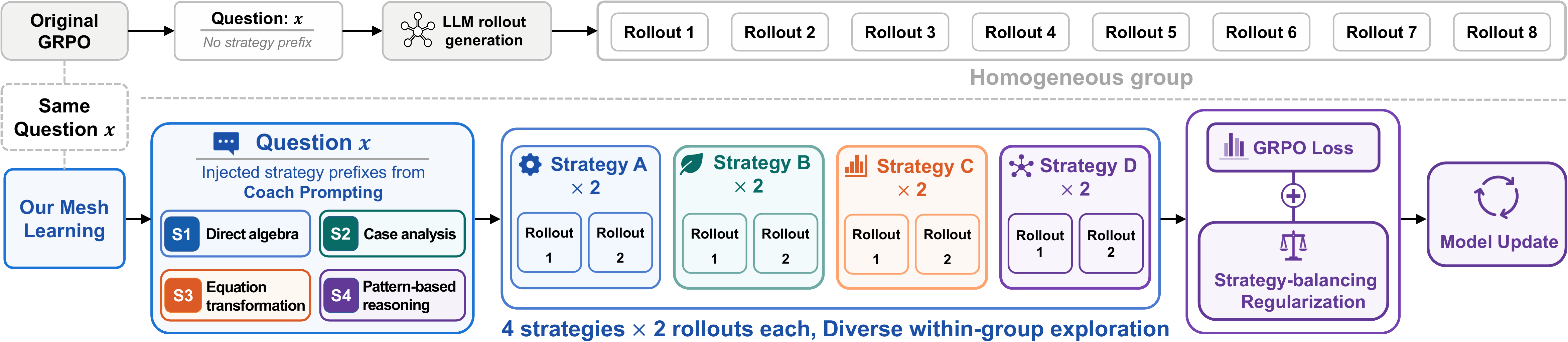}
\caption{{Mesh Learning versus standard GRPO.} Mesh Learning elicits multiple strategies and balances their relative growth during optimization.}
\label{fig:pipeline}
\end{figure}
\hypertarget{back:cp}{}

\subsection{Strategy Elicitation: Coach Prompting}
\label{subsec:data_orchestration}

\begin{wrapfigure}[14]{r}[0em]{0.5\textwidth}
    \vspace{-12pt}
    \centering
    \includegraphics[width=1.0\linewidth]{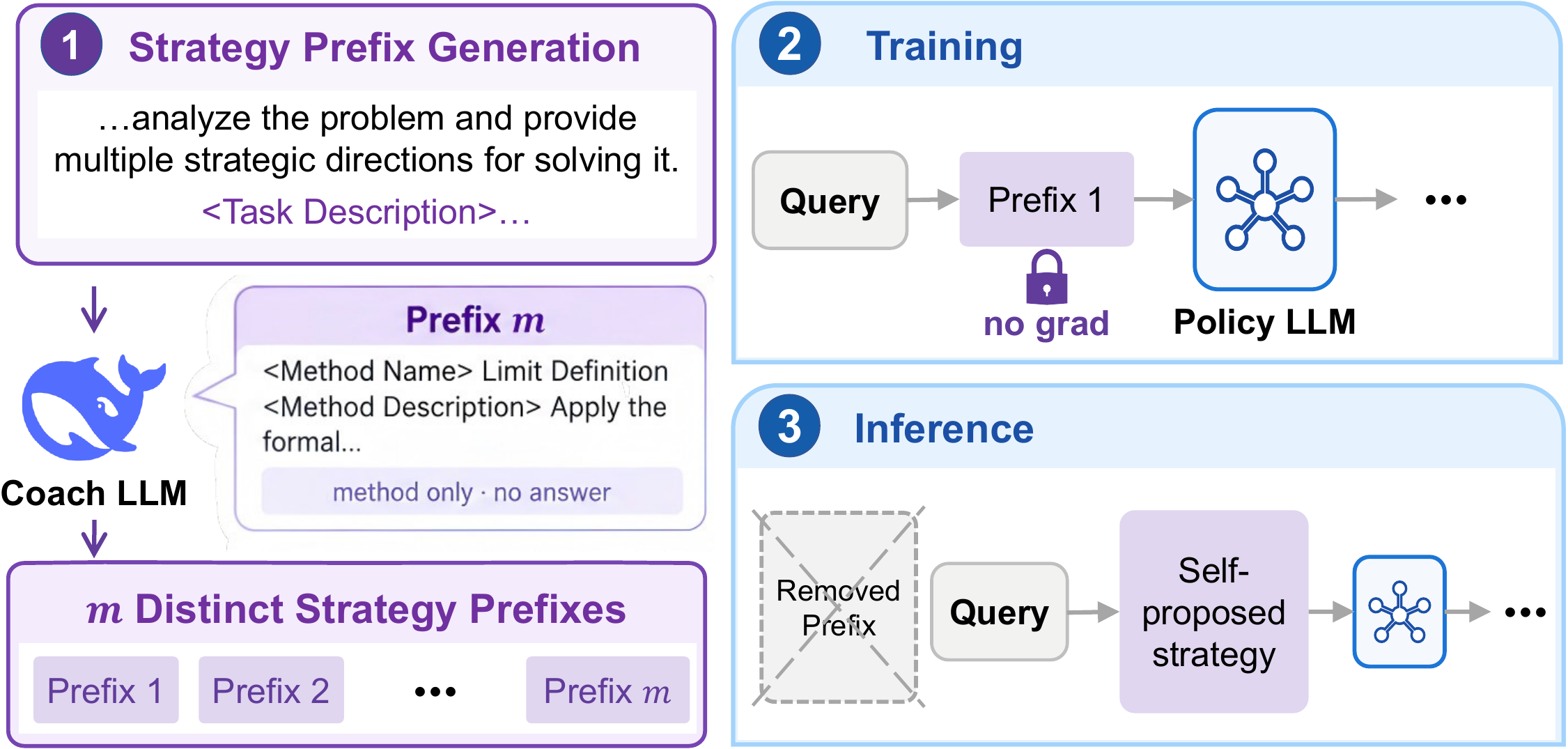}
    \caption{{Coach Prompting.}
    An offline Coach LLM generates fixed strategy prefixes for training; at inference, the policy proposes its own strategy and follows the same verification procedure.}
    \label{fig:method_cp}
\end{wrapfigure}

Sec.~\ref{subsec:2.2} empirically shows that trajectories using the same human-recognizable solution method exhibit similar coupling patterns, motivating solution-method prefixes as practical proxies for the latent strategies in Def.~\ref{def:strategy}. We therefore introduce \textbf{\textit{Coach Prompting} (CP)} to expose multiple strategies for each training query. For each query $Q$, an offline Coach LLM generates $m$ semantically distinct strategy prefixes, each specifying a candidate method, optionally with a brief description, but excluding calculations, intermediate solutions, and final answers. Each prefix conditions a separate rollout group, ensuring explicit exploration of all $m$ strategies rather than relying on spontaneous policy discovery. These prefixes are generated once and fixed throughout training, incurring no online Coach cost.

During training, each prefix is treated as a candidate rather than a fixed instruction: the policy assesses its applicability, underlying mechanism, and alternatives before answering. At inference, no pre-generated prefix is provided; the policy proposes a candidate strategy itself and follows the same verification procedure, requiring no Coach LLM at test time. Coach-provided tokens are masked from the training loss. Full prompts and implementation details are provided in Appendix~\ref{app:cp_prompt}.

\hypertarget{back:m_reg}{}
\subsection{Strategy-Balancing Regularization}
We seek to prevent strategy elimination without hindering their joint improvement. For strategy $S_i$, let $\mathcal{C}_i\subset S_i$ denote its correct trajectories, and define
$z_i=\frac{1}{|\mathcal{C}_i|}\sum_{\tau\in\mathcal{C}_i}\log\pi_\theta(\tau\mid Q)$
as a differentiable surrogate for its current support under the policy.
Let $\mathbf{z}=(z_1,\ldots,z_m)$ and let $\mathcal{U}_m$ denote the uniform distribution over the $m$ strategies. Directly minimizing $D_{\mathrm{KL}}(\mathcal{U}_m\|\mathrm{softmax}(\mathbf z))$ is biased by trajectory length, since longer trajectories accumulate more negative log-probability. Conversely, length normalization $z_i/|\tau_i|$ admits a padding shortcut, as high-probability filler tokens can increase the normalized score without actually strengthening the underlying strategy. We therefore measure each strategy's relative growth against a fixed pre-RLVR reference policy $\pi_{\mathrm{ref}}$.

\begin{bluebox}
\begin{definition}[Strategy-Balancing Regularization]\leavevmode\\
Let $z_i^{\mathrm{ref}}=\frac{1}{|\mathcal{C}_i|}\sum_{\tau\in\mathcal{C}_i}\log\pi_{ref}(\tau\mid Q)$ denote the strategy scores under $\pi_{\mathrm{ref}}$. We define
\begin{equation}
\mathcal{J}_S
=
D_{\mathrm{KL}}
\left(
\mathcal{U}_m
\|
\mathrm{softmax}(\mathbf{z}-\mathbf{z}^{\mathrm{ref}})
\right),
\end{equation}
and objective
$\mathcal{J}=\mathcal{J}_{\mathrm{RLVR}}-\mu\mathcal{J}_S$.
The three forms of $\mathcal{J}_{\mathrm{RLVR}}$ are provided in Appendix~\ref{app:def}. \hypertarget{back:notation}{}
\end{definition}
\end{bluebox}

Let $\mathbf{p}=\mathrm{softmax}(\mathbf{z}-\mathbf{z}^{\mathrm{ref}})$. Then
$\nabla_{\mathbf z}\mathcal{J}_S=\mathbf p-\frac{1}{m}\mathbf 1$,
which suppresses relatively fast-growing strategies while reinforcing lagging ones. Meanwhile,
$\mathcal{J}_S(\mathbf z+c\mathbf 1)=\mathcal{J}_S(\mathbf z)$,
so the regularizer controls only relative strategy growth without opposing their common improvement. In Appendix~\ref{app:mesh_learning}, we further prove that the resulting dynamics converge to existing equilibrium where every strategy retains nonzero probability, while preserving within-strategy trajectory competition to first order.
\section {Experiments}

\subsection{Experimental Setup}\label{subsec:exp_setup}

\noindent\textbf{Models and Data.}
We run full experiments on Qwen3-4B and Qwen2.5-7B-Instruct, with additional cross-family experiments on Phi-4-mini-reasoning. Qwen3-4B and Phi-4-mini-reasoning use thinking mode, while Qwen2.5-7B-Instruct uses standard instruction-tuned generation. For MATH-500, we train on the $7{,}500$-problem Hendrycks-MATH training split~\citep{hendrycks2021measuring}; for AIME25/26 and GPQA-Diamond, we use $7{,}000$ DAPOMath~\citep{yu2025dapo} and $4{,}500$ WildSci~\citep{liu2026wildsci} problems, respectively; and for LiveCodeBench, we use $4{,}500$ Skywork-OR1-RL-Data problems~\citep{he2025skywork}. For Coach Prompting, each query is augmented with $m=4$ distinct strategy prefixes generated offline by DeepSeek-V4-Pro and fixed throughout training.

\noindent\textbf{Training Configuration.}
Mesh Learning is implemented on top of GRPO using VeRL~\citep{sheng2024hybridflow}, with the clip ratio set to $0.2$. For DAPO and GSPO, we follow the clip ratio configurations recommended in their original papers, while keeping all other hyperparameters identical across baselines. We use $16$ rollouts per query and set the Strategy-Balancing Regularization weight to $\mu=10^{-5}$. Remaining configurations and hyperparameters are provided in Appendix~\ref{app:exp_setup}. \hypertarget{back:exp_config}{}

\noindent\textbf{Evaluation and Baselines.}
We evaluate on AIME25/26, MATH-500~\citep{hendrycks2021measuring}, GPQA-Diamond~\citep{rein2024gpqa}, and LiveCodeBench (v6)~\citep{jain2024livecodebench}, covering mathematical, scientific, and code reasoning. We compare Mesh Learning with GRPO, DAPO~\citep{yu2025dapo}, GSPO~\citep{zheng2025groupspo}, and GRPO with KL or JS regularization~\citep{li2026choice_js}.
\subsection{Main Results}
\label{subsec:exp_main_result}
\begin{table}[!tb]
\centering
\caption{
{Main results and controlled ablations (Mean@4 Accuracy \%).}
CP variants isolate the effect of Coach Prompting, while $m=3$ and $m=4$ compare different strategy capacities.
}
\label{tab:main_results}

\setlength{\tabcolsep}{3.2pt}
\renewcommand{\arraystretch}{1.05}
\small

\resizebox{0.97\textwidth}{!}{%
\begin{tabular}{lcccccccc}
\toprule
\multirow{2}{*}{\textbf{Method}}
& \multicolumn{4}{c}{\textbf{Qwen-2.5-7B-Instruct}}
& \multicolumn{4}{c}{\textbf{Qwen-3-4B}} \\
\cmidrule(lr){2-5}
\cmidrule(lr){6-9}
& \textbf{AIME26}
& \textbf{AIME25}
& \textbf{MATH-500}
& \textbf{GPQA}
& \textbf{AIME26}
& \textbf{AIME25}
& \textbf{MATH-500}
& \textbf{GPQA} \\
\midrule

\rowcolor{gray!15}
\multicolumn{9}{l}{\textbf{\textit{Standard RLVR Training}}} \\

Baseline
& 3.3 & 5.0 & 64.8 & 13.7
& 13.3 & 21.6 & 62.5 & 13.7 \\

GRPO
& 9.2 & 13.3 & 69.6 & 33.4
& 43.3 & 40.0 & 81.6 & 38.9 \\

DAPO
& 8.3 & 6.7 & 70.5 & 34.1
& 43.3 & 40.8 & 84.3 & 44.5 \\

GSPO
& 9.2 & 5.8 & 71.4 & 34.3
& 41.6 & 38.3 & 79.8 & 45.8 \\

GRPO + KL
& 9.2 & 13.3 & 71.8 & 31.6
& 43.3 & 40.0 & 82.2 & 42.3 \\

GRPO + JS
& 10.0 & 6.4 & 71.4 & 28.5
& 40.0 & 43.3 & 81.8 & 44.8 \\

\midrule

\rowcolor{gray!15}
\multicolumn{9}{l}{
\textbf{\textit{Controlled with Coach Prompting (CP)}}
} \\

GRPO + CP
& 10.0 & 13.3 & 69.6 & 31.7
& 43.3 & 40.0 & 81.6 & 37.8 \\

DAPO + CP
& 7.5 & 5.8 & 69.7 & 26.3
& 36.6 & 33.3 & 84.5 & 42.3 \\

GSPO + CP
& 6.7 & 8.3 & 68.5 & 32.1
& 41.7 & 33.3 & 85.3 & 43.7 \\

GRPO + KL + CP
& 4.4 & 8.3 & 71.8 & 31.9
& 40.4 & 38.3 & 78.5 & 33.2 \\

GRPO + JS + CP
& 10.0 & 7.8 & 72.0 & 30.7
& 40.0 & 41.4 & 80.4 & 30.7 \\

\midrule

\textbf{Mesh Learning ($m=3$)}
& 10.3
& 16.3
& 72.4
& 36.7
& 52.5
& 43.7
& 84.8
& 44.6 \\

\rowcolor{blue!5}
\textbf{Mesh Learning ($m=4$)}
& \textbf{13.3}
& \textbf{16.7}
& \textbf{73.4}
& \textbf{36.7}
& \textbf{56.7}
& \textbf{48.3}
& \textbf{86.8}
& \textbf{47.9} \\

\bottomrule
\end{tabular}%
}
\end{table}
\noindent\textbf{Overall Performance.}
Table~\ref{tab:main_results} summarizes results across two models and four benchmarks. We highlight two observations: \textbf{\underline{1)} Mesh Learning consistently improves reasoning performance.} With the default $m=4$, Mesh Learning achieves the best performance across all eight model--benchmark settings. Compared with the strongest RLVR baseline in each setting, Mesh Learning improves Qwen2.5-7B-Instruct by $1.6$--$3.4$ pp and Qwen3-4B by $2.1$--$13.4$ pp. The largest gain occurs on AIME26 with Qwen3-4B, improving accuracy from $43.3\%$ to $56.7\%$. \textbf{\underline{2)} We observe consistent gains across tasks and model scales.} Mesh Learning improves challenging competition-level benchmarks (AIME25/26) and broader mathematical and scientific reasoning benchmarks (MATH-500 and GPQA-Diamond). These gains hold across Qwen2.5-7B-Instruct and Qwen3-4B, indicating that the benefit is not specific to a single model or task. Trajectory-level examples are provided in Appendix~\ref{app:case_study}, while ablations on Coach Prompting and strategy number $m$ are presented in Sec.~\ref{subsec:ablation}.
\hypertarget{back:case_study}{}

\begin{figure}[!tb]
\centering
\includegraphics[scale=0.27]{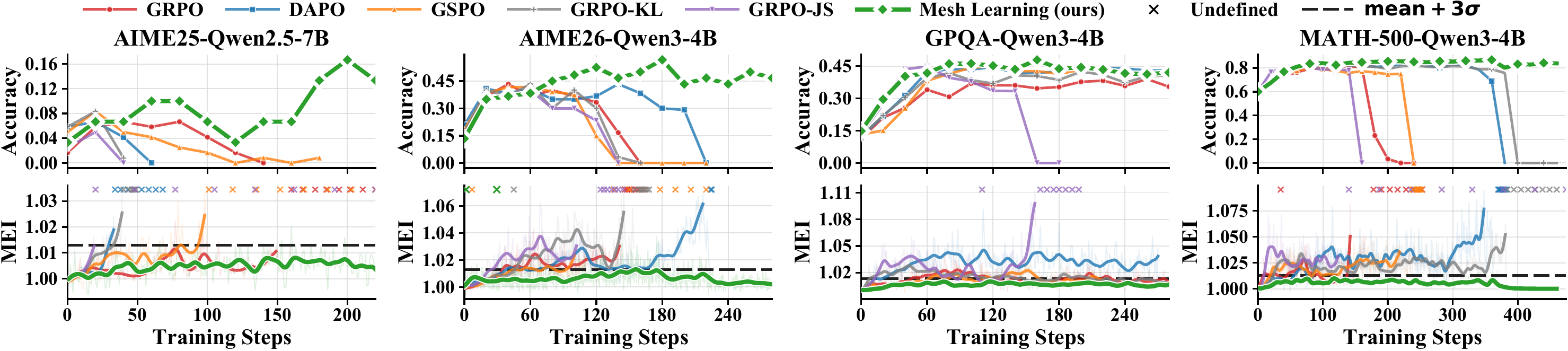}
\caption{Training dynamics of accuracy (top) and MEI (bottom). Baselines can exhibit late-stage collapse with rising MEI, while Mesh Learning maintains stable accuracy and low MEI.}
\label{fig:exp_dynamics}
\end{figure}
\begin{figure}[!tb]
\begin{center}
\includegraphics[width=1.0\linewidth]{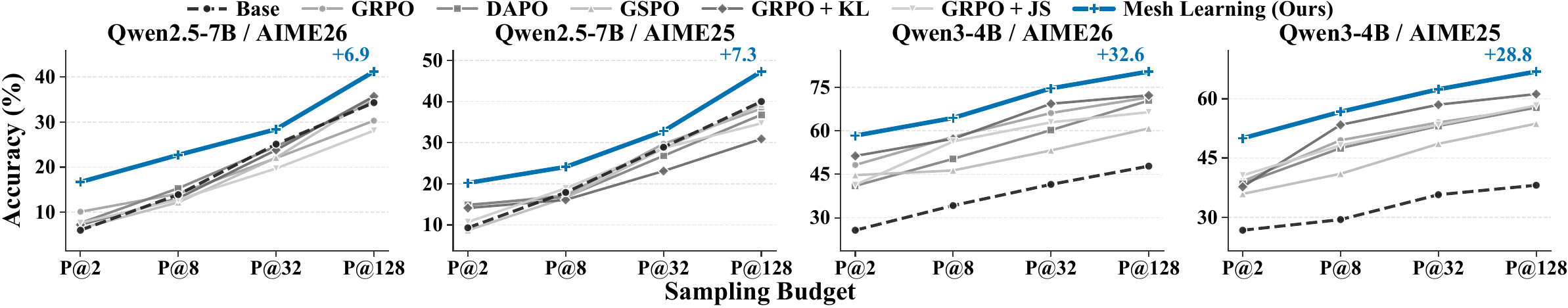}
\end{center}
\caption{
{Pass@$k$ performance on AIME25 and AIME26.}
Across increasing inference budgets, Mesh Learning consistently outperforms all baselines on both models.
}
\label{fig:capability_frontier}
\end{figure}
\noindent\textbf{Training Dynamics.}
Fig.~\ref{fig:exp_dynamics} shows that several RLVR baselines improve but later suffer abrupt accuracy collapse, accompanied by a sharp rise in MEI beyond the $3\sigma$ threshold, while Mesh Learning keeps MEI low and consistently attains higher accuracy. GPQA is an exception: four choices imply $\gamma\approx1$, violating Theorem~\ref{theorem:PowerLaw}'s condition, so strategy concentration need not induce collapse.

\noindent\textbf{Capability Frontier.} To examine whether these gains persist under larger inference budgets, we report Pass@$k$ on AIME25 and AIME26 for $k\in\{2,8,32,128\}$. As shown in Fig.~\ref{fig:capability_frontier}, Mesh Learning consistently outperforms all baselines across both models, with its advantage remaining substantial as $k$ increases. In particular, the gains at Pass@128 reach $6.9$ and $7.3$ pp for Qwen2.5-7B and $32.6$ and $28.8$ pp for Qwen3-4B on AIME26 and AIME25, respectively. These results indicate that Mesh Learning improves not only single-budget accuracy but also the attainable performance under substantially larger sampling budgets.

\begin{wraptable}[8]{r}{0.55\textwidth}
    \centering
    \vspace{-10pt}
    \footnotesize
    \caption{{Cross-family generalization on Phi-4-mini-reasoning
    (Mean@4 Accuracy \%).}}
    \label{tab:phi_results}
    \resizebox{\linewidth}{!}{
        \begin{tabular}{lcccc}
            \toprule
            \textbf{Method}
            & \textbf{AIME26}
            & \textbf{AIME25}
            & \textbf{MATH-500}
            & \textbf{GPQA} \\
            \midrule
            GRPO
            & 22.3 & 24.1 & 79.9 & 18.4 \\
            DAPO
            & 25.0 & 23.4 & 77.3 & 23.0 \\
            \rowcolor{blue!5}
            \textbf{Mesh Learning}
            & \textbf{33.3}
            & \textbf{30.8}
            & \textbf{83.2}
            & \textbf{34.5} \\
            \bottomrule
        \end{tabular}
    }
\end{wraptable}
\noindent\textbf{Cross Family Generalization.} Beyond Qwen3-4B and Qwen2.5-7B-Instruct, we further evaluate Mesh Learning on Phi-4-mini-reasoning to examine transfer across model families. 
As shown in Table~\ref{tab:phi_results}, Mesh Learning consistently outperforms both GRPO and DAPO across all four benchmarks, achieving gains over the stronger baseline of $8.3$, $6.7$, $3.3$, and $11.5$ pp on AIME26, AIME25, MATH-500, and GPQA-Diamond, respectively. These results provide additional evidence that the effectiveness of Mesh Learning is not specific to the Qwen family.

\begin{wraptable}[6]{r}{0.55\textwidth}
    \centering
    \footnotesize
    \vspace{-12pt}
    \caption{LiveCodeBench v6 (Mean@4 Accuracy \%)}
    \label{tab:code_results}
    \resizebox{\linewidth}{!}{
        \begin{tabular}{lcccc}
            \toprule
            \textbf{Model}
            & \textbf{Baseline}
            & \textbf{GRPO}
            & \textbf{DAPO}
            & \textbf{Mesh Learning} \\
            \midrule
            \textbf{Qwen-3-4B}
            & 30.5 & 48.4 & 52.2 & \sota{56.4} \\
            \textbf{Phi-4-mini-reasoning}
            & 24.8 & 31.5 & 32.7 & \sota{36.3} \\
            \bottomrule
        \end{tabular}
    }
\end{wraptable}
\noindent\textbf{Generalization to Coding Tasks.}
To evaluate generalization beyond math reasoning, we extend experiments to LiveCodeBench v6. Table~\ref{tab:code_results} shows that Mesh Learning reaches 56.4\% on Qwen3-4B and 36.3\% on Phi-4-mini-reasoning, outperforming DAPO by 4.2 pp and 3.6 pp, suggesting that its benefits transfer to coding tasks.

\subsection{Ablation Studies}
\label{subsec:ablation}

\noindent\textbf{Ablation for Coach-Prompting Priors.}
Since Coach Prompting provides extra strategy information, we compare each RLVR baseline with a CP-only counterpart using the same prefixes but without Strategy-Balancing Regularization. As shown in Table~\ref{tab:main_results}, CP alone brings little or no improvement, while full Mesh Learning consistently performs better. This indicates that the gains do not come from stronger strategy priors alone, but from preserving these strategies during optimization.

\noindent\textbf{Number of Strategies.}
As shown in Table~\ref{tab:main_results}, increasing $m$ from $3$ to $4$ leads to only small differences overall, showing that Mesh Learning is not sensitive to modest changes in $m$. However, we found that an overly small $m$ limits the number of strategies exposed during training and therefore restricts the strategy capacity that can be preserved. We use $m=4$ by default.
\section{Related Work}
\label{sec:related_work}

\textbf{RLVR optimization.}
Following GRPO~\citep{shao2024DeepSeekMath}, substantial work has improved RLVR through better sampling, clipping, importance weighting, and optimization objectives, including DAPO~\citep{yu2025dapo}, CISPO~\citep{chen2025minimax}, GSPO~\citep{zheng2025groupspo}, SAPO~\citep{gao2025sapo}, Dr. GRPO~\citep{liu2025understanding}, VAPO~\citep{yue2025vapo}, REINFORCE++~\citep{hu2025reinforce++},  GMPO~\citep{ICLR2026_44a1f18a} and FAPO~\citep{ding2026fapo}. These methods primarily improve RLVR's optimization behavior, whereas we focus on the mechanism underlying its late-stage collapse.

\textbf{Understanding and Preventing Collapse.}
Recent work has documented collapse during prolonged RLVR training. \citet{ICLR2026_805cd624} show that Pass@1 improvements can coincide with reduced generation diversity and degraded high-$k$ Pass@$k$, while \citet{nguyen2025reasoning} attribute reasoning-boundary shrinkage to negative interference and winner-take-all dynamics. \citet{yuan2026understanding} link diversity collapse to continued optimization on already-saturated problems, and M-GRPO~\citep{bai2025m} studies abrupt performance collapse under long-horizon training. A common remedy is to constrain policy drift with fixed-reference KL regularization~\citep{shao2024DeepSeekMath}, while alternative divergences such as forward KL and JS have been shown to better preserve solution diversity~\citep{li2026choice_js}. In contrast, we provide a strategy-level theoretical account: starting from trajectory-level update interactions, we derive RLVR's tendency toward strategy concentration and a complementary lower bound on the strategy capacity required for nontrivial task performance. Their tension directly explains catastrophic collapse and motivates Mesh Learning.

\section{Conclusion}

This work characterizes, anticipates, and mitigates catastrophic collapse in RLVR. We develop a strategy-level theory showing that RLVR progressively concentrates strategy mass, while successful reasoning requires sufficient strategy capacity; their conflict results collapse. This analysis further yields MEI as a lightweight online warning signal. We propose \textbf{Mesh Learning}, combining Coach Prompting and Strategy-Balancing Regularization to expose and preserve diverse strategies during optimization. Experiments across AIME26, AIME25, MATH-500, GPQA, and LiveCodeBench consistently demonstrate its effectiveness across Qwen and Phi model families.
In conclusion, \textbf{All Work And No Play Makes Jack A Dull Boy}. RLVR may trade away strategy diversity for continued gains in reward, eventually exhausting the capacity required for further meaningful optimization. Our results highlight strategy preservation as a key principle for stable and effective RLVR.

\subsection*{AI use statement}

In this work, we used generative AI tools to assist with language polishing, code implementation, and the refinement and correctness checking of theorem proofs. Generative AI tools were not used to independently write the manuscript from scratch or to generate citations or bibliographic references.

All AI-assisted code was reviewed, verified, and tested by the authors. All theorem statements and proofs, including those for which generative AI tools were used as an auxiliary correctness check, were independently examined and verified by the authors. No citations were generated by generative AI tools, and all references included in the paper were manually checked by the authors to confirm their authenticity and relevance.

We have reviewed all AI-assisted work and take full responsibility for the final content of this work, including all text, claims, proofs, code, and other artifacts produced with the aid of generative AI.

\subsection*{Reproducibility statement}

All code and scripts are available in our anonymous GitHub repository at \href{https://github.com/Ayanami-0123/Open-Mesh-Learning}{GitHub}. All datasets required for training, including those processed with our Coach Prompting procedure, are provided in the supplementary materials as data.zip, except for the coding dataset, which is too large to include in the submission.

\bibliography{iclr2026_conference}
\bibliographystyle{iclr2026_conference}

\appendix

\clearpage
\section{{Appendix A: Math Proof On The Collapse Of RLVR}}

\subsection{{Definitions and Mathematical Framework}}\label{app:def}
\subsubsection{Notations}
\hyperlink{back:notation}{\color{blue}{Back}}
To formalize collapse in Group Relative Policy Optimization (GRPO), we use the following notation:

\begin{itemize}
    \item $\pi$: model policy; $\theta$: model parameters.
    \item $\mathcal{D}$: test dataset with prompts $Q\in\mathcal{D}$.
    \item $\tau$: generated trajectory.
    \item $\mathcal{T}_Q$: trajectory space conditioned on prompt $Q$.
    \item $\mathcal{T}$: global trajectory space.
\end{itemize}

\subsubsection{Mathematical Expressions of Three RLVR Objectives}

Let $\{\tau_i\}_{i=1}^{G}\sim\pi_{\mathrm{old}}(\cdot\mid Q)$ denote a response group and $A_i$ the group-relative advantage of $\tau_i$. Following the response-level GRPO formulation in DeepSeek-R1~\citep{guo2025deepseekr1}, define
\begin{equation}
    \rho_i(\theta)
    \triangleq
    \frac{\pi_\theta(\tau_i\mid Q)}
         {\pi_{\mathrm{old}}(\tau_i\mid Q)}.
\end{equation}
We display only the policy-gradient objectives; KL regularization is analyzed separately.

For GRPO,
\begin{equation}
    \mathcal{J}_{\mathrm{GRPO}}(\theta)
    =\frac{1}{G}\sum_{i=1}^{G}
    \min\!\left(
        \rho_i(\theta)A_i,
        \operatorname{clip}\!\left(\rho_i(\theta),1-\epsilon_{clip},1+\epsilon_{clip}\right)A_i
    \right).
\end{equation}

For DAPO, define the token-level ratio
\begin{equation}
    r_{i,t}(\theta)
    \triangleq
    \frac{\pi_\theta(\tau_{i,t}\mid Q,\tau_{i,<t})}
         {\pi_{\mathrm{old}}(\tau_{i,t}\mid Q,\tau_{i,<t})},
\end{equation}
with objective
\begin{equation}
\begin{aligned}
    \mathcal{J}_{\mathrm{DAPO}}(\theta)
    &=\frac{1}{\sum_{i=1}^{G}|\tau_i|}
    \sum_{i=1}^{G}\sum_{t=1}^{|\tau_i|}
    \min\!\left(
        r_{i,t}(\theta)A_i,
        \operatorname{clip}\!\left(
            r_{i,t}(\theta),1-\epsilon_{\mathrm{low}},1+\epsilon_{\mathrm{high}}
        \right)A_i
    \right),\\
    &\hspace{2em}\text{s.t.}\quad
    0<\sum_{i=1}^{G}\mathbb{I}\{\tau_i\text{ is correct}\}<G.
\end{aligned}
\end{equation}

For GSPO, define the length-normalized sequence ratio
\begin{equation}
    s_i(\theta)
    \triangleq
    \rho_i(\theta)^{1/|\tau_i|}
    =\exp\!\left(\frac{1}{|\tau_i|}\log\rho_i(\theta)\right),
\end{equation}
and
\begin{equation}
    \mathcal{J}_{\mathrm{GSPO}}(\theta)
    =\frac{1}{G}\sum_{i=1}^{G}
    \min\!\left(
        s_i(\theta)A_i,
        \operatorname{clip}\!\left(s_i(\theta),1-\epsilon_{clip},1+\epsilon_{clip}\right)A_i
    \right).
\end{equation}

\subsubsection{Measures on Strategy Space}
At optimization step $k$, the conditional probability mass of strategy $S_{\alpha,Q}$ under prompt $Q$ is
\begin{equation}
    \nu^{(\epsilon)}_{\alpha,Q,k}
    \triangleq
    \sum_{\tau\in S^{(\epsilon)}_{\alpha,Q}}
    \pi_k(\tau\mid Q).
\end{equation}
\hypertarget{proof:thm1}{}
\subsection{{Proof for theorem 1}}\label{app:theorem1_proof}
\hypertarget{back:thm1:step1}{}
\hypertarget{back:thm1:step2}{}
\hypertarget{back:thm1:step3}{}
\hypertarget{back:thm1:step4}{}
\hypertarget{thm1:step5:back}{}
\begin{bluebox}
\hyperlink{back:thm1}{\color{blue}{Back}}
\noindent\textbf{Proof Sketch for Theorem 1.} 
RLVR acts on the model through the following chain: it directly updates policy parameters, which gradually changes trajectory probabilities and in turn affects strategy distribution. So the core idea of our proof is \textbf{simple}: To study how strategies compete and evolve, we lift the parameter update signals ALL the way up to the strategy distribution, deriving an evolution dynamics on the strategy simplex. The proof is completed through 1 preparation step and 5 main steps, organized as below:

\textbf{Step 0: Dealing with the Clipping Term:}
We make one simple reduction to handle clipping. For GRPO and GSPO, define a
binary indicator $m_i$, where $m_i=0$ if trajectory $\tau_i$ is clipped and $1$ otherwise. We then absorb clipping into
the effective advantage,
\begin{equation}
A_i^{\mathrm{eff}} = m_i A_i,
\end{equation}
so their gradients reduce to
$G^{-1}\sum_i \rho_i A_i^{\mathrm{eff}}$ and
$G^{-1}\sum_i s_i A_i^{\mathrm{eff}}$, respectively.
For DAPO, the same reduction is applied token-wise using indicators $m_{i,t}$,
yielding the corresponding masked token-level gradient.

\textbf{Step 1: The Dynamics of Trajectory Probabilities:} Consider step $k$, we derive that trajectory dynamics share the same form under the three surrogates in first order:
\begin{equation}
\Delta\pi_k(\tau_i|Q)=\eta\pi_k(\tau_i|Q)\sum_j\pi_k(\tau_j|Q')K_{ij,k}(Q,Q')\bar A_{j,k}^{\mathrm{eff}}r_{j,k}+\xi_{i,k}^{(\mathrm{batch})}.
\end{equation}
The noise has two origins, $\xi_{i,k}^{(\mathrm{batch})}=\xi_{i,k}^{(\mathrm{sample})}+\xi_{i,k}^{(\mathrm{advantage})}$, satisfying $\mathbb E[\xi_{k}^{(batch)}|\mathcal F_k]=0$. $r_{i,k}$ is the correction factor, which is $1$ for GRPO, $|\tau_j|^{-1}\rho_{j,k}^{1/|\tau_j|-1}$ for GSPO, and $\rho_{j,k}^{-1}|\tau|^{-1}$ for DAPO in first order. (Appendix \hyperlink{thm1:step1}{\color{blue}{A.2.1}})

\textbf{Step 2: Trajectory Evolution within a Strategy:} We then derive the evolution of trajectories inside a strategy, finding it in the form of a \textbf{replicator dynamics} in first order:
\begin{equation}
    \Delta q_{i|\alpha,Q,k} = \eta \cdot q_{i|\alpha,Q,k} \cdot (f_{i,k} - f_{\alpha}) + \widetilde{\xi}_{i,k}^{(batch)} + \widetilde{\xi}_{i,k}^{(data)}
\label{eq:replicator}
\end{equation}
Where $f_{i,k}=\frac{1}{|\mathcal{D}'|}\sum_{Q'}\sum_{j}\pi_k(\tau_j|Q')K_{ij,k}(Q,Q')\bar A_{j,k}^{\mathrm{eff}}r_{j,k}$ is the fitness for trajectory $\tau_i$,  and $q_{i|\alpha,Q,k}=\frac{\pi_k(\tau_i|Q)}{\nu_{\alpha,Q,k}}$ trajectory distribution within $S_{\alpha,Q}^{(\epsilon)}$. (Appendix \hyperlink{thm1:step2}{A.2.2})

\textbf{Step 3: Adiabatic Approximation:} Before analyzing the strategy dynamics, we establish the scale separation. The variables involved are $\pi,\nu,q,A,K,r$. We define their relative changing rate as $\omega_\pi,\omega_\nu,\omega_q,\omega_r,\omega_K,\omega_A$. Using equation~\ref{eq:replicator}, we prove that in condition that $\eta\ll\epsilon_{clip}<1\ll|\tau|\ll\eta^{-1}$, and within time window $k\ll|\tau|$:
\begin{equation}
    \omega_q,\omega_K,\omega_r,\omega_A\ll\omega_\nu\sim\omega_\pi
\end{equation}
Hence, the quantities $q,A,K,r$ can be frozen in the time window when deriving the strategy dynamics, so \textbf{we can focus attention on the strategy evolution}. (Appendix \hyperlink{thm1:step3}{A.2.3})

\textbf{Step 4: Hitting time analysis:} We first associate each strategy $S_{\alpha,Q}$ with its effective
strategy-level fitness $a_{\alpha,Q}$. By Lemma~\ref{lemma:impossiblity_of_equivalence},
distinct strategies
have almost surely different fitness. We then construct a potential function $\Phi_k(\nu_{\alpha,Q,k},a_{\alpha,Q})$, whose drift is \textbf{almost surely} a positive definite quadric form. This makes $\Phi_k$ a submartingale, and using Azuma-Hoeffding inequality, we prove that in condition $\eta^{-1}\ll|\tau|^3$: 
\begin{equation}
\mathbb P\left(
\exists k<K:
\max_\alpha\nu_{\alpha,Q,k}\ge1-\zeta
\right)
\ge
1-\frac{\delta}{2}
\end{equation}
Further analyze shows that hitting time $K\ll\mathcal O(|\tau|)$, meaning the hitting happens within the adiabatic zone, thereby validating the consistency of our analysis. (Appendix \hyperlink{thm1:step4}{\color{blue}{A.2.4}})

\textbf{Step 5: Preservation after the adiabatic window:} Step 4 shows that the process \textbf{can reach} a concentrated state. We next fix the
dominant strategy $o$ and show, using Azuma--Hoeffding again, that it remains
dominant throughout the remaining time window:
\begin{equation}
\mathbb P\left(
    \exists k\in\mathbb Z^+,
    \ K<k\le T_c:
    \nu_{o,Q,k}<1-\zeta
    \ \middle|\;
    \nu_{o,Q,K}\ge1-\mathcal{O}(\eta)
\right)\le\delta/2\end{equation}
Combining it with the result in Step 4, we can derive the convergence expression in the main text.
(Appendix \hyperlink{thm1:step5}{\color{blue}{A.2.5}})
\textbf{This completes the proof.} \hfill$\blacksquare$
\end{bluebox}

We will start from the simplest case, which only considers influence exerted by the group-relative reward in RLVR algorithms. We emphasize that divergence regularization is not neglected; rather, its role and mechanism are discussed and analyzed separately in Theorem~\ref{theorem:divergence_failure}. Our main focus on group-relative reward in A.2 is to set up a relatively clear and analytic framework, in order to facilitate further discussions and effective theory development.  
\hypertarget{thm1:step1}{}
\subsubsection{{Step 1: Deriving the Dynamics of Trajectory Probabilities:}}
\hyperlink{back:thm1:step1}{\color{blue}{Back}}
At optimization step $k$, let $\mathcal{F}_k$ denote the training history before the current rollout, write $\pi_k\triangleq\pi_{\theta_k}$, and let $\pi_{\mathrm{old},k}$ be the behavior policy used to sample the current group. Define
\begin{equation}
\Delta\pi_k(\tau_i|Q)\triangleq\pi_{\theta_{k+1}}(\tau_i|Q)-\pi_{\theta_k}(\tau_i|Q),
\qquad
U_{i,k}(Q)\triangleq\nabla_\theta\log\pi_\theta(\tau_i|Q)\big|_{\theta=\theta_k}.
\end{equation}
The first-order probability increment used in this step is
\begin{equation}
\Delta\pi_k(\tau_i|Q) = \nabla _ \theta \pi _ \theta (\tau_i|Q)^\top\big|_{\theta=\theta_k}\cdot\Delta \theta_k.
\end{equation}
Using the log-derivative identity, this becomes
\begin{equation}
\Delta\pi_k(\tau_i|Q) = \pi_k(\tau_i|Q)\cdot U_{i,k}(Q)^\top\cdot\Delta \theta_k.
\end{equation}

\textbf{We first give the GRPO derivation}. Let $B_k=\{\tau_j\}_{j=1}^G$ be the response group sampled from $\pi_{\mathrm{old},k}(\cdot|Q')$. Away from clipping boundaries, decompose each effective group-relative advantage as
\begin{equation}
A_{j,k}^{\mathrm{eff}}=\bar A_{j,k}^{\mathrm{eff}}+\varepsilon_{j,k}^{(\mathrm{advantage})},
\qquad
\bar A_{j,k}^{\mathrm{eff}}\triangleq
\mathbb{E}\!\left[A_{j,k}^{\mathrm{eff}}\mid\tau_j,Q',\mathcal{F}_k\right],
\qquad
\mathbb{E}\!\left[\varepsilon_{j,k}^{(\mathrm{advantage})}\mid\tau_j,Q',\mathcal{F}_k\right]=0.
\end{equation}
Note that group-wise centering does not imply
$\bar A^{\mathrm{eff}}_{j,k}=0$ once we condition on the focal trajectory
$\tau_j$. When the clipped branch is active, the surrogate is locally flat
and hence $A^{\mathrm{eff}}_{j,k}=0$. Away from clipping, writing
\[
A_{j,k}
=
r(\tau_j)-\frac{1}{G}\sum_{\ell=1}^{G}r(\tau_\ell),
\]
Calculating its expectation, we have
\begin{align}
\bar A^{\mathrm{eff}}_{j,k}
&=
\mathbb E\!\left[
A_{j,k}\mid \tau_j,Q',\mathcal F_k
\right] \nonumber\\
&=
\frac{G-1}{G}
\left(
r(\tau_j)
-
\mathbb E_{\tau\sim\pi_{\mathrm{old},k}(\cdot\mid Q')}
[r(\tau)\mid Q']
\right),
\label{eq:conditional-grpo-advantage}
\end{align}
which is generally nonzero.

Here the randomness caused by forming a finite group and estimating its group-relative advantage is included in $\varepsilon_{j,k}^{(\mathrm{advantage})}$. With
$\rho_{j,k}=\pi_k(\tau_j|Q')/\pi_{\mathrm{old},k}(\tau_j|Q')$, its stochastic parameter update is
\begin{equation}
\Delta\theta_k = \eta\cdot\frac{1}{G}\sum_{j=1}^G\rho_{j,k}\cdot A_{j,k}^{\mathrm{eff}}\cdot U_{j,k}(Q').
\end{equation}
Define its conditional population drift and the two parameter-space noise terms by
\begin{equation}
\bar g_k^{\mathrm{GRPO}}(Q')
\triangleq
\sum_j\pi_{\mathrm{old},k}(\tau_j|Q')\rho_{j,k}\bar A_{j,k}^{\mathrm{eff}}U_{j,k}(Q')
=
\sum_j\pi_k(\tau_j|Q')\bar A_{j,k}^{\mathrm{eff}}U_{j,k}(Q'),
\end{equation}
\begin{equation}
\epsilon ^{(\mathrm{sample})}
\triangleq
\frac{1}{G}\sum_{j=1}^G\rho_{j,k}\bar A_{j,k}^{\mathrm{eff}}U_{j,k}(Q')
-\bar g_k^{\mathrm{GRPO}}(Q'),
\end{equation}
\begin{equation}
\epsilon ^{(\mathrm{advantage})}
\triangleq
\frac{1}{G}\sum_{j=1}^G\rho_{j,k}\varepsilon_{j,k}^{(\mathrm{advantage})}U_{j,k}(Q').
\end{equation}
By construction,
\begin{equation}
\mathbb{E}\!\left[\epsilon ^{(\mathrm{sample})}\mid Q',\mathcal{F}_k\right]=0,
\qquad
\mathbb{E}\!\left[\epsilon ^{(\mathrm{advantage})}\mid Q',\mathcal{F}_k\right]=0,
\end{equation}
so that
\begin{equation}
\Delta\theta_k
=\eta\bar g_k^{\mathrm{GRPO}}(Q')
+\eta\epsilon ^{(\mathrm{sample})}
+\eta\epsilon ^{(\mathrm{advantage})}.
\end{equation}

Introduce the step-indexed trajectory coupling kernel
\begin{equation}
K_{ij,k}(Q,Q')\triangleq U_{i,k}(Q)^\top U_{j,k}(Q').
\end{equation}
Substitution into the probability increment gives
\begin{equation}
\Delta\pi_k(\tau_i|Q)
=\eta\pi_k(\tau_i|Q)\sum_j\pi_k(\tau_j|Q')K_{ij,k}(Q,Q')\bar A_{j,k}^{\mathrm{eff}}
+\xi_{i,k}^{(\mathrm{batch})},
\end{equation}
where
\begin{equation}
\begin{aligned}
\xi_{i,k}^{(\mathrm{batch})}
&\triangleq\xi_{i,k}^{(\mathrm{sample})}+\xi_{i,k}^{(\mathrm{advantage})},\\
\xi_{i,k}^{(\mathrm{sample})}
&\triangleq\eta\pi_k(\tau_i|Q)U_{i,k}(Q)^\top\epsilon ^{(\mathrm{sample})},\\
\xi_{i,k}^{(\mathrm{advantage})}
&\triangleq\eta\pi_k(\tau_i|Q)U_{i,k}(Q)^\top\epsilon ^{(\mathrm{advantage})}.
\end{aligned}
\end{equation}
Both components, and hence $\xi_{i,k}^{(\mathrm{batch})}$, have zero conditional mean given $(Q',\mathcal{F}_k)$; independence between the two components is not required.

\textbf{We then give the GSPO derivation}, differentiating $s_{j,k}=\rho_{j,k}^{1/|\tau_j|}$ gives
\begin{equation}
\nabla_\theta s_{j,k}
=\frac{s_{j,k}}{|\tau_j|}U_{j,k}(Q')
=\rho_{j,k}\left(\frac{1}{|\tau_j|}\rho_{j,k}^{\frac{1}{|\tau_j|}-1}\right)U_{j,k}(Q').
\end{equation}
Thus GRPO and GSPO share the response-level form
\begin{equation}
\Delta\pi_k(\tau_i|Q)
=\eta\pi_k(\tau_i|Q)\sum_j\pi_k(\tau_j|Q')K_{ij,k}(Q,Q')\bar A_{j,k}^{\mathrm{eff}}r_{j,k}
+\xi_{i,k}^{(\mathrm{batch})},
\end{equation}
with
\begin{equation}
r_{j,k}^{\mathrm{GRPO}}=1,
\qquad
r_{j,k}^{\mathrm{GSPO}}
=\frac{1}{|\tau_j|}\rho_{j,k}^{\frac{1}{|\tau_j|}-1}.
\end{equation}
For GSPO, the same sample/advantage decomposition above is understood with every sampled response contribution multiplied by $r_{j,k}^{\mathrm{GSPO}}$.

\textbf{We finally give the DAPO derivation}, define the token score and token-level importance ratio by
\begin{equation}
u_{j,t,k}(Q')\triangleq
\nabla_\theta\log\pi_\theta(\tau_{j,t}|Q',\tau_{j,<t})\big|_{\theta=\theta_k},
\qquad
\rho_{j,t,k}\triangleq
\frac{\pi_k(\tau_{j,t}|Q',\tau_{j,<t})}
{\pi_{\mathrm{old},k}(\tau_{j,t}|Q',\tau_{j,<t})}.
\end{equation}
For fixed focal trajectory $\tau_i$ and prompts $(Q,Q')$, define the DAPO correction matrix $r_{i,j,k}^{DAPO}$ through a orthogonal condition:
\begin{equation}
\begin{aligned}
U_{i,k}(Q)^\top \left(U_{j,k}(Q')r_{i,j,k}^{\mathrm{DAPO}}
-
\frac{G}{\rho_{j,k}\sum_{j'=1}^{G}|\tau_{j'}|}
\sum_{t=1}^{|\tau_j|}
m_{j,t,k}\rho_{j,t,k}
u_{j,t,k}(Q')\right)\triangleq0
\end{aligned}
\end{equation}
We now prove that all $r_{i,j,k}^{DAPO}$ share a central value $r_{j,k}^{DAPO}$.
\begin{bluebox}
\begin{lemma}[The scalar equivalence of DAPO correction matrix]\leavevmode\\
For the element in DAPO correction matrix $r_{i,j,k}^{DAPO}$, there exists a scalar $r_{j,k}^{DAPO}$ so that:
\begin{equation}
    r_{i,j,k}^{DAPO}=r_{j,k}^{DAPO}+\mathcal O(\eta)
\end{equation}
\end{lemma}
\end{bluebox}
\begin{tcolorbox}[
    enhanced,                         
    breakable,
    title=Proof,               
    colback=cyan!5!white,             
    colframe=black,                   
    colbacktitle=black,               
    coltitle=white,                   
    fonttitle=\bfseries\sffamily,     
    attach boxed title to top left={   
        xshift=3mm, 
        yshift=-2mm, 
        yshifttext=-1mm 
    },
    boxed title style={                
        sharp corners, 
        size=small, 
        colframe=black
    },
    arc=2mm,                          
    boxrule=0.8pt,                    
]
\begin{proof}
Recall that $r_{j,k}^{\mathrm{DAPO}}$ is defined, for each fixed focal
trajectory $\tau_i$, through
\begin{equation}
\begin{aligned}
U_{i,k}(Q)^\top
\left(
U_{j,k}(Q')r_{j,k}^{\mathrm{DAPO}}
-
\frac{G}{\rho_{j,k}Z_k}
\sum_{t=1}^{|\tau_j|}
m_{j,t,k}\rho_{j,t,k}u_{j,t,k}(Q')
\right)
=0.
\end{aligned}
\end{equation}
Since
\begin{equation}
    K_{ij,k}(Q,Q')
    =
    U_{i,k}(Q)^\top U_{j,k}(Q')
    =
    \sum_{t=1}^{|\tau_j|}
    \kappa_{ij,t,k}(Q,Q'),
\end{equation}
we obtain, whenever $K_{ij,k}(Q,Q')\neq0$,
\begin{equation}
    r_{i,j,k}^{\mathrm{DAPO}}
    =
    \frac{G}{\rho_{j,k}Z_k}
    \frac{
        \sum_{t=1}^{|\tau_j|}
        m_{j,t,k}\rho_{j,t,k}
        \kappa_{ij,t,k}(Q,Q')
    }{
        K_{ij,k}(Q,Q')
    }.
\end{equation}
Adding and subtracting
$\sum_t\kappa_{ij,t,k}(Q,Q')=K_{ij,k}(Q,Q')$ in the numerator gives the
exact decomposition
\begin{equation}
\begin{aligned}
    r_{i,j,k}^{\mathrm{DAPO}}
    =
    \frac{G}{\rho_{j,k}Z_k}
    \left[
        1+
        \frac{
            \sum_{t=1}^{|\tau_j|}
            \bigl(m_{j,t,k}\rho_{j,t,k}-1\bigr)
            \kappa_{ij,t,k}(Q,Q')
        }{
            K_{ij,k}(Q,Q')
        }
    \right].
\end{aligned}
\label{eq:dapo-scalar-i-decomposition}
\end{equation}
The token-normalized DAPO update satisfies
\begin{equation}
\begin{aligned}
    \|\Delta\theta_k\|
    &\le
    \frac{\eta}{Z_k}
    \sum_{j=1}^{G}|A_{j,k}|
    \sum_{t=1}^{|\tau_j|}
    m_{j,t,k}\rho_{j,t,k}
    \|u_{j,t,k}(Q')\| \\
    &=O(\eta),
\end{aligned}
\end{equation}
where bounded advantages, token-level ratios, and token-level gradients
have been used. Consequently, a first-order expansion of the
token-level importance ratio gives
\begin{equation}
\begin{aligned}
    \log\rho_{j,t,k}
    &=
    u_{j,t,k}(Q')^\top\Delta\theta_k
    +O(\|\Delta\theta_k\|^2) \\
    &=O(\eta),
\end{aligned}
\end{equation}
and hence
\begin{equation}
    \rho_{j,t,k}=1+O(\eta).
\end{equation}
Because $\eta\ll\epsilon_{\mathrm{clip}}$, the local update does not
cross a clipping boundary. Therefore,
\begin{equation}
    m_{j,t,k}=1,
    \qquad
    m_{j,t,k}\rho_{j,t,k}-1=O(\eta).
\end{equation}

This implies  that
\begin{equation}
\begin{aligned}
&
\left|
\frac{
    \sum_{t=1}^{|\tau_j|}
    \bigl(m_{j,t,k}\rho_{j,t,k}-1\bigr)
    \kappa_{ij,t,k}(Q,Q')
}{
    K_{ij,k}(Q,Q')
}
\right|
\\
&\le
\max_{1\le t\le|\tau_j|}
\left|
    m_{j,t,k}\rho_{j,t,k}-1
\right|
\frac{
    \sum_{t=1}^{|\tau_j|}
    |\kappa_{ij,t,k}(Q,Q')|
}{
    |K_{ij,k}(Q,Q')|
}
\\
&=O(\eta),
\end{aligned}
\end{equation}
uniformly over $i$. Substitution into
Eq.~\eqref{eq:dapo-scalar-i-decomposition} yields
\begin{equation}
    r_{i,j,k}^{\mathrm{DAPO}}
    =
    \frac{G}{\rho_{j,k}Z_k}
    \bigl(1+O(\eta)\bigr)
\end{equation}
uniformly over the focal trajectory.

Define $r_{j,k}^{DAPO}=\frac{G}{\rho_{j,k}Z_k}$, we thus have:
\begin{equation}
    \boxed{
    r_{i,j,k}^{DAPO}=r_{j,k}^{DAPO}+\mathcal O(\eta)
    }
\end{equation}
thus completing the proof.
\end{proof}
\end{tcolorbox}

Indeed,
\begin{equation}
\begin{aligned}
\Delta\pi_k^{\mathrm{DAPO}}(\tau_i|Q)
&=\frac{\eta\pi_k(\tau_i|Q)}{\sum_{j'=1}^{G}|\tau_{j'}|}
\sum_{j=1}^{G}A_{j,k}
\sum_{t=1}^{|\tau_j|}m_{j,t,k}\rho_{j,t,k}
U_{i,k}(Q)^\top u_{j,t,k}(Q')\\
&=\eta\pi_k(\tau_i|Q)\frac{1}{G}
\sum_{j=1}^{G}\rho_{j,k}A_{j,k}
K_{ij,k}(Q,Q')r_{j,k}^{\mathrm{DAPO}} + \mathcal O(\eta^2).
\end{aligned}
\end{equation}
Since we study this dynamics within fixed horizon $T_c\sim\eta^{-1}$, the higher order residue terms can be neglected. Separating $A_{j,k}$ into its conditional mean and zero-mean residual, and separating the finite-group average from its conditional expectation, gives the same population form
\begin{equation}
\Delta\pi_k(\tau_i|Q)
=\eta\pi_k(\tau_i|Q)\sum_j\pi_k(\tau_j|Q')
K_{ij,k}(Q,Q')\bar A_{j,k}^{\mathrm{eff}}r_{j,k}^{\mathrm{DAPO}}
+\xi_{i,k}^{(\mathrm{batch})},
\end{equation}
where $\xi_{i,k}^{(\mathrm{batch})}=\xi_{i,k}^{(\mathrm{sample})}+\xi_{i,k}^{(\mathrm{advantage})}$ and both components have zero conditional mean given $(Q',\mathcal{F}_k)$.
For DAPO in this unified expression, $\bar A_{j,k}^{\mathrm{eff}}$ denotes the conditional mean of $A_{j,k}$ itself, because the token-level clipping masks have already been absorbed into $r_{j,k}^{\mathrm{DAPO}}$.
\hypertarget{thm1:step2}{}
\subsubsection{{Step 2: Deriving the Dynamics of Trajectory Dominance within a Strategy:}}
\hyperlink{back:thm1:step2}{\color{blue}{Back}}
Currently, we haven't applied any analysis on \textbf{Strategies}, so a wise choice is to first focus on the trajectory evolution within a \textbf{specific strategy}. To begin with, let's recall the partitioning method proposed in the Preliminary part.
Throughout this proof, $S=S_{\alpha,Q}$ denotes a fixed equivalence class of the strategy partition induced by the initial parameters $\theta_0$, as specified in Theorem~\ref{theorem:concentration}; hence its membership does not vary with the optimization step $k$.

Based on this partitioning method, we will give a second glance at the derived equation at the end of the previous subsection:
\begin{equation}
\Delta\pi_k(\tau_i|Q) = \eta\cdot\pi_k(\tau_i|Q)\cdot \sum_{j}\pi_k(\tau_j|Q')\cdot K_{ij,k}(Q,Q')\cdot \bar A_{j,k}^{\mathrm{eff}} \cdot r_{j,k} + \xi_{i,k}^{(batch)}
\end{equation}
The updated prompt $Q'$ may change over time, fluctuating the dynamics above. In order to remove this fluctuation, we define:
\begin{equation}
    f_{i,k} = \frac{1}{|\mathcal{D}'|}\sum_{Q'}\sum_j\pi_k(\tau_j|Q')K_{ij,k}(Q,Q')\bar A_{j,k}^{\mathrm{eff}}r_{j,k}
\end{equation}
and define:
\begin{equation}
    \xi_{i,k}^{(data)} =\sum_j\pi_k(\tau_j|Q')K_{ij,k}(Q,Q')\bar A_{j,k}^{\mathrm{eff}}r_{j,k}- \frac{1}{|\mathcal{D}'|}\sum_{Q'}\sum_j\pi_k(\tau_j|Q')K_{ij,k}(Q,Q')\bar A_{j,k}^{\mathrm{eff}}r_{j,k}
\end{equation}
The dynamics can be re-written as:
\begin{equation}
    \Delta\pi_k(\tau_i|Q)=\eta\cdot\pi_k(\tau_i|Q)\cdot f_{i,k} + \eta\cdot\pi_k(\tau_i|Q)\cdot\xi_{i,k}^{(data)} + \xi_{i,k}^{(batch)}
\end{equation}
Based on this, we introduce the relative proportion of a specific trajectory $\tau_i$ in strategy $S$:
\begin{equation}
q_{i|\alpha,Q,k} \triangleq \frac{\pi_k(\tau_i|Q)}{\nu_{\alpha,Q,k}}
\end{equation}
Consider the upgrades in $q_{i|\alpha,Q,k}$ driven by trajectory upgrades, according to chain-rule, we can have:
\begin{equation}
\Delta q_{i|\alpha,Q,k} = \frac{\Delta\pi_k(\tau_i|Q)\cdot \nu_{\alpha,Q,k} - \Delta \nu_{\alpha,Q,k}\cdot\pi_k(\tau_i|Q)}{\nu_{\alpha,Q,k}^2}
\end{equation}
in which,
\begin{equation}
\Delta \nu_{\alpha,Q,k} = \sum_{\tau_\ell\in S_{\alpha,Q}} \Delta\pi_k(\tau_\ell|Q).
\end{equation}
For simplicity, define the global-fitness average within $S$ by
\begin{equation}
f_\alpha\triangleq\sum_{\tau_\ell\in S}q_{\ell|\alpha,Q,k}\cdot f_{\ell,k}.
\end{equation}
The data- and batch-noise terms induced in the normalized proportions are, respectively,
\begin{equation}
\widetilde{\xi}_{i,k}^{(data)}
\triangleq
\eta\cdot q_{i|\alpha,Q,k}\left(
\xi_{i,k}^{(data)}
-\sum_{\tau_\ell\in S}q_{\ell|\alpha,Q,k}\xi_{\ell,k}^{(data)}
\right),
\end{equation}
and
\begin{equation}
\widetilde{\xi}_{i,k}^{(batch)}
\triangleq
\frac{1}{\nu_{\alpha,Q,k}}\left(
\xi_{i,k}^{(batch)}
-q_{i|\alpha,Q,k}\sum_{\tau_\ell\in S}\xi_{\ell,k}^{(batch)}
\right).
\end{equation}
By construction, both normalized noise terms preserve the within-strategy simplex:
\begin{equation}
\sum_{\tau_i\in S}\widetilde{\xi}_{i,k}^{(data)}=0,
\qquad
\sum_{\tau_i\in S}\widetilde{\xi}_{i,k}^{(batch)}=0.
\end{equation}
Under the same conditional-unbiasedness assumptions as in Step 1, including the uniform draw of $Q'$ from $\mathcal D'$, both normalized noise terms also have zero conditional mean given $\mathcal F_k$.

\begin{equation}
    \Delta q_{i|\alpha,Q,k} = \eta\cdot q_{i|\alpha,Q,k}(f_{i,k}-f_\alpha)+\widetilde{\xi}_{i,k}^{(data)}+\widetilde{\xi}_{i,k}^{(batch)}
\end{equation}

This equation has the form of stochastic \textbf{\textit{replicator-like dynamics}}.

\begin{remark}
The evolution advantage $f_{i,k}$ has become less direct after such complicated calculation, so it is necessary to make some explanations on its property, and what does it mean for the total training process: 

To begin with, if $f_{i,k}>f_\alpha$, the trajectory $\tau_i$ receives a positive incentive, meaning the model is actively adopting such expressions, and the higher the $f_{i,k}$, the more the model is willing to \textit{learn} such expressions.

On the contrary, if $f_{i,k}<f_\alpha$, the trajectory $\tau_i$ receives a negative incentive, meaning the model is actively abandoning such expressions, which in our bodytext is named as \textbf{expression pruning}.

Another thing to mention is that, the calculation of $f_{i,k}$ is not restricted to the current prompt; rather, it is computed on the whole dataset. This could lead to a seemingly weird but actually reasonable conclusion: for those \textcolor{red}{\textbf{less general yet correct}} expressions/strategies, they could receive a  \textcolor{red}{\textbf{negative incentive}}, making the model less likely to study such expressions.

Consequently, there exists a trade off: for a given expression/strategy, if it is both correct, useful and generalized (e.g. brute computation), the model would be very much willing to learn it; on the contrary, if a given expression/strategy is correct, but hard and field-specific. The model would tend to give up on this expression/strategy, and in turn pay more attention to more friendly expressions or strategies.
\end{remark}
\hypertarget{thm1:step3}{}
\subsubsection{{Step 3: Adiabatic Approximation}}
\label{theorem1:step3}
\hyperlink{back:thm1:step3}{\color{blue}{Back}} In this section, we revisit the expression of the evolution advantage $f_i$ proposed in \textbf{Step 2}. Fix an optimization step $k$ and a strategy $S=S_{\alpha,Q}$. The key observation is that the replicator dynamics only depends on the relative fitness \(f_i-f_\alpha\): large common components of \(f_i\) are canceled by the baseline \(f_\alpha\), leaving a residual that can evolve on a slower scale. This motivates the adiabatic separation developed next. To make this structure explicit, we decompose the trajectory probability as
\begin{equation}
\pi_k(\tau_i|Q)
=
\nu_{\alpha,Q,k}q_{i\mid\alpha,Q,k},
\qquad \tau_i\in S_{\alpha,Q},
\end{equation}
where $\nu_{\alpha,Q,k}$ denotes the total probability mass of strategy $S_{\alpha,Q}$, and
\begin{equation}
q_{i\mid\alpha,Q,k}
\triangleq
\frac{\pi_k(\tau_i|Q)}
{\nu_{\alpha,Q,k}}
\end{equation}
denotes the corresponding in-strategy distribution.

Accordingly, instead of directly proving that $f_i$ itself evolves slowly, we establish a separation of time scales between the strategy-level mass $\nu_{\alpha,Q,k}$ and the microscopic quantities governing $f_i$. Specifically, we will show that $K_{ij}(Q,Q')$, $A_i$, $r_i$, and the in-strategy distribution $q_{i\mid\alpha,Q,k}$ all evolve much more slowly than the strategy-level dynamics.

Denote $\omega_K\triangleq|\frac{\Delta K_{ij}(Q,Q')}{K_{ij}(Q,Q')}|$, $\omega_A\triangleq|\frac{\Delta \bar{A}^{eff}_i}{\bar{A}^{eff}_i}|$, $\omega_r\triangleq|\frac{\Delta r_i}{r_i}|$, and
\begin{equation}
\omega_q
\triangleq
\left(
\sum_{\tau_i\in S_{\alpha,Q}}
q_{i\mid\alpha,Q,k}
\left|
\frac{\Delta q_{i\mid\alpha,Q,k}}
{q_{i\mid\alpha,Q,k}}
\right|^2
\right)^{1/2},
\end{equation}
We further note that the trajectory-level and strategy-level dynamics evolve on the same leading-order time scale. From the trajectory dynamics,
\begin{equation}
\omega_\pi
\triangleq
\left|
\frac{\Delta\pi_i}{\pi_i}
\right|
\sim
|f_i|,
\end{equation}
while the strategy-level dynamics derived above gives
\begin{equation}
\omega_\nu
\triangleq
\left|
\frac{\Delta\nu_{\alpha,Q,k}}{\nu_{\alpha,Q,k}}
\right|
\sim
|f_S|.
\end{equation}
For any $\tau_i\in S$, the geometric coherence of trajectories within the same strategy implies
\begin{equation}
|f_i-f_S|
=
O(\epsilon  B).
\end{equation}
Hence we may write
\begin{equation}
f_i
=
f_S+\delta_i,
\qquad
|\delta_i|
=
O(\epsilon  B).
\end{equation}
Away from turning points where $|f_S|$ vanishes, with $|f_S|=O(B)$, we obtain
\begin{equation}
\begin{aligned}
\frac{\omega\pi}{\omega_\nu}
&=
\frac{|f_i|}{|f_S|}\
&=
\left|
1+\frac{\delta_i}{f_S}
\right|\
&=
1+O(\epsilon ).
\end{aligned}
\end{equation}
Therefore,
\begin{equation}
\boxed{
\omega_\pi
\sim
\omega_\nu
}.
\end{equation}
Thus, any quantity evolving much more slowly than $\pi_k(\tau_i|Q)$ also evolves much more slowly than the strategy mass $\nu_{\alpha,Q,k}$, up to the same local adiabatic approximation.

To prove that $\omega_K$, $\omega_A$, $\omega_q$, and $\omega_r$ are all slow variables relative to the strategy-level evolution, we first identify that $\omega_K\ll\omega_\pi$, $\omega_A\ll\omega_\pi$, $\omega_r\ll\omega_\pi$. We will start from proving $\omega_K\ll\omega_\pi$:
\begin{orangebox}
\begin{lemma}[The Bounds of log-prob gradients]~\label{lemma:bounds}\leavevmode\\
Fix a finite training window $k\in[k_1,k_2]$. Then
\begin{equation}
    \exists m,M>0,\quad m|\tau|\le\|U_{i,k}\|\le M|\tau|,\quad \text{a.s.}
\end{equation}
throughout this window.
\end{lemma}
\end{orangebox}

\begin{tcolorbox}[
    enhanced,                         
    breakable,
    title=Proof,               
    colback=cyan!5!white,             
    colframe=black,                   
    colbacktitle=black,               
    coltitle=white,                   
    fonttitle=\bfseries\sffamily,     
    attach boxed title to top left={   
        xshift=3mm, 
        yshift=-2mm, 
        yshifttext=-1mm 
    },
    boxed title style={                
        sharp corners, 
        size=small, 
        colframe=black
    },
    arc=2mm,                          
    boxrule=0.8pt,                    
]
To establish the relative rates of change between strategy probabilities and the advantage landscape, we first analyze the magnitude of the score function (log-gradient), $\nabla_\theta \log \pi_\theta(\tau|Q)$, with respect to the trajectory length $|\tau|$. We consider one extreme scenario to bound its scaling behavior.

\textbf{Upper Bound} \\
    In the presence of strong dependencies across the sequence, each conditional log-probability gradient may be bounded by a constant $M$ such that $\|\nabla_\theta \log \pi_\theta(t_i | \text{prefix})\| \le M$. Under these conditions, the gradients accumulate linearly along the trajectory:
    \begin{equation}
        \|U_{i,k}\| \le \sum_{i=1}^{|\tau|} \left\| \nabla_\theta \log \pi_\theta(t_i | \text{prefix}) \right\| \le |\tau| \cdot M
    \end{equation}
    Therefore, the upper bound for the log-gradient magnitude scales as:
    \begin{equation}
        \|\nabla_\theta \log \pi_\theta(\tau|Q)\| = \mathcal{O}(|\tau|)
    \end{equation}

\textbf{Lower Bound}\\
We first exclude the degenerate case \(U_{i,k}=0\). For any fixed
trajectory, \(U_{i,k}\) is a smooth score function of the model parameters.
Unless this score function vanishes identically, its zero level set is a
Lebesgue-null set under generic continuous model parameters. Hence,
\begin{equation}
    U_{i,k}\neq 0
    \qquad \text{a.s.}
\end{equation}
For the finite training window \(k\in[k_1,k_2]\), a finite union of such
null events remains null. Moreover, since the trajectory space is
countable, we may choose one trajectory from each nonempty strategy, and
the union of the corresponding null events is still null. Therefore,
almost surely, every strategy contains a trajectory satisfying
\(U_{i,k}\neq 0\) throughout the whole window.

Because the training window is finite and $|\tau|>0$, the finite sequence
$\{\|U_{i,k}\|/|\tau|\}_{k=k_1}^{k_2}$ consists entirely of strictly
positive terms almost surely.  It therefore attains an almost-surely
positive minimum.  Define
\begin{equation}
    m
    \triangleq
    \min_{k\in[k_1,k_2]}
    \frac{\|U_{i,k}\|}{|\tau|}
    >0
    \qquad \text{a.s.}
\end{equation}
By construction, this gives the uniform lower bound
\begin{equation}
    \|U_{i,k}\|\ge m|\tau|,
    \qquad k\in[k_1,k_2].
\end{equation}

We next propagate this lower bound to the remaining trajectories in the
same strategy. By the definition of dynamic equivalence,
\begin{equation}
    |K_{ii}-K_{ji}|
    \le
    \epsilon\sup_{a,b}|K_{ab}|,
    \qquad
    |K_{ij}-K_{jj}|
    \le
    \epsilon\sup_{a,b}|K_{ab}|.
\end{equation}
Therefore,
\begin{equation}
\begin{aligned}
    \|U_{i,k}-U_{j,k}\|^2
    &=
    K_{ii}+K_{jj}-2K_{ij} \\
    &\le
    2\epsilon\sup_{a,b}|K_{ab}|
    \le
    2\epsilon M^2|\tau|^2,
\end{aligned}
\end{equation}
and hence
\begin{equation}
    \|U_{j,k}\|
    \ge
    \|U_{i,k}\|-\|U_{i,k}-U_{j,k}\|
    \ge
    \left(m-\sqrt{2\epsilon}M\right)|\tau|.
\end{equation}
For sufficiently small fixed \(\epsilon\), the coefficient remains
strictly positive. Together with the upper bound
\(\|U_{i,k}\|\le M|\tau|\), we obtain
\begin{equation}
    \boxed{\|U_{i,k}\|=\Theta(|\tau|)\qquad\text{a.s.}}
\end{equation}
throughout the finite training window.

Thus, the score-function magnitude has linear scale in trajectory length
almost surely, which yields the scaling required for the time-scale
separation in Proposition~3.
\end{tcolorbox}

This is a \textbf{local-in-time} scaling statement, rather than a uniform claim
over the entire training process. The constants $m,M$ may depend on the
chosen training window, but are independent of $|\tau|$. The upper bound
follows from bounded token-level score norms, while the lower bound only
requires that their accumulated score retains a non-vanishing projection
along some direction, excluding asymptotically exact cancellation.
Thus, $\|U_{i,k}\|=\Theta(|\tau|)$ describes scaling with trajectory length
within a fixed window, not monotonic growth over training.

Based on this complexity analysis, we propose the following lemma regarding the relative rates of change, under the condition of $|\tau|\gg1$:

\begin{yellowbox}
\begin{lemma}[Adiabatic Property of Kernel]~\label{lemma:K_pi}\leavevmode\\
Let $\omega_\pi$ denote the relative rate of change in trajectory probability, and $\omega_K$ denote the relative rate of change in the advantage kernel matrix elements $K_{ij}$. Under standard training regimes, we have:
\begin{equation}
\frac{\omega_K}{\omega_\pi} \ll 1
\end{equation}
\end{lemma}
\end{yellowbox}

\begin{tcolorbox}[
    enhanced,                         
    breakable,
    title=Proof,               
    colback=cyan!5!white,             
    colframe=black,                   
    colbacktitle=black,               
    coltitle=white,                   
    fonttitle=\bfseries\sffamily,     
    attach boxed title to top left={   
        xshift=3mm, 
        yshift=-2mm, 
        yshifttext=-1mm 
    },
    boxed title style={                
        sharp corners, 
        size=small, 
        colframe=black
    },
    arc=2mm,                          
    boxrule=0.8pt,                    
]
\begin{proof}
For an autoregressive trajectory $\tau_i$ of length $|\tau|$,
\begin{equation}
    U_i
    =
    \sum_{t=1}^{|\tau|}
    \nabla_\theta
    \log\pi_\theta(\tau_{i,t}\mid Q,\tau_{i,<t}).
\end{equation}
Assuming uniformly bounded token-level gradients and Hessians,
\begin{equation}
    \|U_i\| = O(|\tau|),
    \qquad
    \left\|
    \nabla_\theta U_i
    \right\|_{\mathrm{op}}
    = O(|\tau|).
\end{equation}

Let $\Delta\theta$ denote one optimization update.
For trajectories actively affected by the update, we assume that the
score does not become asymptotically orthogonal to the update direction,
so that
\begin{equation}
    |U_i^\top\Delta\theta|
    =
    \Theta(|\tau|\|\Delta\theta\|).
\end{equation}
Using a first-order Taylor expansion,
\begin{equation}
    \Delta\log\pi_\theta(\tau_i|Q)
    =
    U_i^\top\Delta\theta
    +
    O(|\tau|\|\Delta\theta\|^2).
\end{equation}
Under the usual small-update condition
$|\tau|\|\Delta\theta\|\ll1$, this gives
\begin{equation}
    \frac{\Delta\pi_\theta(\tau_i|Q)}{\pi_\theta(\tau_i|Q)}
    =
    \exp\!\left(\Delta\log\pi_\theta(\tau_i|Q)\right)-1
    =
    \Delta\log\pi_\theta(\tau_i|Q)
    +
    O\!\left(\left(\Delta\log\pi_\theta(\tau_i|Q)\right)^2\right).
\end{equation}
Therefore,
\begin{equation}
    \boxed{
    \omega_{\pi,k}
    =
    \Theta(|\tau|\|\Delta\theta\|).
    }
\end{equation}

We next consider the kernel
\begin{equation}
    K_{ij}=U_i^\top U_j.
\end{equation}
Since
\begin{equation}
    \Delta U_i
    =
    O(|\tau|\|\Delta\theta\|),
\end{equation}
its one-step variation satisfies
\begin{align}
    \Delta K_{ij}
    &=
    (\Delta U_i)^\top U_j
    +
    U_i^\top(\Delta U_j)
    +
    (\Delta U_i)^\top(\Delta U_j)
    \nonumber\\
    &=
    O(|\tau|^2\|\Delta\theta\|).
\end{align}

For active kernel entries whose magnitude is of the natural scale
$|K_{ij}|=\Omega(|\tau|^2)$, we therefore obtain
\begin{equation}
    \boxed{
    \omega_{K,k}
    =
    \left|
    \frac{\Delta K_{ij}}{K_{ij}}
    \right|
    =
    O(\|\Delta\theta\|).
    }
\end{equation}

Combining the two estimates yields
\begin{equation}
    \boxed{
    \frac{\omega_{K,k}}{\omega_{\pi,k}}
    =
    O\left(\frac{1}{|\tau|}\right).
    }
\end{equation}
Hence, for long reasoning trajectories, the kernel evolves on a much
slower time scale than the trajectory probabilities and can be treated
as approximately frozen during strategy redistribution.
\end{proof}
\end{tcolorbox}
Lemma~\ref{lemma:K_pi} holds under the requirement $\eta\ll1\ll|\tau|$. We now proceeds to prove that under the requirement of $\eta|\tau|\ll1$, $\omega_r\ll\omega_\pi$:
\begin{bluebox}
\begin{lemma}[Adiabatic Property of Correction Factor]\leavevmode\\
\label{adiabatic:rho}
The relative changing rate of the correction factor satisfies:
\begin{equation}
\frac{\omega_r}{\omega_\pi}\ll1    
\end{equation}
\end{lemma}
\end{bluebox}
\begin{tcolorbox}[
    enhanced,                         
    breakable,
    title=Proof,               
    colback=cyan!5!white,             
    colframe=black,                   
    colbacktitle=black,               
    coltitle=white,                   
    fonttitle=\bfseries\sffamily,     
    attach boxed title to top left={   
        xshift=3mm, 
        yshift=-2mm, 
        yshifttext=-1mm 
    },
    boxed title style={                
        sharp corners, 
        size=small, 
        colframe=black
    },
    arc=2mm,                          
    boxrule=0.8pt,                    
]
From previous analysis, we already obtained the evolution dynamics of $\pi_\theta(\tau_i|Q)$:
\begin{equation}
    \Delta\pi_\theta(\tau_i|Q) = \eta\cdot\pi_\theta(\tau_i|Q)\cdot f_i + \xi_i^{(batch)} + \xi_i^{(data)}
\end{equation}
This means that we can directly obtain the relative changing rate $\omega_\pi$:
\begin{equation}
    \omega_\pi\sim\eta\cdot f_i\sim\mathcal{O}(\eta)
\end{equation}
Because $r_i$ is a continuous function of $\rho_i$, and $\omega_\rho=\Delta\log\rho_i$, we consider $\Delta\log\rho_i$. To begin with, we first consider the dynamical property of $\rho_i$:
\begin{equation}
    \log\rho_i=\Delta\log\pi_\theta(\tau|Q)=\nabla_\theta\log\pi_\theta(\tau|Q)^\top\cdot\Delta\theta
\end{equation}
Therefore, doing the derivatives on the equation again, we have:
\begin{equation}
    \Delta\log\rho_i = \Delta\theta^\top\nabla_\theta^2\log\pi_\theta(\tau_i|Q)\Delta\theta
\end{equation}
According to previous Lemma~\ref{lemma:bounds}, $\nabla_\theta^2\log\pi_\theta(\tau_i|Q)\sim\mathcal{O}(|\tau|)$, $\Delta\theta\sim\mathcal{O}(\eta)$, therefore:
\begin{equation}
    \Delta\log\rho_i\sim\mathcal{O}(\eta^2|\tau|)
\end{equation}
Therefore, divide $\omega_r$ by $\omega_\pi$, we have:
\begin{equation}
    \frac{\omega_r}{\omega\pi}\sim\mathcal{O}(\eta|\tau|)
\end{equation}
When the experiment configuration satisfies: $\eta|\tau|\ll1$, we have:
\begin{equation}
    \frac{\omega_\rho}{\omega_\pi}\ll1
\end{equation}
Since $r=f(\rho_i)$, we have: $\omega_r=C\cdot\omega_\rho$, in which $C$ is a near constant value compared with $|\tau|$, we have:
\begin{equation}
    \frac{\omega_r}{\omega_\pi}\ll1
\end{equation}
\end{tcolorbox}
\noindent\textbf{Why the correction factor varies at second order.}
Because the old policy is refreshed at every step,
$\pi_{\mathrm{old},k}=\pi_{k-1}$, we have
\begin{equation}
\rho_{i,k}
=
\frac{\pi_k(\tau_i\mid Q)}
     {\pi_{k-1}(\tau_i\mid Q)},
\qquad
\log\rho_{i,k}
=
\log\pi_k(\tau_i\mid Q)
-
\log\pi_{k-1}(\tau_i\mid Q).
\end{equation}
Hence $\log\rho_{i,k}$ is already the first discrete difference of the
log-policy, and its temporal variation is the second difference:
\begin{equation}
\Delta\log\rho_{i,k}
=
\log\pi_{k+1}(\tau_i\mid Q)
-2\log\pi_k(\tau_i\mid Q)
+\log\pi_{k-1}(\tau_i\mid Q).
\end{equation}
Since each parameter update is $\mathcal O(\eta)$ and the score field varies
smoothly across consecutive steps, the common first-order
$\mathcal O(\eta)$ increment cancels, leaving
\begin{equation}
\Delta\log\rho_{i,k}
=
\mathcal O(\eta^2|\tau|).
\end{equation}
Thus the correction factor evolves on a slower scale than the policy itself.

The previous lemma works for GRPO ($\omega_r=0$) and GSPO ($\omega_r\ll\omega_\pi$). However, directly analyzing $\omega_r$ in DAPO optimization is too hard, as it contains too many non-convex surrogates. We instead seek another approach to validate the time-scale separation in DAPO:
\begin{bluebox}
\begin{lemma}[Adiabatic Property of the DAPO Effective Correction]\leavevmode\\
For DAPO, define the mixed token-level kernel
\begin{equation}
    \kappa_{ij,t,k}(Q,Q')
    \triangleq
    U_{i,k}(Q)^\top u_{j,t,k}(Q'),
\end{equation}
and the effective correction term
\begin{equation}
    \mathcal{R}_{ij,k}^{\mathrm{DAPO}}
    \triangleq
    \frac{G}{Z_k}
    \sum_{t=1}^{|\tau_j|}
    m_{j,t,k}\rho_{j,t,k}\kappa_{ij,t,k},
    \qquad
    Z_k\triangleq
    \sum_{\ell=1}^{G}{|\tau_\ell|}.
\end{equation}
Within an optimization block with fixed clipping active set, assuming bounded
token-level gradients and Hessians and excluding asymptotically degenerate
cancellation among active token contributions, the relative changing rate of
$\mathcal{R}_{ij,k}^{\mathrm{DAPO}}$ satisfies
\begin{equation}
    \boxed{
    \frac{\omega_{\mathcal{R}}}{\omega_\pi}
    =
    O\left(\frac{1}{|\tau|}\right)
    \ll 1
    }.
\end{equation}
Moreover, the DAPO population term admits the exact identity
\begin{equation}
    \pi_k(\tau_j|Q')
    K_{ij,k}(Q,Q')
    r_{j,k}^{\mathrm{DAPO}}
    =
    \pi_{\mathrm{old},k}(\tau_j|Q')
    \mathcal{R}_{ij,k}^{\mathrm{DAPO}}.
\end{equation}
Hence, although the raw factor $r_{j,k}^{\mathrm{DAPO}}$ need not itself be
adiabatic, the effective DAPO correction entering the population dynamics
evolves on a much slower time scale than the trajectory probability.
\end{lemma}
\end{bluebox}

\begin{tcolorbox}[
    enhanced,                         
    breakable,
    title=Proof,               
    colback=cyan!5!white,             
    colframe=black,                   
    colbacktitle=black,               
    coltitle=white,                   
    fonttitle=\bfseries\sffamily,     
    attach boxed title to top left={   
        xshift=3mm, 
        yshift=-2mm, 
        yshifttext=-1mm 
    },
    boxed title style={                
        sharp corners, 
        size=small, 
        colframe=black
    },
    arc=2mm,                          
    boxrule=0.8pt,                    
]
\begin{proof}
From the definition of the DAPO correction factor,
\begin{equation}
    r_{j,k}^{\mathrm{DAPO}}
    =
    \frac{G}
    {\rho_{j,k}Z_k}
    \frac{
    \sum_{t=1}^{|\tau_j|}
    m_{j,t,k}\rho_{j,t,k}
    U_{i,k}(Q)^\top u_{j,t,k}(Q')
    }
    {K_{ij,k}(Q,Q')}.
\end{equation}
Multiplying both sides by
$\pi_k(\tau_j|Q')K_{ij,k}(Q,Q')$ gives
\begin{equation}
\begin{aligned}
    &
    \pi_k(\tau_j|Q')
    K_{ij,k}(Q,Q')
    r_{j,k}^{\mathrm{DAPO}}
    \\
    &=
    \frac{G}{Z_k}
    \frac{\pi_k(\tau_j|Q')}{\rho_{j,k}}
    \sum_{t=1}^{|\tau_j|}
    m_{j,t,k}\rho_{j,t,k}
    \kappa_{ij,t,k}(Q,Q').
\end{aligned}
\end{equation}
Since the sequence-level importance ratio satisfies
\begin{equation}
    \rho_{j,k}
    =
    \frac{
    \pi_k(\tau_j|Q')
    }{
    \pi_{\mathrm{old},k}(\tau_j|Q')
    },
\end{equation}
we obtain the exact cancellation
\begin{equation}
    \boxed{
    \pi_k(\tau_j|Q')
    K_{ij,k}(Q,Q')
    r_{j,k}^{\mathrm{DAPO}}
    =
    \pi_{\mathrm{old},k}(\tau_j|Q')
    \mathcal{R}_{ij,k}^{\mathrm{DAPO}}
    }.
\end{equation}
Therefore, it suffices to study the time scale of
$\mathcal{R}_{ij,k}^{\mathrm{DAPO}}$.

We first consider the token-level importance ratio
\begin{equation}
    \rho_{j,t,k}
    =
    \frac{
    \pi_{\theta_k}
    (\tau_{j,t}\mid Q',\tau_{j,<t})
    }{
    \pi_{\mathrm{old},k}
    (\tau_{j,t}\mid Q',\tau_{j,<t})
    }.
\end{equation}
Within one optimization block, the denominator is fixed. Hence, using a
first-order Taylor expansion,
\begin{equation}
\begin{aligned}
    \Delta\log\rho_{j,t,k}
    &=
    u_{j,t,k}(Q')^\top\Delta\theta
    +
    O(\|\Delta\theta\|^2)
    \\
    &=
    O(\|\Delta\theta\|),
\end{aligned}
\end{equation}
where bounded token-level gradients have been used. Therefore,
\begin{equation}
    \frac{\Delta\rho_{j,t,k}}
    {\rho_{j,t,k}}
    =
    O(\|\Delta\theta\|).
\end{equation}

We next consider the mixed token-level kernel
\begin{equation}
    \kappa_{ij,t,k}
    =
    U_{i,k}^\top u_{j,t,k}.
\end{equation}
By the bounded-gradient and bounded-Hessian assumptions used in
Lemma~2,
\begin{equation}
    \|U_{i,k}\|=O(|\tau|),
    \qquad
    \|\Delta U_{i,k}\|
    =
    O(|\tau|\|\Delta\theta\|),
\end{equation}
while for a single token,
\begin{equation}
    \|u_{j,t,k}\|=O(1),
    \qquad
    \|\Delta u_{j,t,k}\|
    =
    O(\|\Delta\theta\|).
\end{equation}
Consequently,
\begin{equation}
\begin{aligned}
    \Delta\kappa_{ij,t,k}
    &=
    (\Delta U_{i,k})^\top u_{j,t,k}
    +
    U_{i,k}^\top(\Delta u_{j,t,k})
    +
    O(\|\Delta\theta\|^2)
    \\
    &=
    O(|\tau|\|\Delta\theta\|).
\end{aligned}
\end{equation}
For active mixed-kernel entries of their natural scale
$|\kappa_{ij,t,k}|=\Omega(|\tau|)$, this yields
\begin{equation}
    \frac{
    \Delta\kappa_{ij,t,k}
    }{
    \kappa_{ij,t,k}
    }
    =
    O(\|\Delta\theta\|).
\end{equation}

Now define each active token contribution by
\begin{equation}
    x_{j,t,k}
    \triangleq
    m_{j,t,k}
    \rho_{j,t,k}
    \kappa_{ij,t,k}.
\end{equation}
On any interval where the clipping active set is fixed,
$\Delta m_{j,t,k}=0$. Hence,
\begin{equation}
\begin{aligned}
    \Delta x_{j,t,k}
    &=
    m_{j,t,k}
    \left(
    \kappa_{ij,t,k}\Delta\rho_{j,t,k}
    +
    \rho_{j,t,k}\Delta\kappa_{ij,t,k}
    \right)
    +
    O(\|\Delta\theta\|^2)
    \\
    &=
    x_{j,t,k}
    O(\|\Delta\theta\|).
\end{aligned}
\end{equation}
Since
\begin{equation}
    \mathcal{R}_{ij,k}^{\mathrm{DAPO}}
    =
    \frac{G}{Z_k}
    \sum_t x_{j,t,k},
\end{equation}
and $G/Z_k$ is fixed inside the optimization block, under the
non-degenerate-cancellation condition
\begin{equation}
    \left|
    \sum_t x_{j,t,k}
    \right|
    =
    \Omega\left(
    \sum_t |x_{j,t,k}|
    \right),
\end{equation}
we obtain
\begin{equation}
    \boxed{
    \omega_{\mathcal{R}}
    =
    \left|
    \frac{
    \Delta\mathcal{R}_{ij,k}^{\mathrm{DAPO}}
    }{
    \mathcal{R}_{ij,k}^{\mathrm{DAPO}}
    }
    \right|
    =
    O(\|\Delta\theta\|)
    }.
\end{equation}

From Lemma~2, the trajectory probability evolves at the rate
\begin{equation}
    \omega_\pi
    =
    \Theta(|\tau|\|\Delta\theta\|).
\end{equation}
Combining the two estimates therefore gives
\begin{equation}
    \boxed{
    \frac{\omega_{\mathcal{R}}}{\omega_\pi}
    =
    O\left(\frac{1}{|\tau|}\right)
    \ll 1
    }.
\end{equation}
Thus, after the exact cancellation of the sequence-level importance ratio,
the remaining DAPO correction consists only of token-level importance ratios
and mixed token-level kernels, both of which evolve adiabatically relative
to the trajectory probability.
\end{proof}
\end{tcolorbox}
Although DAPO admits a scalarized correction at first order, the temporal
variation of its higher-order residual may not be benignly  controlled. We therefore analyze
the exact correction entering the population dynamics, shown in the lemma.

We then prove that: under requirements $\eta\ll\epsilon<1,\eta\cdot|\tau|\ll1$, the following lemma stands:
\begin{bluebox}
\begin{lemma}[Adiabatic Property of Clipped Reward Signal]\leavevmode\\
The relative changing rate of the clipping boundary satisfies:
\begin{equation}
\frac{\omega_A}{\omega_\pi}\ll1    
\end{equation}
\end{lemma}
\end{bluebox}
\begin{tcolorbox}[
    enhanced,                         
    breakable,
    title=Proof,               
    colback=cyan!5!white,             
    colframe=black,                   
    colbacktitle=black,               
    coltitle=white,                   
    fonttitle=\bfseries\sffamily,     
    attach boxed title to top left={   
        xshift=3mm, 
        yshift=-2mm, 
        yshifttext=-1mm 
    },
    boxed title style={                
        sharp corners, 
        size=small, 
        colframe=black
    },
    arc=2mm,                          
    boxrule=0.8pt,                    
]
To prove this, we first consider the clipping effect, for trajectories that can produce normal gradients, they satisfy:
\begin{equation}
    1-\epsilon_{low}\le\rho_i\le1+\epsilon_{high}
\end{equation}
Or equivalently, to first order approximation:
\begin{equation}
    -\epsilon_{low}\le\log\rho_i\le\epsilon_{high}
\end{equation}
Given that $\epsilon$ is still an observable amount compared with $1$, therefore:
\begin{equation}
\eta|\tau|\ll\min\{\epsilon_{low},\epsilon_{high}\}
\end{equation}
Therefore, we claim that:
\begin{equation}
    \Delta\log\rho_i\ll\min\{\epsilon_{low},\epsilon_{high}\}
\end{equation}
This means that, during the update step, only a $\mathcal{O}(\eta^2|\tau|)$ proportion of trajectories will undergo the clipping transition; therefore, the relative change rate of the reward signal $A_i$ follows:
\begin{equation}
    \omega_A
    =
    \left|\frac{\Delta A_i}{A_i}\right|
    <
    \mathcal{O}(\eta^2|\tau|)
    \cdot
    \left|\frac{A_{\mathrm{clipped}}}{A_{\min}}\right|
    =
    \mathcal{O}(\eta^2|\tau|).
\end{equation}
Therefore, we can still conclude:
\begin{equation}
    \frac{\omega_A}{\omega_\pi}\sim\mathcal{O}(\eta|\tau|)\ll1
\end{equation}
\end{tcolorbox}

After proving the four lemmas above, we can already control the last three terms $\omega_K,\omega_A,\omega_r$. We now turn to the last term $q$. To analyze it, we decompose each trajectory probability into its strategy-level mass and its in-strategy distribution, and show that the latter evolves on a much slower time scale due to the geometric coherence of trajectories within the same \textbf{Strategy}. This establishes the adiabaticity of the in-strategy distribution and allows us to locally freeze it in the subsequent derivation.
\begin{bluebox}
\begin{lemma}[Adiabatic Property of Strategies]
\label{lemma:adiabatic_property_of_strategies}
For the fixed strategy $S=S_{\alpha,Q}$, let its relative changing rate be $\omega_\nu\triangleq\frac{\Delta\nu_{\alpha,Q,k}}{\nu_{\alpha,Q,k}}$. Then:
\begin{equation}
    \frac{\omega_\nu}{\omega_q}\gg1
\end{equation}
\end{lemma}
\end{bluebox}

To begin with, we first prove a useful lemma regarding $f_i$:
\begin{bluebox}
\begin{lemma}[$f_i$ is unbiased]\leavevmode\\
For the $f_i$ definition given, we have:
\begin{equation}
    \sum_i\pi_\theta(\tau_i|Q)\cdot f_i=0
\end{equation}
\end{lemma}
\end{bluebox}
\begin{tcolorbox}[
    enhanced,                         
    breakable,
    title=Proof,               
    colback=cyan!5!white,             
    colframe=black,                   
    colbacktitle=black,               
    coltitle=white,                   
    fonttitle=\bfseries\sffamily,     
    attach boxed title to top left={   
        xshift=3mm, 
        yshift=-2mm, 
        yshifttext=-1mm 
    },
    boxed title style={                
        sharp corners, 
        size=small, 
        colframe=black
    },
    arc=2mm,                          
    boxrule=0.8pt,                    
]
By definition, we can expand the summation above:
\begin{equation}
    \sum_i\pi_\theta(\tau_i|Q)\cdot f_i=\sum_{i,j}\pi_\theta(\tau_i|Q)\cdot K_{ij}(Q,Q')\cdot\pi_\theta(\tau_j|Q')\cdot A_j\cdot r_j
\end{equation}
Denote:
\begin{equation}
    U_i(Q)\triangleq\nabla_\theta\log(\tau_i|Q)
\end{equation}
then $K_{ij}(Q,Q')=U_i(Q)^\top\cdot U_j(Q)$, also it's obvious that:
\begin{equation}
    \sum_i\pi_\theta(\tau_i|Q)\cdot U_i(Q)=0
\end{equation}
Therefore:
\begin{equation}
    \sum_i\pi_\theta(\tau_i|Q)\cdot f_i=\sum_{j}\big(\sum_i\pi_\theta(\tau_i|Q)\cdot U_i(Q)^\top\big)\cdot U_j(Q')\cdot\pi_\theta(\tau_j|Q')\cdot A_j\cdot r_j=0
\end{equation}
Therefore we conclude:
\begin{equation}
    \sum_i\pi_\theta(\tau_i|Q)\cdot f_i=0
\end{equation}
\end{tcolorbox}
Using this beautiful lemma, we can rewrite the dynamics into an anti-symmetric form:
\begin{equation}
\Delta\pi_k(\tau_i|Q)=\eta\cdot\pi_k(\tau_i|Q)\cdot f_{i,k}+\xi_{i,k}=\eta\cdot\pi_k(\tau_i|Q)\cdot\sum_j\pi_k(\tau_j|Q)\cdot f_{i,k}+\xi_{i,k}
\end{equation}
Using $\sum_j\pi_k(\tau_j|Q)\cdot f_{j,k}=0$, we have:
\begin{equation}
\Delta\pi_k(\tau_i|Q)=\eta\cdot\sum_{i,j}\pi_k(\tau_i|Q)\cdot\pi_k(\tau_j|Q)\cdot(f_{i,k}-f_{j,k}) + \xi_{i,k}
\end{equation}
Let $J_{ij,k}\triangleq\pi_k(\tau_i|Q)\cdot\pi_k(\tau_j|Q)\cdot(f_{i,k}-f_{j,k})$, we have:
\begin{equation}
    \boxed{J_{ij,k}=-J_{ji,k}}
\end{equation}
For the fixed strategy $S=S_{\alpha,Q}$, expand $\Delta\nu_{\alpha,Q,k}$:
\begin{equation}
\Delta\nu_{\alpha,Q,k}=\sum_{\tau_i\in S_{\alpha,Q}}\Delta\pi_k(\tau_i|Q)
\end{equation}
Substitute $\Delta\pi_k(\tau_i|Q)$ with the dynamics derived before, we have:
\begin{equation}
\Delta\nu_{\alpha,Q,k}=\eta\cdot\sum_{\tau_i\in S_{\alpha,Q}}\sum_{\tau_j}J_{ij,k} + \xi_{\alpha,Q,k}
\end{equation}
Then split the right hand side into an intra-strategy term and inter-strategy term:
\begin{equation}
    \Delta\nu_{\alpha,Q,k}=\eta\cdot\sum_{\tau_i\in S_{\alpha,Q}}\sum_{\tau_j \in S_{\alpha,Q}}J_{ij,k} + \eta\cdot\sum_{\tau_i\in S_{\alpha,Q}}\sum_{\tau_j \notin S_{\alpha,Q}}J_{ij,k} + \xi_{\alpha,Q,k}
\end{equation}
Because of the anti-symmetric property of $J_{ij,k}$, the first term will vanish:
\begin{equation}
    \sum_{\tau_i\in S_{\alpha,Q}}\sum_{\tau_j\in S_{\alpha,Q}}J_{ij,k}=0
\end{equation}
Therefore, we have:
\begin{equation}
    \Delta\nu_{\alpha,Q,k}=\eta\cdot\sum_{\tau_i\in S_{\alpha,Q}}\sum_{\tau_j \notin S_{\alpha,Q}}J_{ij,k} + \xi_{\alpha,Q,k}
\end{equation}
We now show that the same anti-symmetric flow structure is exactly preserved after coarse-graining trajectories into strategies. For any two strategies $S=S_{\alpha,Q}$ and $T=S_{\beta,Q}$, define the strategy-level flow as
\begin{equation}
    J_{\alpha\beta,Q,k}
    \triangleq
    \sum_{\tau_i\in S_{\alpha,Q}}\sum_{\tau_j\in S_{\beta,Q}}J_{ij,k}.
\end{equation}
Expanding $J_{ij,k}$ gives
\begin{equation}
\begin{aligned}
J_{\alpha\beta,Q,k}
&=
\sum_{\tau_i\in S_{\alpha,Q}}\sum_{\tau_j\in S_{\beta,Q}}
\pi_k(\tau_i|Q)\pi_k(\tau_j|Q)(f_{i,k}-f_{j,k})\\
&=
\left(
\sum_{\tau_i\in S_{\alpha,Q}}\pi_k(\tau_i|Q)f_{i,k}
\right)
\left(
\sum_{\tau_j\in S_{\beta,Q}}\pi_k(\tau_j|Q)
\right)\\
&\quad-
\left(
\sum_{\tau_i\in S_{\alpha,Q}}\pi_k(\tau_i|Q)
\right)
\left(
\sum_{\tau_j\in S_{\beta,Q}}\pi_k(\tau_j|Q)f_{j,k}
\right).
\end{aligned}
\end{equation}
Recall that
\begin{equation}
    \nu_{\alpha,Q,k}
    =
    \sum_{\tau_i\in S_{\alpha,Q}}\pi_k(\tau_i|Q),
\end{equation}
and define the strategy-level fitness as
\begin{equation}
    f_{\alpha,Q,k}
    \triangleq
    \frac{1}{\nu_{\alpha,Q,k}}
    \sum_{\tau_i\in S_{\alpha,Q}}
    \pi_k(\tau_i|Q)f_{i,k}.
\end{equation}
Therefore,
\begin{equation}
    \sum_{\tau_i\in S_{\alpha,Q}}
    \pi_k(\tau_i|Q)f_{i,k}
    =
    \nu_{\alpha,Q,k}f_{\alpha,Q,k},
\end{equation}
and similarly for $S_{\beta,Q}$. Hence,
\begin{equation}
\begin{aligned}
J_{\alpha\beta,Q,k}
&=
\nu_{\alpha,Q,k}f_{\alpha,Q,k}\nu_{\beta,Q,k}
-
\nu_{\alpha,Q,k}\nu_{\beta,Q,k}f_{\beta,Q,k}\\
&=
\boxed{
\nu_{\alpha,Q,k}\nu_{\beta,Q,k}(f_{\alpha,Q,k}-f_{\beta,Q,k})
}.
\end{aligned}
\end{equation}
Thus, the trajectory-level anti-symmetric flow induces an exactly analogous strategy-level flow. In particular,
\begin{equation}
    \boxed{J_{\alpha\beta,Q,k}=-J_{\beta\alpha,Q,k}}.
\end{equation}

Therefore, the strategy dynamics can be written as
\begin{equation}
\Delta\nu_{\alpha,Q,k}
=
\eta\cdot
\sum_{\beta\neq\alpha}
\nu_{\alpha,Q,k}\nu_{\beta,Q,k}(f_{\alpha,Q,k}-f_{\beta,Q,k})
+\xi_{\alpha,Q,k}.
\end{equation}
Since the strategy partition covers the trajectory space, the $\beta=\alpha$ term is identically zero, and we may equivalently write
\begin{equation}
\Delta\nu_{\alpha,Q,k}
=
\eta\cdot
\nu_{\alpha,Q,k}
\sum_\beta
\nu_{\beta,Q,k}(f_{\alpha,Q,k}-f_{\beta,Q,k})
+\xi_{\alpha,Q,k}.
\end{equation}
Expanding the summation,
\begin{equation}
\begin{aligned}
\Delta\nu_{\alpha,Q,k}
&=
\eta\cdot\nu_{\alpha,Q,k}
\left[
f_{\alpha,Q,k}\sum_\beta\nu_{\beta,Q,k}
-
\sum_\beta\nu_{\beta,Q,k}f_{\beta,Q,k}
\right]
+\xi_{\alpha,Q,k}.
\end{aligned}
\end{equation}
Because
\begin{equation}
    \sum_\beta\nu_{\beta,Q,k}=1,
\end{equation}
and
\begin{equation}
\begin{aligned}
\sum_\beta\nu_{\beta,Q,k}f_{\beta,Q,k}
&=
\sum_\beta\sum_{\tau_i\in S_{\beta,Q}}
\pi_k(\tau_i|Q)f_{i,k}\\
&=
\sum_i\pi_k(\tau_i|Q)f_{i,k}
=
0,
\end{aligned}
\end{equation}
we finally obtain
\begin{equation}
    \boxed{
    \Delta\nu_{\alpha,Q,k}
    =
    \eta\cdot\nu_{\alpha,Q,k}f_{\alpha,Q,k}+\xi_{\alpha,Q,k}
    }.
\end{equation}

Then
\begin{equation}
    \boxed{
    \Delta\nu_{\alpha,Q,k}
    \sim
    \nu_{\alpha,Q,k} f_{\alpha,Q,k}
    }.
\end{equation}
Therefore, the relative time scale of the strategy mass is
\begin{equation}
    \boxed{
    \omega_\nu
    \triangleq
    \frac{|\Delta\nu_{\alpha,Q,k}|}{\nu_{\alpha,Q,k}}
    \sim
    |f_{\alpha,Q,k}|
    }.
\end{equation}

We now derive the dynamics within the strategy. Recall that
\begin{equation}
    q_{i|\alpha,Q,k}
    =
    \frac{\pi_k(\tau_i|Q)}{\nu_{\alpha,Q,k}},
\end{equation}
we have
\begin{equation}
\begin{aligned}
    \Delta q_{i|\alpha,Q,k}
    \sim
    q_{i|\alpha,Q,k}(f_i-f_{\alpha,Q,k}).
\end{aligned}
\end{equation}
Thus,
\begin{equation}
    \boxed{
    \frac{\Delta q_{i|\alpha,Q,k}}{q_{i|\alpha,Q,k}}
    \sim
    f_i-f_{\alpha,Q,k}
    }.
\end{equation}
Unlike $\nu_{\alpha,Q,k}$, whose evolution is driven by the absolute strategy-level fitness $f_{\alpha,Q,k}$, the intra-strategy redistribution is driven only by fitness fluctuations within the strategy.

To quantify this difference, define the RMS relative intra-strategy speed
\begin{equation}
    \omega_{q,S}^2
    \triangleq
    \sum_{i\in S}
    q_{i|\alpha,Q,k}
    \left(
    \frac{\Delta q_{i|\alpha,Q,k}}{q_{i|\alpha,Q,k}}
    \right)^2.
\end{equation}
Using the dynamics above,
\begin{equation}
\begin{aligned}
    \omega_{q,S}^2
    &=
    \sum_{i\in S}
    q_{i|\alpha,Q,k}(f_i-f_{\alpha,Q,k})^2.
\end{aligned}
\end{equation}
Define the within-strategy fitness fluctuation
\begin{equation}
    \sigma_S^2
    \triangleq
    \sum_{i\in S}
    q_{i|\alpha,Q,k}(f_i-f_{\alpha,Q,k})^2.
\end{equation}
Then
\begin{equation}
    \boxed{
    \omega_{q,S}
    =
    \sigma_S
    }.
\end{equation}
Combining the two exact identities gives
\begin{equation}
    \boxed{
    \frac{\omega_{q,S}}
    {\omega_\nu}
    =
    \frac{\sigma_S}{|f_{\alpha,Q,k}|}
    }.
\end{equation}

It remains to bound $\sigma_S$. Recall that the trajectory-level fitness takes the form
\begin{equation}
    f_i(Q)
    =
    \frac{1}{|\mathcal D'|}
    \sum_{Q'}
    \sum_j
    \pi_j(Q')
    K_{ij}(Q,Q')
    a_j(Q'),
\end{equation}
where
\begin{equation}
    a_j(Q')
    \triangleq
    \bar A_j^{\rm eff}(Q')r_j(Q').
\end{equation}
For two trajectories $i,\ell\in S_{\alpha,Q}$,
\begin{equation}
\begin{aligned}
    f_i-f_\ell
    =
    \frac{1}{|\mathcal D'|}
    \sum_{Q',j}
    \pi_j(Q')
    \left[
    K_{ij}(Q,Q')
    -
    K_{\ell j}(Q,Q')
    \right]
    a_j(Q').
\end{aligned}
\end{equation}
Before proceeding, we prove a useful lemma:
\begin{bluebox}
\begin{lemma}[Expanded $\epsilon-$similarity of strategies]\leavevmode\\
\label{lemma:expanded_similarity}
For $\tau_i,\tau_l\in S_{\alpha,Q}$, for any other trajectory $\tau_k$ generated under $Q'$, we still have:
\begin{equation}
    \sup_{Q',j}
    \|
    K_{ij}(Q,Q')
    -
    K_{\ell j}(Q,Q')
    \|
    \le
    \mathcal O(\epsilon)
\end{equation}
\end{lemma}
\end{bluebox}
\begin{tcolorbox}[
    enhanced,                         
    breakable,
    title=Proof,               
    colback=cyan!5!white,             
    colframe=black,                   
    colbacktitle=black,               
    coltitle=white,                   
    fonttitle=\bfseries\sffamily,     
    attach boxed title to top left={   
        xshift=3mm, 
        yshift=-2mm, 
        yshifttext=-1mm 
    },
    boxed title style={                
        sharp corners, 
        size=small, 
        colframe=black
    },
    arc=2mm,                          
    boxrule=0.8pt,                    
]
\begin{proof}
Note that the trajectories generated given $Q$ is nearly $\infty$, so obviously, for trajectory $\tau_s$, the fisher score $U_s(Q)$ should be desne within the vector space. Therefore, we can find one trajectory $s$ satisfying
\begin{equation}
\|U_{s}(Q)-U_{j}(Q')\|\le\epsilon
\end{equation}

Therefore:
\begin{equation}
    \|K_{ij}(Q,Q')-K_{\ell j}(Q,Q')\|=\|(U_i(Q)-U_\ell(Q))U_j(Q')\|
\end{equation}

Using Cauchy inequality:
\begin{equation}
\begin{aligned}
    \|K_{ij}(Q,Q')-K_{\ell j}(Q,Q')\|&\le\|K_{is}(Q,Q)-K_{\ell s}(Q,Q)\| \\
    & + \|(U_i(Q)-U_\ell(Q))(U_j(Q')-U_s(Q))\|
\end{aligned}
\end{equation}

Using Definition~\ref{def:strategy}, we have:
\begin{equation}
        \sup_{Q',j}
    \|
    K_{ij}(Q,Q')
    -
    K_{\ell j}(Q,Q')
    \|
    \le
    \mathcal O(\epsilon)
\end{equation}

This complets the lemma.
\end{proof}
\end{tcolorbox}
Then
\begin{equation}
\begin{aligned}
    |f_i-f_\ell|
    &\le
    \frac{\epsilon}{|\mathcal D'|}
    \sum_{Q',j}
    \pi_k(\tau_j|Q')
    |a_j(Q')|.
\end{aligned}
\end{equation}
Define
\begin{equation}
    B
    \triangleq
    \frac{1}{|\mathcal D'|}
    \sum_{Q',j}
    \pi_k(\tau_j|Q')|a_j(Q')|.
\end{equation}
Therefore,
\begin{equation}
    \boxed{
    |f_i-f_\ell|
    \le
    \epsilon  B
    }.
\end{equation}
Hence the fitness variation within a strategy is uniformly of order $O(\epsilon B)$, which implies
\begin{equation}
    \boxed{
    \sigma_S
    =
    O(\epsilon B)
    }.
\end{equation}

Away from turning points where the strategy-level mean fitness vanishes, assume
\begin{equation}
    |f_{\alpha,Q,k}|
    =
    O(B).
\end{equation}
We then obtain
\begin{equation}
\begin{aligned}
    \frac{\omega_{q,S}}
    {\omega_\nu}
    &=
    \frac{\sigma_S}{|f_{\alpha,Q,k}|}\\
    &=
    O(\epsilon )
    \ll 1.
\end{aligned}
\end{equation}
Therefore,
\begin{equation}
    \boxed{
    \omega_q
    \ll
    \omega_\nu
    }.
\end{equation}

Thus, strategy selection through $\nu_{\alpha,Q,k}$ occurs on the fast time scale governed by the strategy-level mean fitness, whereas redistribution among trajectories within the same strategy is driven only by the much smaller intra-strategy fitness fluctuations. Consequently, $q_{i|\alpha,Q,k}$ can be treated as adiabatically frozen over the local time scale on which $\nu_{\alpha,Q,k}$ evolves.

Having established the preceding adiabatic estimates, the quantities \(q,A,K,r\) remain effectively frozen over time zone $k\ll|\tau|$.  \textbf{Hence, we can restrict attention to the strategy-level evolution.}
\hypertarget{proposition1:back}{}
\hypertarget{thm1:step4}{}
\subsubsection{{Step 4: Hitting time analysis}}
\hyperlink{back:thm1:step4}{\color{blue}{Back}}
We prove that the strategy dynamics induced by GRPO/GSPO/DAPO admits a monotonic potential function, which inevitably drives the strategy distribution toward concentration under mild generic conditions.

Fix an optimization step $k$.  For each query $Q$, write its active
strategy partition as $\{S_{\alpha,Q}\}_{\alpha=1}^{m_Q}$ and define
\begin{equation}
    \nu_{\alpha,Q,k}
    \triangleq
    \sum_{\tau_i\in S_{\alpha,Q}}\pi_k(\tau_i\mid Q),
    \qquad
    q_{i\mid\alpha,Q,k}
    \triangleq
    \frac{\pi_k(\tau_i\mid Q)}{\nu_{\alpha,Q,k}},
    \quad \tau_i\in S_{\alpha,Q}.
\end{equation}
Thus $\sum_{i\in S_{\alpha,Q}}q_{i\mid\alpha,Q,k}=1$.  The fitness of
$S_{\alpha,Q}$ is the conditional trajectory average

\begin{equation}
f_{\alpha,Q,k}
=
\sum_{\tau_i\in S_{\alpha,Q}}
q_{i\mid\alpha,Q,k}f_{i,k}(Q) .
\end{equation}

According to the trajectory-level evolution derived in Step 2,

\begin{equation}
f_{i,k}(Q)
=
\frac1{|\mathcal D'|}
\sum_{Q'\in\mathcal D'}\sum_j
\pi_k(\tau_j\mid Q')
K_{ij,k}(Q,Q')
a_{j,k}(Q'),
\qquad
a_{j,k}(Q')
\triangleq
\bar A_{j,k}^{\mathrm{eff}}r_{j,k} .
\end{equation}

Therefore,

\begin{equation}
\boxed{
f_{\alpha,Q,k}
=
\frac1{|\mathcal D'|}
\sum_{Q'\in\mathcal D'}\sum_j
\pi_k(\tau_j\mid Q')
a_{j,k}(Q')
\sum_{i\in S_{\alpha,Q}}
q_{i\mid\alpha,Q,k}
K_{ij,k}(Q,Q').}
\end{equation}

This indicates that the evolution of one strategy is determined by the
gradient coupling with trajectories from all training problems.

For a strategy $S_{\beta,Q'}$, define its conditional mean effective
coefficient and its mean score by
\begin{equation}
    a_{\beta,Q',k}
    \triangleq
    \sum_{j\in S_{\beta,Q'}}
    q_{j\mid\beta,Q',k}a_{j,k}(Q'),
    \qquad
    \overline U_{\beta,Q',k}
    \triangleq
    \sum_{j\in S_{\beta,Q'}}
    q_{j\mid\beta,Q',k}U_{j,k}(Q').
\end{equation}
Using $K_{ij,k}(Q,Q')=U_{i,k}(Q)^\top U_{j,k}(Q')$, the strategy
fitness can then be written as
\begin{equation}
\begin{aligned}
f_{\alpha,Q,k}
&=
\frac{1}{|\mathcal D'|}
\sum_{Q'\in\mathcal D'}\sum_{\beta=1}^{m_{Q'}}
\nu_{\beta,Q',k}a_{\beta,Q',k}
\overline U_{\alpha,Q,k}^{\top}
\overline U_{\beta,Q',k}
+R_{\alpha,Q,k}^{(a)},
\end{aligned}
\end{equation}
where $R_{\alpha,Q,k}^{(a)}$ is the within-strategy coefficient
factorization error quantified below.  The strategy evolution is
\begin{equation}
\Delta\nu_{\alpha,Q,k}
=
\eta\nu_{\alpha,Q,k}f_{\alpha,Q,k}
+\xi_{\alpha,Q,k}.
\end{equation}

Before proceeding, we combine the adiabatic results in \textbf{Step 3} to
show that $a_{\beta,Q',k}$ is also slow relative to the strategy mass in
GRPO, GSPO, and DAPO.
\begin{bluebox}
\begin{lemma}[The adiabatic property of $a_{\beta,Q',k}$]
\label{lemma:existence_of_strategy_reward}
Let $\omega_a$ be the relative changing rate of $a_{\beta,Q',k}$ and
retain $\omega_\nu$ for that of $\nu_{\beta,Q',k}$.  Then
\begin{equation}
\frac{\omega_a}{\omega_\nu}\ll1
\end{equation}
\end{lemma}
\end{bluebox}

\begin{tcolorbox}[
    enhanced,                         
    breakable,
    title=Proof,               
    colback=cyan!5!white,             
    colframe=black,                   
    colbacktitle=black,               
    coltitle=white,                   
    fonttitle=\bfseries\sffamily,     
    attach boxed title to top left={   
        xshift=3mm, 
        yshift=-2mm, 
        yshifttext=-1mm 
    },
    boxed title style={                
        sharp corners, 
        size=small, 
        colframe=black
    },
    arc=2mm,                          
    boxrule=0.8pt,                    
]
Recall that
\begin{equation}
a_{\beta,Q',k}
=
\sum_{j\in S_{\beta,Q'}}
q_{j\mid\beta,Q',k}
\bar A_{j,k}^{\mathrm{eff}}r_{j,k}.
\end{equation}

In GRPO, GSPO scenarios, we have already obtained:
\begin{equation}
\omega_q\ll\omega_\nu
\end{equation}
\begin{equation}
\omega_r\ll\omega_\nu
\end{equation}
\begin{equation}
\omega_U\sim\omega_K\ll\omega_\nu
\end{equation}
\begin{equation}
\omega_A\ll\omega_\nu    
\end{equation}
Consequently, $\omega_a$ is controlled by
$\omega_q+\omega_A+\omega_r$ and hence
$\omega_a\ll\omega_\nu$ for GRPO and GSPO.

For DAPO scenario, the previous analysis cannot be fully repeated, the difficulty lies in $\omega_r$. Instead we seek another approach by $\omega_\mathcal{R}$. First we recall the definition of $\mathcal{R}_{ij}$:
\begin{equation}
\mathcal{R}_{ij}^{\mathrm{DAPO}}
\triangleq
\frac{G}{Z}\sum_t m_{j,t}\rho_{j,t}\,U_i^\top u_{j,t}=\rho_jK_{ij}(Q,Q')r_{ij}^{DAPO}
\end{equation}
Note that $\omega_K\ll\omega_\nu$, $\omega_\rho\ll\omega_\nu$ and $\omega_{\mathcal{R}}\ll\omega_\nu$, we have:
\begin{equation}
\omega_r^{DAPO}=(\omega_{\mathcal{R}}-\omega_K-\omega_{\rho})\sim\omega_{\mathcal{R}}
\ll\omega_\nu
\end{equation}
Then following the standard procedure before, we also have:
\begin{equation}
\omega_a^{DAPO}\ll\omega_\nu
\end{equation}
\end{tcolorbox}
This lemma allows $a_{\beta,Q',k}$ to be frozen on the local validity
window.  Below, $a_{\beta,Q'}$ denotes this frozen value; all remaining
$k$-dependence is carried by the strategy masses and mean scores. To facilitate our following analysis, we prove another generic property of $a_{\beta,Q'}$ and $a_{\alpha,Q}$:
\begin{bluebox}
\begin{lemma}[Impossibility of exact equivalence]\leavevmode\\
\label{lemma:impossiblity_of_equivalence}
Let $\{a_{\alpha,Q}\}_{\alpha=1}^{m_{\epsilon,Q}}$ be the fitness set of strategies partitioned on Q.
The fitness satisfies a generic property:
\begin{equation}
\mathbb P(a_{\alpha,Q}=a_{\beta,Q})=0,
\qquad
\forall\,\alpha\neq\beta.
\end{equation}
\end{lemma}
\end{bluebox}
\begin{tcolorbox}[
    enhanced,                         
    breakable,
    title=Proof,               
    colback=cyan!5!white,             
    colframe=black,                   
    colbacktitle=black,               
    coltitle=white,                   
    fonttitle=\bfseries\sffamily,     
    attach boxed title to top left={   
        xshift=3mm, 
        yshift=-2mm, 
        yshifttext=-1mm 
    },
    boxed title style={                
        sharp corners, 
        size=small, 
        colframe=black
    },
    arc=2mm,                          
    boxrule=0.8pt,                    
]
\begin{proof}
The trajectory-level fitness $a_{j,Q}=\bar A^{\mathrm{eff}}_{j,k}r_{j,k}$
is piecewise continuous in the underlying importance ratio $\rho_{j,k}$. Clipping
may create an atom at zero for an individual trajectory, but the strategy
fitness
\begin{equation}
a_{\alpha,Q}
=
\sum_{j\in S_{\alpha,Q}}q_{j|\alpha,Q}\,a_{j,Q}
\end{equation}
averages over a countably infinite family of trajectories. Except for the
degenerate event that all effective contributions are clipped or structurally
cancel, a strategy therefore contains an active trajectory with positive
weight whose fitness varies continuously.

For two distinct strategies $\alpha\neq\beta$, suppose that
$S_{\alpha,Q}$ contains an active trajectory $j^\star$ with
$q_{j^\star|\alpha,Q}>0$.  Fix all other trajectory contributions and denote
the varying contribution of $j^\star$ by $x$.  Then the two strategy fitnesses
can be written as
\begin{equation}
a_{\alpha,Q}=c+q_{j^\star|\alpha,Q}x,
\qquad
a_{\beta,Q}=d,
\end{equation}
where $c$ and $d$ are fixed by the remaining trajectories.  Therefore,
$a_{\alpha,Q}=a_{\beta,Q}$ requires
\begin{equation}
x=\frac{d-c}{q_{j^\star|\alpha,Q}},
\end{equation}
i.e., one exact value of a continuously varying quantity.  Such an event has
probability \textbf{zero}.

Thus, unless all active contributions vanish, cancel identically, or the two
strategies are related by an exact structural symmetry,
\begin{equation}
\mathbb P\!\left(a_{\alpha,Q}=a_{\beta,Q}\right)=0,
\qquad \forall\,\alpha\neq\beta .
\end{equation}
\end{proof}
\end{tcolorbox}
We then construct the potential function.
With the coefficients frozen as above, define
\begin{equation}
    \Phi_k
    \triangleq
    \frac1{|\mathcal D|}
    \sum_{Q\in\mathcal D}
    \sum_{\alpha=1}^{m_Q}
    a_{\alpha,Q}\nu_{\alpha,Q,k}.
    \label{eq:phi-definition}
\end{equation}
Also define the two strategy-level empirical vectors
\begin{equation}
\begin{aligned}
    \widehat g_{\mathcal D,k}
    &\triangleq
    \frac1{|\mathcal D|}
    \sum_{Q\in\mathcal D}
    \sum_{\alpha=1}^{m_Q}
    a_{\alpha,Q}\nu_{\alpha,Q,k}
    \overline U_{\alpha,Q,k},\\
    \widehat g_{\mathcal D',k}
    &\triangleq
    \frac1{|\mathcal D'|}
    \sum_{Q'\in\mathcal D'}
    \sum_{\beta=1}^{m_{Q'}}
    a_{\beta,Q'}\nu_{\beta,Q',k}
    \overline U_{\beta,Q',k}.
\end{aligned}
\label{eq:two-empirical-strategy-vectors}
\end{equation}
Substituting the strategy fitness into Eq.~\eqref{eq:phi-definition}
gives
\begin{equation}
    \Delta\Phi_k
    =
    \eta D_k+\Xi_k,
    \qquad
    D_k
    \triangleq
    \left\langle
        \widehat g_{\mathcal D,k},
        \widehat g_{\mathcal D',k}
    \right\rangle
    +R_k^{(a)},
    \label{eq:drift-definition}
\end{equation}
where
\begin{equation}
\begin{aligned}
    R_k^{(a)}
    &\triangleq
    \frac1{|\mathcal D|}
    \sum_{Q\in\mathcal D}
    \sum_{\alpha=1}^{m_Q}
    a_{\alpha,Q}\nu_{\alpha,Q,k}R_{\alpha,Q,k}^{(a)},\\
    \Xi_k
    &\triangleq
    \frac1{|\mathcal D|}
    \sum_{Q\in\mathcal D}
    \sum_{\alpha=1}^{m_Q}
    a_{\alpha,Q}\xi_{\alpha,Q,k}.
\end{aligned}
\label{eq:drift-remainders}
\end{equation}
Thus $D_k$ is the normalized deterministic drift, $R_k^{(a)}$ is the
error caused by replacing the trajectory coefficients inside a strategy
by their conditional mean, and $\Xi_k$ is the martingale noise.

\begin{bluebox}
\begin{lemma}[Strategy-factorization error]
\label{lemma:strategy-factorization-error}
On the frozen window, suppose $|a_{i,k}(Q)|\le A$ and
$|K_{ij,k}(Q,Q')|\le K_{\max}$.  Apply Lemma~\ref{lemma:expanded_similarity}, so
that trajectories $j,j_\beta\in S_{\beta,Q'}$ satisfy
\begin{equation}
    \sup_{Q,i}
    \left|
        K_{ij,k}(Q,Q')-K_{i j_\beta,k}(Q,Q')
    \right|
    \le
    \mathcal O(\epsilon).
    \label{eq:cross-prompt-profile-control}
\end{equation}
Then
\begin{equation}
    b_{a,k}
    \triangleq
    |R_k^{(a)}|
    \le
    2A^2\mathcal O(\epsilon)
    =\mathcal O(\epsilon).
    \label{eq:strategy-error-bound}
\end{equation}
\end{lemma}
\end{bluebox}
\begin{tcolorbox}[
    enhanced,                         
    breakable,
    title=Proof,               
    colback=cyan!5!white,             
    colframe=black,                   
    colbacktitle=black,               
    coltitle=white,                   
    fonttitle=\bfseries\sffamily,     
    attach boxed title to top left={   
        xshift=3mm, 
        yshift=-2mm, 
        yshifttext=-1mm 
    },
    boxed title style={                
        sharp corners, 
        size=small, 
        colframe=black
    },
    arc=2mm,                          
    boxrule=0.8pt,                    
]
\begin{proof}
For $j\in S_{\beta,Q'}$, write
$a_{j,k}(Q')=a_{\beta,Q'}+d_{j,Q',k}$.  By the definition of the
conditional mean,
$\sum_jq_{j\mid\beta,Q',k}d_{j,Q',k}=0$.  Hence, for every opposing
trajectory $i\in S_{\alpha,Q}$,
\begin{equation}
\begin{aligned}
&\left|
\sum_{j\in S_{\beta,Q'}}
 q_{j\mid\beta,Q',k}d_{j,Q',k}K_{ij,k}(Q,Q')
\right|\\
&=\left|
\sum_{j\in S_{\beta,Q'}}
 q_{j\mid\beta,Q',k}d_{j,Q',k}
 \left(
 K_{ij,k}(Q,Q')-K_{i j_\beta,k}(Q,Q')
 \right)
\right|\\
&\le
\epsilon K_{\max}
\sum_{j\in S_{\beta,Q'}}
 q_{j\mid\beta,Q',k}|d_{j,Q',k}|
\le 2A\mathcal O(\epsilon).
\end{aligned}
\end{equation}
Averaging over both datasets and using
$|a_{\alpha,Q}|\le A$ proves
Eq.~\eqref{eq:strategy-error-bound}.
\end{proof}
\end{tcolorbox}

Since $\mathcal D$ and $\mathcal D'$ are sampled from the same task
distribution, both vectors in Eq.~\eqref{eq:two-empirical-strategy-vectors}
estimate
$
    g_k
    \triangleq
    \mathbb E_{Q\sim\mathcal H}
    \left[
        \sum_{\alpha=1}^{m_Q}
        a_{\alpha,Q}\nu_{\alpha,Q,k}
        \overline U_{\alpha,Q,k}
    \right].
    \label{eq:population-strategy-vector}$
Write
$\delta_{\mathcal D,k}
    \triangleq
    \widehat g_{\mathcal D,k}-g_k,
    \qquad
    \delta_{\mathcal D',k}
    \triangleq
    \widehat g_{\mathcal D',k}-g_k.
    \label{eq:empirical-vector-deviations}$
\begin{bluebox}
\begin{lemma}[Finite-data error]
\label{lemma:finite-data-drift-error}
Suppose that the query-wise vector inside
Eq.~\eqref{eq:population-strategy-vector} has norm at most $R$ on the
frozen window.  Conditional on the frozen state, suppose also that the
queries entering the two empirical vectors are independent samples from
$\mathcal H$; for a reused training dataset, it suffices instead to assume
the corresponding uniform or stability-based generalization bound.  Then
\begin{equation}
\begin{aligned}
    b_{\mathrm{data},k}
    &\triangleq
    \|g_k\|
    \left(
        \|\delta_{\mathcal D,k}\|
        +\|\delta_{\mathcal D',k}\|
    \right)
    +
    \|\delta_{\mathcal D,k}\|
    \|\delta_{\mathcal D',k}\|\\
    &=
    \mathcal O_{\mathbb P}\!\left(
        |\mathcal D|^{-1/2}
        +|\mathcal D'|^{-1/2}
    \right).
\end{aligned}
\label{eq:finite-data-error-bound}
\end{equation}
Uniformly over $0\le k<T_c$, the same bound holds with the usual factor
$\sqrt{\log(T_c/\rho)}$ at confidence $1-\rho$.
\end{lemma}
\end{bluebox}

\begin{tcolorbox}[
    enhanced,                         
    breakable,
    title=Proof,               
    colback=cyan!5!white,             
    colframe=black,                   
    colbacktitle=black,               
    coltitle=white,                   
    fonttitle=\bfseries\sffamily,     
    attach boxed title to top left={   
        xshift=3mm, 
        yshift=-2mm, 
        yshifttext=-1mm 
    },
    boxed title style={                
        sharp corners, 
        size=small, 
        colframe=black
    },
    arc=2mm,                          
    boxrule=0.8pt,                    
]
\begin{proof}
Hilbert-space concentration gives, with probability at least $1-\rho$,
\begin{equation}
    \|\delta_{\mathcal D,k}\|
    \le
    cR\sqrt{\frac{\log(4/\rho)}{|\mathcal D|}},
    \qquad
    \|\delta_{\mathcal D',k}\|
    \le
    cR\sqrt{\frac{\log(4/\rho)}{|\mathcal D'|}}
\end{equation}
for a universal constant $c$.  Expanding the inner product and applying
Cauchy--Schwarz yields
\begin{equation}
\begin{aligned}
\left\langle
\widehat g_{\mathcal D,k},
\widehat g_{\mathcal D',k}
\right\rangle
&=
\|g_k\|^2
+\langle g_k,\delta_{\mathcal D,k}\rangle
+\langle g_k,\delta_{\mathcal D',k}\rangle
+\langle\delta_{\mathcal D,k},\delta_{\mathcal D',k}\rangle\\
&\ge
\|g_k\|^2-b_{\mathrm{data},k}.
\end{aligned}
\label{eq:empirical-inner-product-lower-bound}
\end{equation}
A union bound over the finite window proves the uniform statement.
\end{proof}
\end{tcolorbox}

\begin{bluebox}
\begin{lemma}[Almost sure positivity of the population drift]
\label{lemma:generic-positive-population-drift}
Fix $k$ and the frozen coefficients.  Suppose that, for at least one query
of positive $\mathcal H$-mass, the active strategy coefficients are not all
equal.  Then, with respect to Lebesgue measure on the admissible
strategy-score configurations satisfying the score identities,
\begin{equation}
    \boxed{\|g_k\|^2>0\quad a.s.}
    \label{eq:g-positive-ae}
\end{equation}
Consequently, the conclusion holds almost surely under any geometry law
that is absolutely continuous with respect to this Lebesgue measure.
\end{lemma}
\end{bluebox}

\begin{tcolorbox}[
    enhanced,                         
    breakable,
    title=Proof,               
    colback=cyan!5!white,             
    colframe=black,                   
    colbacktitle=black,               
    coltitle=white,                   
    fonttitle=\bfseries\sffamily,     
    attach boxed title to top left={   
        xshift=3mm, 
        yshift=-2mm, 
        yshifttext=-1mm 
    },
    boxed title style={                
        sharp corners, 
        size=small, 
        colframe=black
    },
    arc=2mm,                          
    boxrule=0.8pt,                    
]
\begin{proof}
For each query, the score identity gives
\begin{equation}
    \sum_{\alpha=1}^{m_Q}
    \nu_{\alpha,Q,k}\overline U_{\alpha,Q,k}=0.
\end{equation}
Consequently,
\begin{equation}
    g_k
    =
    \mathbb E_{Q\sim\mathcal H}
    \left[
        \sum_{\alpha=1}^{m_Q-1}
        \nu_{\alpha,Q,k}
        \left(a_{\alpha,Q}-a_{m_Q,Q}\right)
        \overline U_{\alpha,Q,k}
    \right].
    \label{eq:g-contrast-form}
\end{equation}
Because the coefficients are not all equal, at least one coefficient in
Eq.~\eqref{eq:g-contrast-form} is nonzero.  Thus $g_k$ is a nonzero linear
function of the free strategy-score vectors.  Its zero set is the kernel
of a nonzero linear map, hence a proper linear subspace and therefore a
Lebesgue-null set.  Outside this set $g_k\neq0$, which is equivalent to
$\|g_k\|^2>0$.
\end{proof}
\end{tcolorbox}
Before officially using Azuma-Hoeffding inequality, we first prove a lemma showing this inequality can still be applied to almost surely positive drifts.
\begin{bluebox}
\begin{lemma}[Azuma under almost-surely positive drift]\leavevmode\\
Consider a process on the fixed finite horizon $0\le k<n\le T_c$ satisfying
$X_{k+1}-X_{k}=D_k+\xi_k$, where
$\mathbb{E}[\xi_k\mid\mathcal F_k]=0$, $|\xi_k|\le b$, and
$D_k>0$ almost surely for every $k<T_c$.  Then, outside a null event, the
finite drift sequence admits a lower bound $c>0$, and for every $t>0$,
\begin{equation}
\mathbb{P}\left(X_n-X_0\le nc-t\right)
\le
\mathbb{P}\left(\sum_{k=0}^{n-1}\xi_k\le-t\right).
\end{equation}
Consequently, the usual Azuma--Hoeffding bound applies on this full-measure
event.
\end{lemma}
\end{bluebox}
\begin{tcolorbox}[
    enhanced,                         
    breakable,
    title=Proof,               
    colback=cyan!5!white,             
    colframe=black,                   
    colbacktitle=black,               
    coltitle=white,                   
    fonttitle=\bfseries\sffamily,     
    attach boxed title to top left={   
        xshift=3mm, 
        yshift=-2mm, 
        yshifttext=-1mm 
    },
    boxed title style={                
        sharp corners, 
        size=small, 
        colframe=black
    },
    arc=2mm,                          
    boxrule=0.8pt,                    
]
\begin{proof}
For every $k<T_c$, introduce the null event $N_k$, defined by
\begin{equation}
    N_k=\{D_k=0\},
    \qquad
    \mathbb P(N_k)=0.
\end{equation}
Because the horizon is finite, define
\begin{equation}
    N=\bigcup_{i=0}^{T_c-1}N_i,
    \qquad
    \mathbb P(N)=0.
\end{equation}
On $N^c$, every member of the finite sequence
$D_0,\ldots,D_{T_c-1}$ is strictly positive.  Hence its minimum exists and
satisfies
\begin{equation}
    c=\min_{0\le k<T_c}D_k>0.
\end{equation}
Let $A=\{X_n-X_0\le nc-t\}$.  Since $\mathbb P(N)=0$, the correct
decomposition is
\begin{equation}
\begin{aligned}
\mathbb P(A)
&=\mathbb P(A\cap N)+\mathbb P(A\cap N^c)\\
&=\mathbb P(A\cap N^c).
\end{aligned}
\end{equation}
On $A\cap N^c$, we have
\begin{equation}
    X_n-X_0
    =\sum_{k=0}^{n-1}D_k+\sum_{k=0}^{n-1}\xi_k
    \ge nc+\sum_{k=0}^{n-1}\xi_k.
\end{equation}
Therefore,
\begin{equation}
    A\cap N^c
    \subseteq
    \left\{\sum_{k=0}^{n-1}\xi_k\le-t\right\},
\end{equation}
which gives
\begin{equation}
\mathbb{P}\left(X_n-X_0\le nc-t\right)
\le
\mathbb{P}\left(\sum_{k=0}^{n-1}\xi_k\le-t\right).
\end{equation}
This proves that the null degeneracy set does not affect the
Azuma--Hoeffding argument on the fixed finite horizon.
\end{proof}
\end{tcolorbox}
Combining Eqs.~\eqref{eq:drift-definition},
\eqref{eq:strategy-error-bound}, and
\eqref{eq:empirical-inner-product-lower-bound} gives the pointwise bound
\begin{equation}
    \boxed{
    D_k
    \ge
    \|g_k\|^2-b_{a,k}-b_{\mathrm{data},k}.}
    \label{eq:drift-three-term-lower-bound}
\end{equation}
Fix a target query $Q_0$ and $\zeta>0$, and let
\begin{equation}
    \tau_\zeta
    \triangleq
    \inf\left\{
        0\le k\le T_c:
        \max_{1\le\alpha\le m_{Q_0}}
        \nu_{\alpha,Q_0,k}
        \ge 1-\zeta
    \right\},
    \label{eq:collapse-hitting-time}
\end{equation}
with $\inf\varnothing=T_c+1$.  The pre-collapse indices are the finite
set
\begin{equation}
    \mathcal I_{\zeta,Q_0}
    \triangleq
    \{0\le k<T_c:k<\tau_\zeta\}.
\end{equation}
Outside the finite union of the null sets in
Lemma~\ref{lemma:generic-positive-population-drift}, either this set is
empty or
\begin{equation}
    C'
    \triangleq
    \min_{k\in\mathcal I_{\zeta,Q_0}}
    \|g_k\|^2
    >0.
    \label{eq:finite-window-population-margin}
\end{equation}
Therefore, reconsidering the drift term $D$ we have:
\begin{equation}
D_k \ge C'-b_{data,k}-b_{a,k}=C'-\mathcal{O}(\epsilon)-\mathcal{O}(|D|^{-1/2}+|D'|^{-1/2})\triangleq C > 0, a.s.
\end{equation}

Fix any $\zeta>0$, and freeze $a_{Q,\alpha}$ throughout the
following analysis. Since
$\{\nu_{\alpha,Q_0,k}\}_{\alpha=1}^{m_{Q_0}}$ lies in the probability simplex,
the potential is uniformly bounded above by
\begin{equation}
    \Phi_k
    \le
    \frac{1}{|\mathcal D|}
    \sum_Q
    \max_{\alpha\in Q}a_{Q,\alpha}.
\label{eq:phi-uniform-upper-bound}
\end{equation}

We introduce the indicator
\begin{equation}
    \xi_k
    =
    \mathbf 1\left\{
    \max_\alpha\nu_{\alpha,Q_0,k}<1-\zeta
    \right\}.
\label{eq:non-collapse-indicator}
\end{equation}
Thus, $\xi_k=1$ exactly when the strategy distribution remains
outside the $\zeta$-collapse region.

By Lemma~\ref{lemma:generic-positive-population-drift} and the finite-horizon
argument above, the deterministic drift is almost surely strictly positive
at every pre-collapse step in the closed region
\begin{equation}
    \max_\alpha\nu_{\alpha,Q_0,k}\le 1-\zeta.
\end{equation}
Consequently, for this fixed $\zeta$ and fixed finite horizon, the realized
finite drift sequence admits an almost-sure lower bound $C>0$, so the
deterministic drift is at least $\eta C$ whenever $\xi_k=1$.
Outside this region, the deterministic drift remains non-negative.
Therefore,
\begin{equation}
    \Delta\Phi_k
    \ge
    \eta C\xi_k+\Xi_k,
\label{eq:phi-indicator-drift}
\end{equation}
where
\begin{equation}
    \mathbb E[\Xi_k\mid\mathcal F_k]=0,
    \qquad
    |\Xi_k|\le\eta b.
\label{eq:phi-noise-condition}
\end{equation}

Summing Eq.~\eqref{eq:phi-indicator-drift} from $k=0$ to $K-1$
gives
\begin{equation}
\begin{aligned}
    \Phi_K-\Phi_0
    &\ge
    \eta C\sum_{k=0}^{K-1}\xi_k
    +
    \sum_{k=0}^{K-1}\Xi_k.
\end{aligned}
\label{eq:phi-indicator-sum}
\end{equation}

In particular, on the event that the dynamics remains outside the
$\zeta$-collapse region throughout the first $K$ steps, we have
$\xi_0=\cdots=\xi_{K-1}=1$, and hence
\begin{equation}
    \Phi_K-\Phi_0
    \ge
    K\eta C
    +
    \sum_{k=0}^{K-1}\Xi_k.
\label{eq:phi-survival-lower-bound}
\end{equation}
Combining this inequality with
Eq.~\eqref{eq:phi-uniform-upper-bound}, such an event necessarily
implies
\begin{equation}
    \sum_{k=0}^{K-1}\Xi_k
    \le
    \frac{1}{|\mathcal D|}
    \sum_Q\max_{\alpha\in Q}a_{Q,\alpha}
    -
    \Phi_0
    -
    K\eta C.
\label{eq:noise-required-for-survival}
\end{equation}

If
\begin{equation}
    K
    \ge
    \frac{2}{\eta C}
    \left(
    \frac{1}{|\mathcal D|}
    \sum_Q\max_{\alpha\in Q}a_{Q,\alpha}
    -
    \Phi_0
    \right),
\label{eq:potential-gap-condition}
\end{equation}
then Eq.~\eqref{eq:noise-required-for-survival} implies
\begin{equation}
    \sum_{k=0}^{K-1}\Xi_k
    \le
    -\frac{1}{2}K\eta C.
\end{equation}
By the Azuma--Hoeffding inequality,
\begin{equation}
\begin{aligned}
&
\mathbb P\left(
\xi_0=\cdots=\xi_{K-1}=1
\right)
\\
&\le
\mathbb P\left(
\sum_{k=0}^{K-1}\Xi_k
\le
-\frac{1}{2}K\eta C
\right)
\\
&\le
\exp\left(
-\frac{
\left(\frac{1}{2}K\eta C\right)^2
}{
2K\eta^2b^2
}
\right)
=
\exp\left(
-\frac{KC^2}{8b^2}
\right).
\end{aligned}
\label{eq:azuma-collapse-hitting}
\end{equation}

Therefore, if
\begin{equation}
\boxed{
K
\ge
\max\left\{
\frac{2}{\eta C}
\left(
\frac{1}{|\mathcal D|}
\sum_Q\max_{\alpha\in Q}a_{Q,\alpha}
-\Phi_0
\right),
\;
\frac{8b^2}{C^2}
\log\left(\frac{2}{\delta}\right)
\right\},
}
\label{eq:required-collapse-hitting-steps}
\end{equation}
then
\begin{equation}
\boxed{
\mathbb P\left(
\exists k<K:
\max_\alpha\nu_{\alpha,Q_0,k}\ge1-\zeta
\right)
\ge
1-\frac{\delta}{2}.
}
\label{eq:collapse-region-hitting}
\end{equation}
We further estimates the scale of $K$ to show that it is much smaller than our focused validity horizon $T_c\sim\mathcal{O}(\frac{1}{\eta})$. Before that, we recall from \textbf{Step 3} that:
\begin{equation}
U_{j,k}=\Theta(|\tau|)
\end{equation}
Recalling the definition of $\|g_k\|^2$, we found that $U_{j,k}$ appears here in a quadric form, therefore we have:
\begin{equation}
C\sim\|g_k\|^2\sim\|U_{i,k}\|^2=\Theta(|\tau|^2)
\end{equation}
Note that  $a_{Q,\alpha}$ is basically the clipped version of RLVR advantage, its scale is \begin{equation}
a_{Q,\alpha}\sim\mathcal{O}(1)
\end{equation}
Therefore, so long as $\eta^{-1}\ll|\tau|^3$, the first term can be estimated as:
\begin{equation}
\boxed{\frac{2}{\eta C}\left(\frac{1}{|\mathcal{D}|}\sum_{Q\in\mathcal{D}}\max_{S_\alpha\in Q}a_{Q,\alpha}-\Phi_0\right)\sim\mathcal{O}(\eta^{-1}\cdot|\tau|^{-2})\ll \mathcal{O}(|\tau|)\ll \mathcal{O}(\eta^{-1})\sim T_c}
\end{equation}

Considering the second term, note that $\delta>\eta$, we have:
\begin{equation}
\frac{8b^2}{C^2}
\log\left(\frac{2}{\delta}\right)\sim\mathcal{O}(\log(\eta^{-1}))\ll T_c
\end{equation}
This result of the burn-in time $K\ll\mathcal{O}(|\tau|)$ manifests that the near-collapsed strategy distribution can be achieved during the adiabatic period, meaning that our analysis is self-consistent in the scale-separation operation. We further show that the near-collapse distribution can be preserved throughout the fixed horizon $T_c$, known as the \textbf{preservation property}.
\hypertarget{thm1:step5}{}
\subsubsection{Step 5: Preservation after the adiabatic window.}
\hyperlink{thm1:step5:back}{Back}
Apply the preceding hitting result, we set an extreme $\zeta'$ to
$\mathcal O(\eta)$ and with failure probability $\delta/2$.  For each query
$Q$, let $o$ be the strategy that carries the maximal mass at the end of
this hitting stage, and keep this index fixed throughout the argument below.
For a finite collection of queries, take $K_Q$ to be the hitting time of each query $Q$, and let the maximum of these time as $K\triangleq\max_Q K_Q$.  Absorbing the constant in the tightened tolerance into
$\eta$, the hitting result gives
\begin{equation}
\mathbb P\left(
    \nu_{o,Q,K}\ge 1-\mathcal O(\eta)\right)
\ge 1-\frac{\delta}{2},
\qquad
\nu^\star_{Q,k}\triangleq1-\nu_{o,Q,k}.
\label{eq:preservation-initial-event}
\end{equation}
It remains to show that the fixed strategy $o$ cannot lose its dominance
before the validity horizon $T_c$ except on an event of probability at most
$\delta/2$.

After the first adiabatic period $T_{ad}=\mathcal{O}(|\tau|)$, the strategy
fitnesses can no longer be frozen.  Retaining the notation of the preceding
derivation, define
\begin{equation}
    \Delta_{Q,k}
    \triangleq
    \frac{1}{1-\nu^\star_{Q,k}}
    \sum_{\alpha\ne o}
    \nu_{\alpha,Q,k}f_{\alpha,Q,k}
    -f_{o,Q,k}.
\label{eq:preservation-fitness-gap}
\end{equation}
The exact strategy normalization identity gives
\begin{equation}
    \sum_{\alpha\ne o}
    \nu_{\alpha,Q,k}f_{\alpha,Q,k}
    +(1-\nu^\star_{Q,k})f_{o,Q,k}=0.
\label{eq:preservation-fitness-normalization}
\end{equation}
Consequently,
\begin{equation}
    \Delta_{Q,k}=-2f_{o,Q,k},
    \qquad
    \sum_{\alpha\ne o}
    \nu_{\alpha,Q,k}f_{\alpha,Q,k}
    =\frac{1}{2}(1-\nu^\star_{Q,k})\Delta_{Q,k}.
\label{eq:preservation-gap-identity}
\end{equation}

To bound this gap, recall the strategy-level score identity
\begin{equation}
    (1-\nu^\star_{Q,k})\overline U_{o,Q,k}
    +
    \sum_{\alpha\ne o}
    \nu_{\alpha,Q,k}\overline U_{\alpha,Q,k}
    =0.
\end{equation}
Up to the first possible violation of the preservation event, the tightened
tolerance gives $\nu^\star_{Q,k}\le\zeta'=\mathcal O(\eta)$.  Together with
$\|\overline U_{\alpha,Q,k}\|=\mathcal O(|\tau|)$, this implies
\begin{equation}
    \|\overline U_{o,Q,k}\|
    \le
    \frac{\nu^\star_{Q,k}}{1-\nu^\star_{Q,k}}
    \mathcal O(|\tau|)
    =\mathcal O(\eta|\tau|).
\end{equation}
Substituting this estimate into
\begin{equation}
\begin{aligned}
f_{o,Q,k}
&=
\frac{1}{|\mathcal D'|}
\sum_{Q'\in\mathcal D'}\sum_{\beta}
\nu_{\beta,Q',k}a_{\beta,Q',k}
\overline U_{o,Q,k}^{\top}
\overline U_{\beta,Q',k}
+R_{o,Q,k}^{(a)}
\end{aligned}
\end{equation}
The rest part can also be estimated, first consider:
\begin{equation}
\sum_\beta\nu_{\beta,Q',k}\cdot\bar U_{\beta,Q',k}=0
\end{equation}
Let the optima in $Q'$ space be $o'$, then:
\begin{equation}
\begin{aligned}
\frac{1}{|\mathcal D'|}&
\sum_{Q'\in\mathcal D'}\sum_{\beta}
\nu_{\beta,Q',k}a_{o',Q',k}
\overline U_{\beta,Q',k}\\
&= \frac{1}{|\mathcal D'|}
\sum_{Q'\in\mathcal D'}\sum_{\beta\ne o'}
\nu_{\beta,Q',k}(a_{\beta,Q',k} - a_{o',Q',k}) 
\overline U_{\beta,Q',k}\\
&\le \mathcal O(\eta)\cdot \mathcal O(1)\cdot\mathcal O(|\tau|)=\mathcal{O} (\eta|\tau|)
\end{aligned}
\end{equation}
and absorbing the higher-order factorization remainder yields
\begin{equation}
    |\Delta_{Q,k}|
    \le\bar\Delta,
    \qquad
    \bar\Delta=\mathcal O(\eta^2|\tau|^2)\ll1.
\label{eq:preservation-gap-bound}
\end{equation}

We now establish the martingale bound directly for the dominant mass
$\nu_{o,Q,k}=1-\nu^\star_{Q,k}$.  Its evolution satisfies
\begin{equation}
\begin{aligned}
\Delta\nu^\star_{Q,k}
&=
\eta\sum_{\alpha\ne o}
\nu_{\alpha,Q,k}f_{\alpha,Q,k}
+\Xi_k^\star
=
\frac{\eta}{2}\nu_{o,Q,k}\Delta_{Q,k}
+\Xi_k^\star,
\\
\nu_{o,Q,k+1}
&=
\left(1-\frac{\eta}{2}\Delta_{Q,k}\right)
\nu_{o,Q,k}
-\Xi_k^\star.
\end{aligned}
\label{eq:dominant-mass-recurrence}
\end{equation}
Because $\Delta_{Q,k}\le\bar\Delta$,
\begin{equation}
    \nu_{o,Q,k+1}
    \ge
    \left(1-\frac{\eta}{2}\bar\Delta\right)
    \nu_{o,Q,k}
    -\Xi_k^\star.
\label{eq:dominant-mass-lower-recurrence}
\end{equation}
Let
\begin{equation}
    A_k
    =
    \left(1-\frac{\eta}{2}\bar\Delta\right)^{k-K},
    \qquad K\le k\le T_c.
\end{equation}
The scale condition guarantees $1-\eta\bar\Delta/2>0$.  Iterating
Eq.~\eqref{eq:dominant-mass-lower-recurrence} gives
\begin{equation}
    \frac{\nu_{o,Q,k}}{A_k}
    \ge
    \nu_{o,Q,K}
    -
    \sum_{i=K}^{k-1}\frac{\Xi_i^\star}{A_{i+1}}.
\label{eq:dominant-mass-normalized-recurrence}
\end{equation}
Moreover,
\begin{equation}
    0<a\triangleq A_{T_c}
    \le A_k\le1,
    \qquad
    a
    \sim e^{-\bar\Delta/2}
    \approx1.
\label{eq:preservation-factor-bound}
\end{equation}
Retain the existing martingale notation
\begin{equation}
    M_k
    =
    \sum_{i=K}^{k-1}\frac{\Xi_i^\star}{A_{i+1}}.
\end{equation}
Since $\mathbb E[\Xi_i^\star\mid\mathcal F_i]=0$ and
$|\Xi_i^\star|\le b\eta$, its increments satisfy
\begin{equation}
    \left|\frac{\Xi_i^\star}{A_{i+1}}\right|
    \le\frac{b\eta}{a}.
\label{eq:preservation-martingale-increment}
\end{equation}
On the event in Eq.~\eqref{eq:preservation-initial-event}, if
$\nu_{o,Q,k}<1-\zeta$ for some $K<k\le T_c$, then
Eqs.~\eqref{eq:dominant-mass-normalized-recurrence} and
\eqref{eq:preservation-factor-bound} imply
\begin{equation}
    M_k
    >
    1-\mathcal O(\eta)-\frac{1-\zeta}{A_k}
    \ge
    \frac{a(1-\mathcal O(\eta))-(1-\zeta)}{a}.
\label{eq:preservation-exit-implies-martingale}
\end{equation}
Under the stated scale condition, the tightened hitting tolerance is chosen
so that $a(1-\mathcal O(\eta))>1-\zeta$.  The maximal Azuma--Hoeffding inequality and
$T_c-K\le\eta^{-1}$ therefore give
\begin{equation}
\begin{aligned}
&\mathbb P\left(
    \exists k\in\mathbb Z^+,
    \ K<k\le T_c:
    \nu_{o,Q,k}<1-\zeta
    \ \middle|\;
    \nu_{o,Q,K}\ge1-\mathcal O(\eta)
\right)
\\
&\qquad\le
\mathbb P\left(
    \max_{K<k\le T_c}M_k
    \ge
    \frac{a(1-\mathcal O(\eta))-(1-\zeta)}{a}
\right)
\\
&\qquad\le
\exp\left(
    -\frac{
        \left(a(1-\mathcal O(\eta))-(1-\zeta)\right)^2
    }{2b^2\eta}
\right)=\mathcal{O}(\exp(-\eta^{-1})).
\end{aligned}
\label{eq:maximal-azuma-preservation}
\end{equation}
Considering that $\delta=\Omega(\eta)$, we have:
\begin{equation}
    \frac{\delta}{2}=\Omega(\eta)\gg\mathcal{O}(\exp(-\eta^{-1}))\ge\mathbb P\left(
    \max_{K<k\le T_c}M_k
    \ge
    \frac{a(1-\mathcal O(\eta))-(1-\zeta)}{a}
    \right)
\end{equation}
Therefore,
\begin{equation}
\begin{aligned}
\mathbb P\left(
    \forall k\in\mathbb Z^+,
    \ K<k\le T_c,
    \nu_{o,Q,k}\ge1-\zeta
    \ \middle|\;
    \nu_{o,Q,K}\ge1-\zeta'
\right)
\ge1-\frac{\delta}{2}.
\end{aligned}
\label{eq:strategy-collapse-in-probability}
\end{equation}
Combining this bound with Eq.~\eqref{eq:preservation-initial-event} gives
\begin{equation}
\begin{aligned}
    \mathbb P&\left(
    \forall k\in\mathbb Z^+,
    \ K<k\le T_c,
    \nu_{o,Q,k}\ge1-\zeta\right)\\
    &\ge \mathbb P\left(
    \forall k\in\mathbb Z^+,
    \ K<k\le T_c,
    \nu_{o,Q,k}\ge1-\zeta
    \ \middle|\;
    \nu_{o,Q,K}\ge1-\zeta'
    \right)\cdot\mathbb{P}\left(\nu_{o,Q,K}\ge1-\zeta'\right)\\
&\ge \left(1-\frac{\delta}{2}\right)^2 > 1-\delta
\end{aligned}
\end{equation}
Finally, recall that the partition sensitivity $\epsilon$ used to define
$S_\alpha$ was never restricted throughout the derivation above.
In particular, none of the preceding arguments depends on any specific
choice of $\epsilon$. Therefore, $\epsilon$ is arbitrary, and the same
conclusion holds for every admissible partition sensitivity. Hence,
\begin{equation}
\begin{gathered}
    \forall\epsilon>0,\epsilon\ll1,\zeta,\delta\in(\eta,1),\zeta^2\gg\eta\quad
    \exists K\in\mathbb{Z}^{+}
    ,K\ll |\tau|<T_c,\\
    \forall k\in\mathbb{Z}^{+}
    ,K<k\le T_c,\quad
    \mathbb{P}\!\left(
        \max_{1\le\alpha\le m_{\epsilon,Q}}
        \nu_{\alpha,Q,k}^{(\epsilon)}
        \ge 1-\zeta
    \right)
    \ge 1-\delta.
\end{gathered}\tag{$\star$}
\end{equation}

This corresponds to the concentration of most strategy probability mass on a
single dominant strategy, completing the proof.
\hfill$\blacksquare$
\hypertarget{thm1:step5}{}
\subsubsection{The Consequences of Strategy Dynamics}

Since the strategy-level evolution equations derived for GRPO, GSPO, and DAPO share the same dynamical structure as Equation~$(\star)$, the analysis developed in Sec.~\textbf{A2.3} applies uniformly to all three RLVR objectives. Moreover, as established above, strategy-level probability redistribution occurs on a faster timescale than the internal rearrangement of trajectories within each strategy. Consequently, the dominant training dynamics are governed by competition among strategies, while the within-strategy trajectory distribution evolves only comparatively slowly.

This timescale separation allows the long-term behavior of the policy to be characterized directly at the strategy level. Under GRPO, GSPO, and DAPO, differences in strategy advantage are progressively amplified, causing probability mass to concentrate toward an increasingly dominant strategy. In particular, once one strategy acquires a relative advantage, the resulting positive feedback further increases its probability and suppresses competing strategies, eventually driving the policy toward strategy concentration.

Therefore, all three objectives systematically amplify the model's existing preference over strategies. The strategy that is most favorable under the current model state---i.e., the one with the highest effective evolutionary advantage---is preferentially reinforced and becomes increasingly dominant, providing the dynamical mechanism underlying strategy collapse.

\clearpage
\hypertarget{proof:thm2}{}
\subsection{{Proof for theorem 2}}\label{app:theorem2_proof}
\begin{bluebox}
\hyperlink{back:thm2}{\color{blue}{Back}}
\noindent\textbf{Proof Sketch for Theorem 2.} Let $S_{\epsilon,k}$ be the realized random strategy at step $k$. 
We start from analyzing strategy entropy $H(S_{\epsilon,k})$ and consider the expansion $H(S_{\epsilon,k})=H(S_{\epsilon,k}|Q)+I(S_{\epsilon,k};Q)$. Following the two-step procedure below, we seek a lower bound for each term on the right hand side.

\textbf{Step 1: Controlling the Conditional Entropy. } Consider the definition of conditional entropy, $H(S_{\epsilon,k}|Q)=\mathbb{E}_{Q\sim\mathcal{D}}[H(S_{\epsilon,k}|Q=Q)]=-\mathbb{E}_{Q\sim\mathcal{D}}[\sum_{\alpha}\nu_{\alpha,Q,k}\log\nu_{\alpha,Q,k}]$. Introducing the pre-RLVR distribution as a reference, we have:
\begin{equation}
H(S_{\epsilon,k}|Q)=-\mathbb{E}_{Q\sim\mathcal{D}}\left[D_{KL}(\nu_{Q,k}\|\nu_{Q,0})\right]-\mathbb{E}_{Q\sim\mathcal{D}}\left[\sum_{\alpha}\nu_{\alpha,Q,k}\log{\nu_{\alpha,Q,0}}\right]
\end{equation}
According to the adiabatic properties established in Lemma~\ref{adiabatic:rho} of Appendix~\ref{theorem1:step3}, we prove that $D_{KL}(\nu_{Q,k}\|\nu_{Q,0})=\mathcal O(\eta^2|\tau|^2),$ where $|\tau|$ is the generated trajectory length (Appendix \hyperlink{thm2:step1.1}{\color{blue}{A.3.2}}). According to our requirements for the pre-RLVR model, we have $\mathbb{E}_{Q\sim\mathcal{D}}\left[\sum_{\alpha}\nu_{\alpha,Q,k}\log{\nu_{\alpha,Q,0}}\right]\le\log_2\frac{(1+\rho)^2}{4\rho m_\epsilon}$ (Appendix \hyperlink{thm2:step1.2}{\color{blue}{A.3.3}}), thereby controlling the conditional entropy.

\textbf{Step 2: Controlling the Mutual Information.} Let $A\triangleq\mathbf 1\{\widehat Y(\tau)=Y(Q)\}$, where
$\tau$ is sampled from the realized strategy $S_{\epsilon,k}$, so that
$\mathbb P(A=1)=P_{\mathrm{acc},k}$.
Since $\mathcal D$ and all distinct ground-truth answers
$\{Y(Q'):Q'\in\mathcal D\}$ are known, whenever $A=1$ the terminal
answer $\widehat Y(\tau)=Y(Q)$ uniquely identifies the originating query,
\textbf{unless} the same realized strategy contains a trajectory whose answer
collides with the ground truth of another query $Q'\neq Q$.
Under the $\gamma$-separation condition, this collision occurs with
probability at most $\gamma$.
Hence we construct a collision-aware decoder
$\widehat Q_k(S_{\epsilon,k})$ whose error probability satisfies
\begin{equation}
\mathbb P\!\left(
\widehat Q_k(S_{\epsilon,k})\neq Q
\right)
\le 1-P_{\mathrm{acc},k}+\gamma.
\end{equation}
Applying Fano's inequality then gives the lower bound (Appendix \hyperlink{thm2:step2}{\color{blue}{A.3.4}}):
\begin{equation}
    I(Q;S_{\epsilon,k})\ge (P_{acc,k}-\gamma)\log_2 N-h_2(P_{acc,k}-\gamma)
\end{equation}
\textbf{Notably}, the subscript \(Q\) in \(S_{\alpha,Q}\) is \textbf{only an analytical index} indicating the conditional partition from which the strategy is realized; it is \textbf{not part of the observation} available to the decoder.

\textbf{Finally,} combining these two results, we conclude $|\mathcal S_{\epsilon,k}|=\Omega_{\rho,\gamma}(m_\epsilon N^{P_{acc,k}-\gamma})$. The coefficient of this bound is around $2\rho$, establishing this as a nontrivial bound. 

\hfill$\blacksquare$
\end{bluebox}

We now prove the statement.  All logarithms in this subsection are base two,
and $Q$ is uniform on the $N$ evaluation queries in $\mathcal D$.

\subsubsection{$\rho$-admissible strategy resolution}

The order of quantifiers is essential: first fix $\rho$, and only then choose
$\epsilon$.  We call $\epsilon$ \emph{$\rho$-admissible} if the strategy
partition at step $0$ satisfies
\begin{equation}
    \nu^{(\epsilon)}_{\alpha,Q,0}
    \ge \rho
    \label{eq:rho-admissible-partition}
\end{equation}
for every $Q\in\mathcal D$ and every
$S_\alpha^{(\epsilon)}$ in its induced partition.  The partition remains
fixed during optimization on $\mathcal D'$, and we write
$\nu^{(\epsilon)}_{\alpha,Q,k}=\mathbb P(S_{\epsilon,k}
=S_{\alpha,Q}^{(\epsilon)}\mid Q)$ for its step-$k$ mass.  For each query,
\begin{equation}
    1
    =
    \sum_{\alpha=1}^{m_{\epsilon,Q}}\nu^{(\epsilon)}_{\alpha,Q,0}
    \ge
    m_{\epsilon,Q}\rho,
\end{equation}
and consequently $m_{\epsilon,Q}\le\lfloor\rho^{-1}\rfloor$ for all
$Q\in\mathcal D$, hence
\begin{equation}
    \boxed{
    \rho\cdot m_\epsilon
    \le
    1.}
    \label{eq:rho-resolution-cardinality}
\end{equation}
Thus Theorem~\ref{theorem:PowerLaw} is not a claim about arbitrarily fine
partitions.  It is a capacity statement at a fixed, pre-specified probability
resolution.
\hypertarget{thm2:step1.1}{}
\subsubsection{An adiabatic KL bound at the chosen resolution}
Define the conditional reverse KL budget at step $k$ by
\begin{equation}
    B_k
    \triangleq
    \mathbb E_Q
    D_{\mathrm{KL}}\!\left(
        \nu_k(\cdot\mid Q)
        \,\middle\|\,
        \nu_0(\cdot\mid Q)
    \right).
    \label{eq:rho-kl-budget}
\end{equation}
Here step $0$ denotes the beginning of the same adiabatic window, and
$k\le|\tau|$. Because the minibatch is the training batch, the behavior
policy is refreshed after every update:
\begin{equation}
    \pi_{\mathrm{old},k}=\pi_{k-1},
    \qquad
    \rho_{i,k}
    =
    \frac{\pi_k(\tau_i\mid Q)}
         {\pi_{k-1}(\tau_i\mid Q)}.
\end{equation}
Consequently, $\log_2\rho_{i,k}$ is the first discrete difference of
$\log_2\pi_k(\tau_i\mid Q)$, and its temporal variation is the corresponding
second discrete difference. By the adiabatic correction-factor estimate
established in Theorem~\ref{theorem:concentration},
\begin{equation}
    \Delta\log_2\rho_{i,k}
    =
    \mathcal O\!\left(\eta^2|\tau|\right).
\end{equation}
Hence, throughout the window,
\begin{equation}
    \log_2\frac{\rho_{i,k}}{\rho_{i,0}}
    =
    \mathcal O\!\left(\eta^2|\tau|^2\right).
\end{equation}

The one-step relative policy change is $\mathcal O(\eta)$. The accumulated
effect of the slowly varying correction factor contributes only
$\mathcal O(\eta^3|\tau|k^2)$ to the endpoint log-likelihood ratio.
Therefore, for every $k\le|\tau|$,
\begin{equation}
    \sup_{Q,\tau_i}
    \left|
        \log_2
        \frac{\pi_k(\tau_i\mid Q)}
             {\pi_0(\tau_i\mid Q)}
    \right|
    =
    \mathcal O(\eta k)
    +
    \mathcal O(k\eta^2)
    +
    \mathcal O(\eta^3|\tau|k^2)
    \label{eq:adiabatic-endpoint-log-ratio}
\end{equation}
Note that $\eta|\tau|\ll1$, we have $\mathcal O(\eta^2k)\ll\mathcal O(\eta|\tau|)$, $\mathcal O(\eta^3|\tau|k^2)=\mathcal O(\eta^3|\tau|^3)\ll\mathcal O(\eta|\tau|)$, we have:
\begin{equation}
        \sup_{Q,\tau_i}
    \left|
        \log_2
        \frac{\pi_k(\tau_i\mid Q)}
             {\pi_0(\tau_i\mid Q)}
    \right|
    =\mathcal O(\eta|\tau|)
\end{equation}
where the last equality uses $\eta|\tau|\ll1$. It follows that
\begin{equation}
    \frac{\pi_k(\tau_i\mid Q)}
         {\pi_0(\tau_i\mid Q)}
    -1
    =
    \mathcal O(\eta|\tau|)
\end{equation}
uniformly on the same window.

The intra-strategy distribution is frozen by the adiabatic approximation.
Thus, for every $\tau_i\in S_\alpha^{(\epsilon)}$,
\begin{equation}
    \pi_k(\tau_i\mid Q)
    =
    \nu_{\alpha,k}\,
    q_{i\mid S_\alpha}(Q),
    \qquad
    \pi_0(\tau_i\mid Q)
    =
    \nu_{\alpha,0}\,
    q_{i\mid S_\alpha}(Q).
\end{equation}
The trajectory-level and strategy-level divergences are therefore exactly
equal:
\begin{align}
    D_{\mathrm{KL}}\!\left(
        \pi_k(\cdot\mid Q)
        \,\middle\|\,
        \pi_0(\cdot\mid Q)
    \right)
    &=
    \sum_{\alpha=1}^{m_{\epsilon,Q}}
    \sum_{\tau_i\in S_\alpha^{(\epsilon)}}
    \nu_{\alpha,k}q_{i\mid S_\alpha}(Q)
    \log_2
    \frac{
        \nu_{\alpha,k}q_{i\mid S_\alpha}(Q)
    }{
        \nu_{\alpha,0}q_{i\mid S_\alpha}(Q)
    }
    \\
    &=
    \sum_{\alpha=1}^{m_{\epsilon,Q}}
    \nu_{\alpha,k}
    \log_2
    \frac{\nu_{\alpha,k}}{\nu_{\alpha,0}}
    \\
    &=
    D_{\mathrm{KL}}\!\left(
        \nu_k(\cdot\mid Q)
        \,\middle\|\,
        \nu_0(\cdot\mid Q)
    \right).
    \label{eq:trajectory-strategy-kl-equality}
\end{align}

Finally, expanding the trajectory-level divergence around
$\pi_k/\pi_0=1$, which is the abbreviation of $\pi_k(\cdot|Q)/{\pi_0(\cdot|Q)}$, gives
\begin{equation}
    D_{\mathrm{KL}}\!\left(
        \pi_k(\cdot\mid Q)
        \,\middle\|\,
        \pi_0(\cdot\mid Q)
    \right)
    =
    \mathbb E_{\pi_0(\cdot\mid Q)}
    \left[
        \frac{\pi_k}{\pi_0}
        \log_2\frac{\pi_k}{\pi_0}
    \right]
\end{equation}
The first-order term vanishes exactly because
\begin{equation}
    \mathbb E_{\pi_0(\cdot\mid Q)}
    \left[
        \frac{\pi_k}{\pi_0}-1
    \right]
    =0.
\end{equation}
Therefore the expansion can be expressed as:
\begin{equation}
    D_{\mathrm{KL}}\!\left(
        \pi_k(\cdot\mid Q)
        \,\middle\|\,
        \pi_0(\cdot\mid Q)
    \right) =
    \mathbb E_{\pi_0(\cdot\mid Q)}
    \left[
        \frac12
        \left(\frac{\pi_k}{\pi_0}-1\right)^2
        +
        \mathcal O\!\left(
            \left|
                \frac{\pi_k}{\pi_0}-1
            \right|^3
        \right)
    \right].
\end{equation}

Combining this identity with
Eq.~\eqref{eq:adiabatic-endpoint-log-ratio} and
Eq.~\eqref{eq:trajectory-strategy-kl-equality} yields
\begin{equation}
    D_{\mathrm{KL}}\!\left(
        \nu_k(\cdot\mid Q)
        \,\middle\|\,
        \nu_0(\cdot\mid Q)
    \right)
    =
    \mathcal O\!\left(\eta^2|\tau|^2\right).
\end{equation}
After averaging over $Q$,
\begin{equation}
    \boxed{
    B_k
    =
    \mathcal O\!\left(\eta^2|\tau|^2\right).}
    \label{eq:rho-kl-upper-bound}
\end{equation}
This is a finite-time consequence of the adiabatic scale separation rather
than a purely geometric consequence of the probability resolution. In
particular, because $\eta|\tau|\ll1$, the information displacement remains
second order throughout the frozen window.
\hypertarget{thm2:step1.2}{}
\subsubsection{Bounded cross entropy}

For any $S_\alpha^{(\epsilon)}$ in the partition of a fixed $Q$, the other
$m_{\epsilon,Q}-1$ atoms carry total step-$0$ mass at least
$(m_{\epsilon,Q}-1)\rho$.  Hence
\begin{equation}
    \nu_{\alpha,0}
    \le
    1-(m_{\epsilon,Q}-1)\rho.
    \label{eq:rho-reference-first-bound}
\end{equation}
Moreover,
\begin{equation}
\begin{aligned}
    m_{\epsilon,Q}\bigl[1-(m_{\epsilon,Q}-1)\rho\bigr]
    &=
    \frac{(1+\rho)^2}{4\rho}
    -
    \rho
    \left(
        m_{\epsilon,Q}-\frac{1+\rho}{2\rho}
    \right)^2
    \\
    &\le
    \frac{(1+\rho)^2}{4\rho}.
\end{aligned}
\end{equation}
Combining the last two displays gives the pointwise upper bound
\begin{equation}
    \boxed{
    \nu_{\alpha,0}
    \le
    \frac{(1+\rho)^2}{4\rho\,m_{\epsilon,Q}}.}
    \label{eq:rho-reference-upper-bound}
\end{equation}

The conditional entropy of the step-$k$ strategy obeys the exact identity
\begin{align}
    H(S_{\epsilon,k}\mid Q)
    &=-B_k
    -
    \mathbb E_Q
    \sum_{\alpha=1}^{m_{\epsilon,Q}}\nu_{\alpha,k}
    \log_2\nu_{\alpha,0}.
    \label{eq:entropy-kl-identity}
\end{align}
Using Eqs.~\eqref{eq:rho-kl-upper-bound} and
\eqref{eq:rho-reference-upper-bound},
\begin{align}
    H(S_{\epsilon,k}\mid Q)
    &\ge
    \frac1N\sum_{Q\in\mathcal D}\log_2m_{\epsilon,Q}
    -
    \log_2\frac{(1+\rho)^2}{4\rho}
    -B_k
    \\
    &\ge
    \boxed{
    \log_2m_\epsilon
    -
    \log_2\frac{(1+\rho)^2}{4\rho}-\mathcal O(\eta^2|\tau|^2).}
    \label{eq:rho-conditional-entropy-bound}
\end{align}
\hypertarget{thm2:step2}{}
\subsubsection{Collision-aware Fano step}
Let \(A=\mathbf 1\{\widehat Y(\tau)=Y(Q)\}\), so that
\(\mathbb P(A=1)=P_{acc,k}\). In addition to distinct ground-truth
answers, we assume the evaluation dataset is \(\gamma\)-separated under the
evaluated policy and strategy resolution:
\begin{equation}
\mathbb P\!\left(
\exists Q'\neq Q,\exists \tau\in S^{(\epsilon)}_{\cdot,Q}:
\widehat Y(\tau)=Y(Q')
\right)\le \gamma .
\label{eq:cross-query-answer-collision}
\end{equation}
Whenever $S^{(\epsilon)}_{\alpha,Q}$ is usable at step $k$, it is an element of the global usable-strategy set $\mathcal S_{\epsilon,k}$. Let $S_{\epsilon,k}\in\mathcal S_{\epsilon,k}$ denote the realized strategy at step $k$.

That is, except for a probability-\(\gamma\) collision event, a strategy
generated for \(Q\) cannot produce the ground-truth answer of another query.
Note that this condition does not require trajectories within the same
strategy to share identical answers.

The decoding intuition is as follows. Since the evaluation set
\(\mathcal D\) and all ground-truth answers
\(\{Y(Q'):Q'\in\mathcal D\}\) are known, the only unknown is the association
between the observed strategy set and its originating query. If the strategy
contains a trajectory that solves \(Q\), then its terminal answer identifies
\(Q\) uniquely:
\begin{equation}
\widehat Y(\tau)=Y(Q),\quad \tau\in S_{\epsilon,k}
\quad\Longrightarrow\quad
Q=\arg\max_{Q'\in\mathcal D}
\mathbf 1\{Y(Q')=\widehat Y(\tau)\},
\label{eq:strategy-decoding}
\end{equation}
except for the cross-query collision event.
Conversely, when no trajectory in the strategy recovers \(Y(Q)\), the
strategy set does not provide a reliable identifier of \(Q\). Therefore, the
decoder failure is contained in
\begin{equation}
\{\widehat Y(\tau)\neq Y(Q)\}
\cup
\{\exists Q'\neq Q,\tau\in S_{\epsilon,k}:
\widehat Y(\tau)=Y(Q')\},
\label{eq:decoder-failure-event}
\end{equation}
whose probability is at most \(1-P_{acc,k}+\gamma\).

Define a decoder
\(\widehat Q_k(S_{\epsilon,k})\) by matching the terminal answers in
\(S_{\epsilon,k}\) against \(\{Y(Q'):Q'\in\mathcal D\}\).
Then
\begin{equation}
\mathbb P\!\left(
\widehat Q_k(S_{\epsilon,k})\neq Q
\right)
\le
1-P_{acc,k}+\gamma .
\label{eq:collision-aware-decoder-error}
\end{equation}

Since \(\widehat Q_k\) is a deterministic function of
\(S_{\epsilon,k}\),
\begin{equation}
I(Q;S_{\epsilon,k})
\ge
I(Q;\widehat Q_k(S_{\epsilon,k})).
\label{eq:successful-strategy-dpi}
\end{equation}

For \(P_{acc,k}-\gamma\ge N^{-1}\), applying Fano's inequality with
Eq.~\eqref{eq:collision-aware-decoder-error} gives
\begin{equation}
H(Q|\widehat Q_k)
\le
h_2(P_{acc,k}-\gamma)
+
(1-P_{acc,k}+\gamma)\log_2(N-1),
\label{eq:fano-collision}
\end{equation}
and therefore
\begin{equation}
\boxed{
I(Q;S_{\epsilon,k})
\ge
(P_{acc,k}-\gamma)\log_2N
-
h_2(P_{acc,k}-\gamma).
}
\label{eq:successful-strategy-fano-bound}
\end{equation}

For \(P_{acc,k}-\gamma<N^{-1}\), the trivial bound
\(I(Q;S_{\epsilon,k})\ge0\) applies.
\subsubsection{Combining the two information requirements}

Both the accuracy statement and the conditional-entropy bound concern the
same step-$k$ strategy law.  Combining
Eqs.~\eqref{eq:rho-conditional-entropy-bound},
\eqref{eq:successful-strategy-dpi}, and
\eqref{eq:successful-strategy-fano-bound} gives
\begin{align}
    H(S_{\epsilon,k})
    &=
    H(S_{\epsilon,k}\mid Q)
    +I(Q;S_{\epsilon,k})
    \\
    &\ge
    \log m_{\epsilon}
    -
    \log_2\frac{(1+\rho)^2}{4\rho}-\mathcal O(\eta^2|\tau|^2)
    +
    (P_{acc,k}-\gamma)\log_2N-1.
    \label{eq:rho-total-strategy-entropy}
\end{align}
The query-specific strategy sets form the disjoint range of
$S_{\epsilon,k}$, and therefore
\begin{equation}
    H(S_{\epsilon,k})
    \le
    \log_2|\mathcal S_{\epsilon,k}|.
\end{equation}
Together with Eq.~\eqref{eq:rho-total-strategy-entropy}, this proves
\begin{equation}
    \boxed{
    |\mathcal S_{\epsilon,k}|
    \ge
    \frac{2\rho m_\epsilon}{(1+\rho)^2}
    2^{-\mathcal O(\eta^2|\tau|^2)} N^{P_{acc,k}-\gamma}
    =
    \Omega_{\rho,\gamma}\!\left(
        m_\epsilon N^{P_{acc,k}-\gamma}
    \right).}
\end{equation}
In saturated regime, we have $\rho\cdot m_\epsilon\approx1$. Given that $\mathcal O(\eta^2|\tau|^2)\ll1$, we have $2^{-\mathcal O(\eta^2|\tau|^2)}\approx1$. Therefore we have:
\begin{equation}
    \frac{2\rho m_\epsilon}{(1+\rho)^2}
    2^{-\mathcal O(\eta^2|\tau|^2)} N^{P_{acc,k}-\gamma}\sim 2N^{P_{acc,k}-\gamma}
\end{equation}
This scale implies that this bound is not vacuous within the adiabatic time zone, \textbf{making it possible for collapsed model to violate this lower bound}. The analysis above completes the resolution-conditioned proof of
Theorem~\ref{theorem:PowerLaw}.
\hfill$\blacksquare$

\begin{remark}[Entropy interpretation]
The contradiction admits a simple information-theoretic interpretation.
An initially non-collapsed model distributes probability mass over
$m_\epsilon$ viable strategies, carrying roughly $\log m_\epsilon$ of
strategy entropy. Concentrating these strategies onto one removes
\[
\Delta H_{\mathrm{req}}\approx \log m_\epsilon .
\]
However, within the adiabatic regime, the amount of entropy that can be
recovered through the slowly evolving strategy distribution is only
$\Delta H_{\mathrm{ad}}\sim B_k=\mathcal O(\eta^2|\tau|^2)=o(\log m_\epsilon)$.
Hence the entropy loss induced by strategy concentration eventually
exceeds the available adiabatic budget, making the non-collapse regime
unsustainable and forcing an accuracy decline.
\end{remark}

\clearpage
\hypertarget{proof:appB}{}
\section{{Appendix B: On the Properties of MEI}}\label{app:mei}
\hyperlink{back:mei}{Back}
\subsection{{Requirements for a High Quality Metrics}}

Before the formal discussion of the Mirrored Entanglement Index (MEI), we need to first examine, under \textbf{WHAT} requirements a metrics is of high quality.

\begin{itemize}
\item The metrics should share the same evolutionary tendency with the dynamics proposed in Appendix A.
\item The metric should be easy to compute.
\item The data used to compute the metrics should be easy to acquire and poses no pressure on GPU memories.
\end{itemize}
\subsection{{The Convergence of Gradients}}
\begin{bluebox}
\begin{theorem}[The concentration of gradients]
Fix a prompt $Q$ and an $\epsilon$-strategy partition. 
Suppose that after strategy collapse, a dominant strategy $S$ carries probability mass
\begin{equation}
\nu_S
=
\mathbb{P}_{\tau\sim\pi_\theta(\cdot|Q)}(\tau\in S)
\ge 1-\zeta .
\end{equation}
Sample $G$ trajectories
$\{\tau_1,\ldots,\tau_G\}\sim\pi_\theta(\cdot|Q)$, define
\begin{equation}
K_{\max}\triangleq \sup_{a,b}|K_{ab}|.
\end{equation}
Then
\begin{equation}
\mathbb{P}\left(
\max_{i,j}
\|U_i-U_j\|
>
\sqrt{2\epsilon K_{\max}}
\right)
\le
1-\nu_S^G
\le
G(1-\nu_S)
\le
G\zeta .
\end{equation}
Therefore, as the strategy distribution concentrates and
$\zeta\rightarrow 0$, the sampled parameter-space gradients concentrate
inside an $O(\sqrt{\epsilon K_{\max}})$ neighborhood with probability
approaching one.
\end{theorem}
\end{bluebox}

\begin{proof}
Let
\[
\mathcal{A}
\triangleq
\{\tau_1,\ldots,\tau_G\in S\}
\]
be the event that all sampled trajectories belong to the dominant strategy.
Since each trajectory is sampled from $\pi_\theta(\cdot|Q)$,
\[
\mathbb{P}(\mathcal{A})
=
\nu_S^G,
\]
and hence
\[
\mathbb{P}(\mathcal{A}^c)
=
1-\nu_S^G
\le
G(1-\nu_S)
\le
G\zeta .
\]

Conditioned on $\mathcal{A}$, every pair
$\tau_i,\tau_j$ belongs to the same strategy $S$.
By Definition~1, dynamically equivalent trajectories have similar
coupling profiles. In particular,
\[
|K_{ii}-K_{ji}|
\le
\epsilon K_{\max},
\qquad
|K_{ij}-K_{jj}|
\le
\epsilon K_{\max}.
\]
Since $K_{ij}=K_{ji}$,
\begin{align}
\|U_i-U_j\|^2
&=
K_{ii}+K_{jj}-2K_{ij} \\
&=
(K_{ii}-K_{ji})-(K_{ij}-K_{jj}) \\
&\le
|K_{ii}-K_{ji}|
+
|K_{ij}-K_{jj}| \\
&\le
2\epsilon K_{\max}.
\end{align}
Therefore, on $\mathcal{A}$,
\[
\max_{i,j}\|U_i-U_j\|
\le
\sqrt{2\epsilon K_{\max}}.
\]
Consequently,
\[
\mathbb{P}\left(
\max_{i,j}
\|U_i-U_j\|
>
\sqrt{2\epsilon K_{\max}}
\right)
\le
\mathbb{P}(\mathcal{A}^c)
\le
G\zeta .
\]
This establishes parameter-gradient concentration as a direct consequence
of strategy concentration, without requiring convergence to a single
trajectory.
\end{proof}
\subsection{{Heuristic Insight}}
Consider a rollout process, during which the model has already reached a collapsed state. According to definition, a model in this state could only generate trajectories inside one strategy, or more specifically:

\begin{equation}
\forall \tau_i, \tau_j \in S, \|\nabla_\theta \log(\tau_i |Q) - \nabla_\theta \log(\tau_j |Q)\|\le\epsilon  
\end{equation}

Based on the fact that all trajectories are highly aligned in the back-propagation process, we can try this parameter:

\begin{equation}
\alpha = \frac{\left(\sum_{i}\nabla_\theta \log(\tau_i |Q)\right)^2}{\sum_{i}\nabla_\theta \log(\tau_i |Q)^2}
\end{equation}

when the generated trajectories have reached a highly aligned state, it is obvious that $\alpha\to G$, where $G$ is the rollout times; more importantly, the more the trajectories are aligned, the more they approaches to $G$. Therefore, this index $\alpha$ has already met with our requirement 1. 

\subsection{{The Formal Definition of MEI and its Convergence}}
In subsection B.2, we have already derived a parameter that could monitor the collapse process. However, given a second thought, it can't be ignored that the vectors used to calculate $\alpha$ are too big in size, which could easily consume $64B$ of the total GPU memory. Therefore, in order to meet requirement 2 and 3, it is imperative to construct a low-dimension projection of $\alpha$, which leads to our formal definition of MEI.

\begin{definition}[The Mirrored Entanglement Index]\leavevmode\\
Consider a rollout process, where the model generates $G$ trajectories. For each trajectory $\tau_i$, compute its log-prob gradient on logits space $\nabla_z\log(\tau_i|Q)$. \\
The Mirrored Entanglement Index is then defined as:
\begin{equation}
MEI = \frac{\left(\sum_{i}\nabla_z \log(\tau_i |Q)\right)^2}{\sum_{i}\nabla_z \log(\tau_i |Q)^2}
\end{equation}
\end{definition}

This index is both easy to compute and easy to acquire, the $\nabla_z$ is actually the derivatives on latent space, which could be easily stored in the first stage of back-propagation. Also, the dimensions of all these vectors are equal to latent space dimension $D$; therefore, the dot product could be done in time complexity $O(G\cdot D^2)$, which is almost negligible compared with full-attention calculation and empowered with GPUs who are experts in matrix calculation.

In short, the MEI index has already satisfied requirements 2 and 3. Its coherent evolutionary tendency with the dynamics is guaranteed by the following theorem.

\begin{theorem}[The convergence of logit gradients]\leavevmode\\
Let $v(\tau_i) = \nabla_z\log(\tau_i|Q)$ and $g(\tau_i) = \nabla_\theta \log(\tau_i|Q)$.if:
\begin{equation}
\|v(\tau_i) -v(\tau_j)\|\to0
\end{equation}
then:
\begin{equation}
\|g(\tau_i) - g(\tau_j)\|\to0
\end{equation}
\end{theorem}

\textbf{Proof.}
According to chain rule, and based on the fact that all layers of the neural network forms a continuous mapping, there exists a continuous mapping between $g(\tau_i)$ and $v(\tau_i)$:
\begin{equation}
g(\tau_i) = v(\tau_i)\cdot\frac{\partial z}{\partial \theta}
\end{equation}

If this mapping is stable, then by applying Lipschitz continuity property, the theorem has already been solved. Fortunately, after a decade of iteration, modern LLM designs have already controlled problems of gradient explosion, which could be formalized by the following lemma

\begin{lemma}
The module of the Jacobian $J \triangleq \frac{\partial z}{\partial \theta}$ is bounded:
\begin{equation}
\|J\| \le C_1
\end{equation}
\end{lemma}

To complete the proof, we consider the Moore-Penrose pseudo-inverse of the Jacobian $J^\dagger = \frac{\partial z}{\partial \theta}^\dagger$. Since $g = v \cdot J$, we can express the latent-space gradient as $v = g \cdot J^\dagger$.
(See the full derivation of Lemma 1 and 2 in B.3.1)

Based on Lemma 2, we can apply the Lipschitz continuity for the inverse mapping:
\begin{equation}\begin{aligned}\|g(\tau_i) - g(\tau_j)\| &= \|(v(\tau_i) - v(\tau_j)) \cdot J\| \le \|v(\tau_i) - v(\tau_j)\| \cdot \|J\| \le C_2 \cdot \|v(\tau_i) - v(\tau_j)\|\end{aligned}\end{equation}
Since $\|g(\tau_i) - g(\tau_j)\| \to 0$ and $C_2$ is a constant determined by the model architecture's stability (e.g., the conditioning of the weight matrices), it follows that:\begin{equation}\|g(\tau_i) - g(\tau_j)\| \to 0\end{equation} This concludes the proof that the collapse in parameter space is faithfully reflected in the latent Logit space. 

$Q.E.D.$

\subsubsection{Proof of Supporting lemmas 1 \& 2}
\textbf{Proof of Lemma 1.}
Consider each layer of the Transformer as a continuous mapping, which could be formalized as:

\begin{equation}
f_n\in C^1: x_{n+1}=f_n(x_n)
\end{equation}

Then the final mapping of the model could be written as:

\begin{equation}
z =f_L(f_{L-1}(...(f_1(\theta))...))
\end{equation}

Therefore, the whole Jacobian we concern could be expanded using chain rule:

\begin{equation}
J = \prod_{i}^{N-1}\frac{\partial h^{i+1}}{\partial h^i}\cdot \frac{\partial h^1}{\partial \theta}
\end{equation}

In order to prove that J is capped in its module, we only need to ensure that all the Jacobian layers could be controlled by a constant.

In modern LLM designs, one single layer could be conceptualized as:

\begin{tikzpicture}[
    scale=0.6,
    transform shape,
    node distance=1.2cm and 1.5cm,
    block/.style={rectangle, draw, fill=blue!10, text width=2.2cm, align=center, minimum height=1cm, rounded corners=2pt, font=\sffamily},
    norm/.style={rectangle, draw, fill=yellow!20, text width=2cm, align=center, minimum height=0.8cm, font=\sffamily\small},
    point/.style={circle, fill, inner sep=1pt},
    arrow/.style={-Stealth, thick}
]

    \node (input) {Input $x$};
    
    \node [point, right=of input] (p1) {};
    \node [block, right=of p1] (mha) {Multi-Head\\Attention};
    \node [norm, right=of mha] (add1) {Add \& Norm};
    
    \node [point, right=of add1] (p2) {};
    \node [block, right=of p2] (ffn) {Feed Forward};
    \node [norm, right=of ffn] (add2) {Add \& Norm};
    
    \node [right=of add2] (output) {Output};

    \draw [arrow] (input) -- (mha);
    \draw [arrow] (mha) -- (add1);
    \draw [arrow] (add1) -- (ffn);
    \draw [arrow] (ffn) -- (add2);
    \draw [arrow] (add2) -- (output);

    \draw [arrow, dashed, rounded corners=5pt] (p1.center) -- ++(0, 0.8) -| node[pos=0.25, above] {Residual} (add1.north);
    
    \draw [arrow, dashed, rounded corners=5pt] (p2.center) -- ++(0, 0.8) -| node[pos=0.25, above] {Residual} (add2.north);

\end{tikzpicture}

In the following part, we will show that each part in the graph above, could be controlled by a constant:

\textbf{Norm Layer:}
The mathematical definition of the Norm Layer could be written as:
\begin{equation}
LN(x) = \gamma \odot \frac{f(x) - \mu}{\sqrt{\sigma^2 + \epsilon}} + \beta
\end{equation}
where $f(x) \triangleq x + f_L(x)$ is the residue term.

Using chain rule, we can get:
\begin{equation}
\frac{\partial LN(x)}{\partial x} = \frac{\gamma}{\sqrt{\sigma^2 + \epsilon}} \odot \frac{\partial f}{\partial x}
\end{equation}

Therefore, the bounds could be easily written as:
\begin{equation}
\|J_{LN}\| \le \|C_{LN}\| \cdot \|J_f\|
\end{equation}

\textbf{Residue Connections:}
The mathematical expression of residue connections has already been given before:
\begin{equation}
f(x) = x + f_L(x)
\end{equation}

Using chain rule, we can easily get the bounds in Residue Connection:
\begin{equation}
\|J_{f}\| = \|I+J_L\|\le 1+\|J_L\|
\end{equation}

\textbf{Multi-head Attention:}
The mathematical expression for the Attention layer could be written as:
\begin{equation}
Attn(Q,K,V) = Softmax(QK^\top/\sqrt d)\cdot V
\end{equation}

Therefore, the Jacobian could be written as:
\begin{equation}
\|J_A\| = \|J_{softmax}\|\cdot2\|XW_QW_K^\top\|/\sqrt{d}
\end{equation}

There is a famous conclusion that: $\|J_{softmax}\|\le1$

Therefore, so long as the attention kernels are bounded during pre-training process (which is a must or it could lead to gradient explosion), and that the input $X$ is also bounded, then the attention head is bounded:
\begin{equation}
\|J_A\| \le C_A
\end{equation}

\textbf{Feed Forward Layer:}
The Feed Forward Layer often uses a combination of GELU gates and Linear layers, which could be written as:
\begin{equation}
FF(x) = W_2\cdot GELU(W_1X+b_1)+b_2
\end{equation}

since the gradient of GELU is smaller than 1, we can control the FF gradients by:
\begin{equation}
\|J_{FF}\| = \|W_2\| \cdot \|W_1\|
\end{equation}

Eq. (24), (26), (29) and (31) together demonstrated that all the Jacobians inside a single Transformer Layer could be controlled by a cap, which is also foreseeable from real pre-training experiences.

Combine all these Layers, and we can draw an upper limit to the Jacobian that we originally concerned:
\begin{equation}
\|J\|\le \mathcal{O}(C^N)
\end{equation}
$(Q.E.D.)$

\clearpage
\hypertarget{proof:exp2}{}
\section{{Appendix C: Analysis on Different Regularization Methods}}
\subsection{Proof for Theorem 3}\label{app:kl_js_fail}
\begin{bluebox}
\hyperlink{back:exp2}{\color{blue}{Back}} \textbf{Proof Sketch For Theorem 3.} From the derivation process of \textbf{Theorem~\ref{theorem:concentration}}, in equation~\ref{eq:phi-definition}, we have constructed a potential function that is proved to have an almost sure positive drift during the optimization process:
\begin{equation}
    \Phi_k
    \triangleq
    \frac1{|\mathcal D|}
    \sum_{Q\in\mathcal D}
    \sum_{\alpha=1}^{m_Q}
    a_{\alpha,Q}\nu_{\alpha,Q,k}.
\end{equation}
in which $a_{\alpha,Q}$ is proved to be an adiabatic constant within adiabatic time zone $k\ll|\tau|.$ (Lemma~\ref{lemma:existence_of_strategy_reward})

From the derivation process of \textbf{Theorem~\ref{theorem:PowerLaw}}, in equation~\ref{eq:trajectory-strategy-kl-equality}, we proved that a divergence on trajectories is equal to the same divergence on strategies:
\begin{equation}
D_{KL}(\pi_\theta(\cdot|Q)\|\pi_{ref}(\cdot|Q))=D_{KL}(\nu(\cdot|Q)\|\nu_{ref}(\cdot|Q))
\end{equation}
as in-strategy trajectory distribution $q_i$ is also proved to be adiabatic in this time zone. (Lemma~\ref{lemma:adiabatic_property_of_strategies}) The result can expand to forward KL and JS divergence with same method.

Based on these two results, we obtain an \textbf{effective strategy-level RL system} under scale requirements $\eta\ll\epsilon_{clip}<1\ll|\tau|\ll\eta^{-1}\ll|\tau|^3$, and within time zone \(k\ll|\tau|\), where a fixed reward $R(S_\alpha)=a_{\alpha,Q}$ is directly imposed on the strategies. From Lemma~\ref{lemma:impossiblity_of_equivalence}, these rewards are not likely to be exactly equal to each other in real practice, reducing the system to a convex optimization problem:
\begin{equation}
    \boxed{\mathcal J(\nu)=\mathbb E_\nu[R]-\beta D(\nu,\nu_{ref})}
\label{eqa:easy_mode}
\end{equation}
The proof can then be completed in three direct calculations:

\textbf{Step 1: Why Reverse KL Divergence Fails.} Optimizing on equation~\ref{eqa:easy_mode}, we find the equilibrium follows a Gibbs distribution.

\textbf{Step 2: Why Forward KL Divergence Fails.} Optimizing on equation~\ref{eqa:easy_mode}, we find the Forward KL Divergence vulnerable to excessive rewards.

\textbf{Step 3: Why JS Divergence Fails.} Optimizing on equation~\ref{eqa:easy_mode}, we find JS Divergence extremely vulnerable to excessive rewards.
\end{bluebox}
For clarity and simplicity, we will first show the main experimental results that expose the drawbacks of each conventional divergence forms.
\begin{figure}[h]
\begin{center}
\includegraphics[scale=0.15]{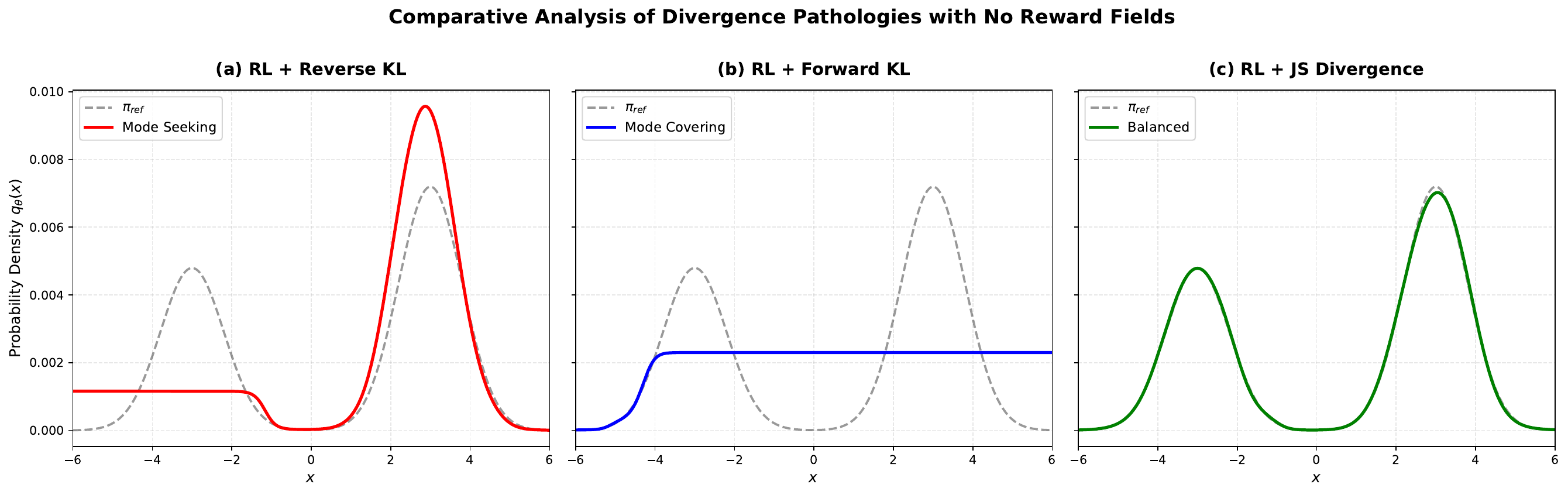}
\end{center}
\caption{Demonstrative Experimental results of fitting a multi-mode distribution with a 24-Parameter MLP (No rewards).}
\label{fig:collapse}
\end{figure}
\begin{figure}[h]
\begin{center}
\includegraphics[scale=0.20]{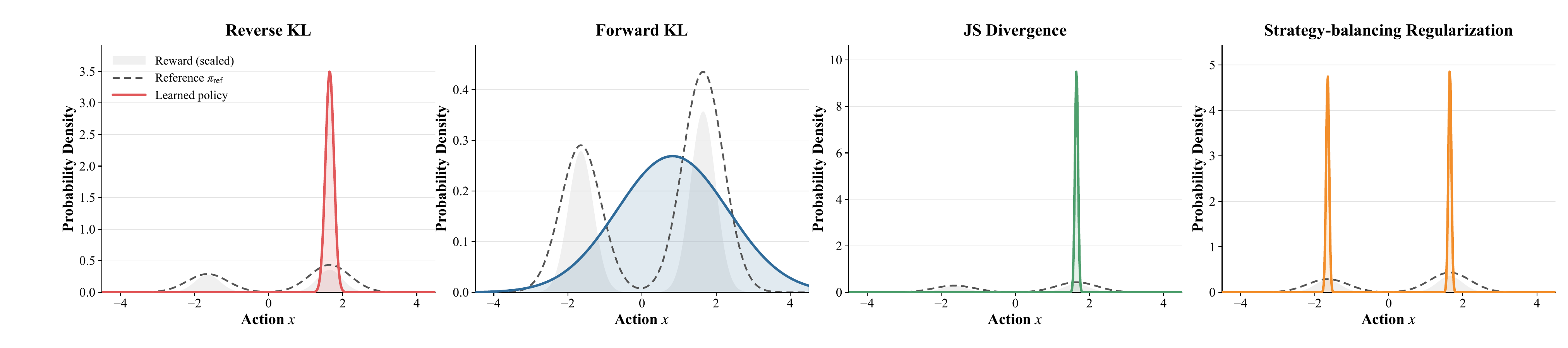}
\end{center}
\caption{Experimental fitting of a static multimodal distribution with a 24-parameter MLP under rewards, using KL, JS and Strategy-balancing Regularization}
\label{fig:collapse}
\end{figure}
\subsubsection{Step 1: Why does Reverse KL Divergence fail}

Consider the effective strategy-level objective derived in Theorem~\ref{theorem:divergence_failure}:
\begin{equation}
\max_{\nu}
\mathbb E_\nu[R]
-\beta D_{\rm KL}(\nu\|\nu^{ref}),
\end{equation}
where
\begin{equation}
D_{\rm KL}(\nu\|\nu^{ref})
=
\sum_\alpha
\nu_\alpha
\log\frac{\nu_\alpha}{\nu_\alpha^{ref}}.
\end{equation}

Introducing a Lagrange multiplier $\lambda$ for the simplex constraint
$\sum_\alpha\nu_\alpha=1$, the stationary condition gives
\begin{equation}
R(S_\alpha)
-\beta
\left(
\log\frac{\nu_\alpha}{\nu_\alpha^{ref}}+1
\right)
-\lambda=0.
\end{equation}

Therefore, the optimal strategy distribution follows the Gibbs form:
\begin{equation}
\nu_\alpha^D
=
\frac1Z
\nu_\alpha^{ref}
\exp\left(\frac{R(S_\alpha)}{\beta}\right),
\end{equation}
where
\begin{equation}
Z=\sum_\alpha\nu_\alpha^{ref}
\exp\left(\frac{R(S_\alpha)}{\beta}\right)
\end{equation}
is the normalization constant.

Consider $o=\arg\max_\alpha R(S_\alpha)$, then, using $R(S_o)-R(S_\alpha)\ge\Delta_R$:
\begin{equation}
\nu_o^D\ge\frac{\nu_o^{ref}}{\nu_o^{ref}+(1-\nu_o^{ref})\exp(-\frac{\Delta_R}{\beta})}
\end{equation}

If we want $\nu_o^{D}\ge1-\epsilon$, we can obtain:
\begin{equation}
\frac{\Delta_R}{\beta}\ge\ln\frac{(1-\nu_o^{ref})(1-\epsilon)}{\nu_o^{ref}\epsilon}
\end{equation}

Note that $\nu^{ref}$ is a full-support reference distribution, the right hand side is surely bounded:
\begin{equation}
\ln\frac{(1-\nu_o^{ref})(1-\epsilon)}{\nu_o^{ref}\epsilon}\le C_D
\end{equation}

Therefore:
\begin{equation}
    \frac{\Delta_R}{\beta}\ge C_D
\end{equation}
yields this near-one-hot distribution.

\subsubsection{Step 2: Why Forward KL Divergences fails}
Let $\nu_i^{\mathrm{ref}}>0$ be the probability mass of the reference model on our selected strategy $S_i$. Ignoring constants independent of $\nu$, the regularized objective is
\begin{equation}
\max_{\nu\in\Delta}
\left\{
\sum_\alpha \nu_\alpha R_\alpha+\beta\sum_\alpha \nu^{ref}_\alpha\log\nu_\alpha
\right\}.
\end{equation}
Since $\nu^{\mathrm{ref}}$ has full support, the optimum is interior. The KKT condition gives
\begin{equation}
R(S_\alpha)+\beta\frac{\nu^{ref}_\alpha}{\nu_\alpha}-\lambda=0,
\qquad
\nu_i^D=\frac{\beta \nu^{ref}_\alpha}{\lambda-R(S_\alpha)},
\end{equation}
with $\lambda>R_o\triangleq\max_\alpha R(S_\alpha)$, which was exerted on strategy $S_o$. Let $x=\lambda-R_o>0$ and
$\Delta_\alpha=R_o-R(S_\alpha)\ge\Delta_R$ for $\alpha\neq o$. Then
\begin{equation}
1-\nu_o^D
=
\sum_{\alpha\neq o}\frac{\beta \nu^{ref}_\alpha}{x+\Delta_\alpha}
\le
\frac{\beta}{\Delta_R}\sum_{\alpha\neq o}\nu^{ref}_\alpha
=
\frac{\beta(1-\nu^{ref}_o)}{\Delta_R}.
\end{equation}
Hence, if
\begin{equation}
\frac{\Delta_R}{\beta}
\ge
\frac{1-\nu^{ref}_o}{\epsilon},
\end{equation}
then $\nu_o^D\ge1-\epsilon$. Therefore one may take
\begin{equation}
C_D=\frac{1-\nu_o^{\mathrm{ref}}}{\epsilon}<\infty.
\end{equation}

\subsubsection{Step 3: Why JS Divergence fails}
Consider the effective strategy-level optimization derived in
Theorem~\ref{theorem:divergence_failure}:
\begin{equation}
\max_{\nu}
\mathbb E_\nu[R]
-\beta D_{\rm JS}(\nu\|\nu^{\rm ref}),
\end{equation}
where $R(S_\alpha)$ is the effective strategy reward and
$\nu^{\rm ref}$ is a full-support reference distribution. The JS divergence is defined as
\begin{equation}
D_{\rm JS}(\nu\|\nu^{\rm ref})
=
\frac12
D_{\rm KL}\left(
\nu\middle\|\frac{\nu+\nu^{\rm ref}}2
\right)
+
\frac12
D_{\rm KL}\left(
\nu^{\rm ref}\middle\|\frac{\nu+\nu^{\rm ref}}2
\right).
\end{equation}

Introducing a Lagrange multiplier $\lambda$ for the simplex constraint
$\sum_\alpha\nu_\alpha=1$, the Lagrangian is
\begin{equation}
\mathcal L(\nu,\lambda)
=
\sum_\alpha \nu_\alpha R(S_\alpha)
-\beta D_{\rm JS}(\nu\|\nu^{\rm ref})
-\lambda(\sum_\alpha\nu_\alpha-1).
\end{equation}

The derivative of the JS divergence with respect to $\nu_\alpha$ is
\begin{equation}
\frac{\partial}{\partial\nu_\alpha}
D_{\rm JS}(\nu\|\nu^{\rm ref})
=
\frac12
\log
\frac{2\nu_\alpha}
{\nu_\alpha+\nu_\alpha^{\rm ref}}.
\end{equation}

Therefore, the stationary condition gives
\begin{equation}
R(S_\alpha)
-
\frac{\beta}{2}
\log
\frac{2\nu_\alpha}
{\nu_\alpha+\nu_\alpha^{\rm ref}}
-\lambda=0 .
\end{equation}

Rearranging,
\begin{equation}
\frac{2\nu_\alpha}
{\nu_\alpha+\nu_\alpha^{\rm ref}}
=
\exp\left(
\frac{2(R(S_\alpha)-\lambda)}{\beta}
\right).
\end{equation}

Let
\begin{equation}
\gamma_\alpha
=
k\exp\left(\frac{2R(S_\alpha)}{\beta}\right),
\end{equation}
where
\begin{equation}
k=\exp(-2\lambda/\beta)>0.
\end{equation}
Solving for the optimal strategy distribution gives
\begin{equation}
\nu_\alpha^D
=
\frac{\gamma_\alpha}{2-\gamma_\alpha}
\nu_\alpha^{\rm ref}
=
\frac{
k\exp(2R(S_\alpha)/\beta)
}{
2-k\exp(2R(S_\alpha)/\beta)
}
\nu_\alpha^{\rm ref}.
\end{equation}
Unlike reverse KL, where the normalization constant appears as a
decoupled partition function, JS regularization entangles the global
normalization factor with the denominator. For the optimal distribution
to remain valid, it requires
\begin{equation}
k\exp\left(\frac{2R(S_\alpha)}{\beta}\right)<2,
\quad\forall\alpha .
\end{equation}

To make the concentration result explicit, let
\[
o=\arg\max_\alpha R(S_\alpha),
\]
so that \(R(S_o)-R(S_\alpha)\ge \Delta_R\) for every \(\alpha\neq o\).
For any strategy distribution \(\nu\),
\begin{equation}
\mathbb E_\nu[R]
\le
R(S_o)-\Delta_R(1-\nu_o).
\end{equation}
Let \(e_o\) denote the point mass on the best strategy \(S_o\). Since
\(\nu^D\) is optimal,
\begin{equation}
\mathcal J(\nu^D)
\ge
\mathcal J(e_o)
=
R(S_o)-\beta D_{\rm JS}(e_o\|\nu^{\rm ref}).
\end{equation}
On the other hand, using \(D_{\rm JS}(\nu^D\|\nu^{\rm ref})\ge0\),
\begin{equation}
\mathcal J(\nu^D)
\le
R(S_o)-\Delta_R(1-\nu_o^D).
\end{equation}
Therefore,
\begin{equation}
\Delta_R(1-\nu_o^D)
\le
\beta D_{\rm JS}(e_o\|\nu^{\rm ref})
\le
\beta\log 2,
\end{equation}
where the last inequality follows from the boundedness of Jensen--Shannon
divergence. Hence, if
\begin{equation}
\frac{\Delta_R}{\beta}
\ge
\frac{\log 2}{\epsilon},
\end{equation}
then
\begin{equation}
\nu_o^D\ge1-\epsilon.
\end{equation}
Thus one may take
\begin{equation}
C_D=\frac{\log 2}{\epsilon}<\infty,
\end{equation}
which proves that a fixed JS regularizer cannot universally prevent
near-one-hot strategy concentration when the reward gap is sufficiently large.

This provides the JS divergence failure mode
predicted by Theorem~\ref{theorem:divergence_failure}.

\clearpage
\hypertarget{proof:m_reg}{}
\subsection{Why Mesh Learning Works}\label{app:mesh_learning}

\subsubsection{Loss Shaping}
\label{b3.-the-loss-shaping}
\hyperlink{back:m_reg}{Back} 
For a prompt $Q$, suppose that we use $m\ge 2$ strategies
$S_1,\ldots,S_m$. Let $n_i$ be the set of successful trajectories
sampled from strategy $i$, with $|n_i|>0$. We compute their mean
log-probabilities under the current model and a fixed reference model:
\begin{equation}
\begin{aligned}
 \bar z_i(\theta)
 &=\frac{1}{|n_i|}\sum_{j\in n_i}z_{ij}(\theta),\\
 b_i
 &=\bar z_i^{\mathrm{ref}}
 =\frac{1}{|n_i|}\sum_{j\in n_i}z_{ij}^{\mathrm{ref}}.
\end{aligned}
\label{eq:mean-current-reference}
\end{equation}
The reference score is detached from the gradient. Strategy-balancing 
Regularization uses the change relative to the reference model:
\begin{equation}
 r_i=\bar z_i-b_i,\qquad
 p_i=\frac{\exp(r_i)}
 {\sum_{j=1}^{m}\exp(r_j)}.
\label{eq:reference-softmax}
\end{equation}
Thus, $p$ is a softmax distribution over reference-centered scores.
Let $u=m^{-1}\mathbf 1$. The loss minimized by gradient descent is
\begin{equation}
\begin{aligned}
 \mathcal J_{S}
 &=\mu D_{\mathrm{KL}}(u\|p)\\
 &=\frac{\mu}{m}\sum_{i=1}^{m}
   \log\frac{1/m}{p_i},
 \qquad \mu>0.
\end{aligned}
\label{eq:mutation-loss-reference}
\end{equation}
Equivalently, a maximization objective contains
$-\mu D_{\mathrm{KL}}(u\|p)$. We use the minimization convention below.

The analysis uses the deterministic continuous-time approximation with
a fixed reference model, fixed score Jacobians, and a fixed mean GRPO
drift. Sampling noise and feature drift are not included in this
approximation.

\subsubsection{Cross-Prompt Dynamics and the Global Attractor}
\label{b4.-why-cross-entropy-works-here}

\begin{bluebox}
\textbf{Proposition 1 (global attraction of Strategy-Balancing dynamics).}

Consider the update dynamics (which will be derived later on)
\begin{equation}
\Delta \bar z_{\alpha,Q}
=
\eta f_{\alpha,Q}
+
\frac{\eta_1\mu}{|\mathcal D'|}
\sum_{Q'\in\mathcal D'}
\sum_{\beta=1}^{m_{Q'}}
\overline U_{\alpha,Q}^{\top}
\overline U_{\beta,Q'}
\left(
\frac{1}{m_{Q'}}-p_{\beta,Q'}
\right).
\label{eq:mesh-dynamics}
\end{equation}
If the induced $p$-dynamics admit a full-support equilibrium $p^*$,
i.e.,
\[
p_{\alpha,Q}^*>0,
\qquad
\forall Q\in\mathcal D',\
\forall \alpha\in\{1,\ldots,m_Q\},
\]
then every trajectory initialized from finite logits is attracted to the
set of full-support equilibria. Consequently, for the limiting strategy
distribution $\nu^*$,
\begin{equation}
\boxed{
\forall Q\in\mathcal D',\
\forall \alpha\in\{1,\ldots,m_Q\},
\qquad
\nu_{\alpha,Q}^*>0.
}
\end{equation}
Thus, Strategy-Balancing Regularization preserves all represented
strategies whenever a full-support equilibrium exists. The existence of
such an equilibrium depends on the dataset $\mathcal D'$ and the
regularization strength $\mu$, for which we make no universal claim.
\end{bluebox}
\textbf{Proof.}
We analyze the deterministic dataset-average dynamics within the same
adiabatic window as above, where the strategy-level fitness and score
vectors are frozen.  To simplify notation, we suppress the step index $k$.

For each training query $Q\in\mathcal D'$, let
$\{S_{\alpha,Q}\}_{\alpha=1}^{m_Q}$ denote its active strategies and define
\begin{equation}
    \nu_{\alpha,Q}
    \triangleq
    \sum_{\tau_i\in S_{\alpha,Q}}
    \pi(\tau_i\mid Q).
\end{equation}
Let $\mathcal C_{\alpha,Q}\subseteq S_{\alpha,Q}$ be the nonempty set of successful
sampled trajectories from strategy $S_{\alpha,Q}$, and define
\begin{equation}
\bar z_{\alpha,Q}
\triangleq
\frac{1}{|\mathcal C_{\alpha,Q}|}
\sum_{\tau_i\in \mathcal C_{\alpha,Q}}
\log \pi(\tau_i\mid Q),
\qquad
b_{\alpha,Q}
\triangleq
\frac{1}{|\mathcal C_{\alpha,Q}|}
\sum_{\tau_i\in \mathcal C_{\alpha,Q}}
\log \pi_{\mathrm{ref}}(\tau_i\mid Q),
\end{equation}
where $b_{\alpha,Q}$ is fixed and detached from the gradient.  Strategy
Balancing acts on the reference-centered auxiliary distribution
\begin{equation}
p_{\alpha,Q}
=
\frac{\exp(\bar z_{\alpha,Q}-b_{\alpha,Q})}
{\sum_{\lambda=1}^{m_Q}
 \exp(\bar z_{\lambda,Q}-b_{\lambda,Q})}.
\label{eq:mesh-p}
\end{equation}

For each query, the regularization loss is
\begin{equation}
\mathcal L_{\mathrm{mut},Q}
=
\frac{\mu}{m_Q}
\sum_{\alpha=1}^{m_Q}
\log\frac{1/m_Q}{p_{\alpha,Q}},
\end{equation}
so that
\begin{equation}
\frac{\partial\mathcal L_{\mathrm{mut},Q}}
{\partial(\bar z_{\alpha,Q}-b_{\alpha,Q})}
=
\mu\left(p_{\alpha,Q}-\frac1{m_Q}\right).
\label{eq:mesh-gradient}
\end{equation}

Let
\begin{equation}
\overline U_{\alpha,Q}
\triangleq
\nabla_\theta \bar z_{\alpha,Q}
=
\frac{1}{|\mathcal C_{\alpha,Q}|}
\sum_{\tau_i\in \mathcal C_{\alpha,Q}} U_i(Q).
\end{equation}
Averaging the regularizer over the training set, the reference-centered
score obeys
\begin{equation}
\Delta \bar z_{\alpha,Q}
=
\eta f_{\alpha,Q}
+
\frac{\eta\mu}{|\mathcal D'|}
\sum_{Q'\in\mathcal D'}
\sum_{\beta=1}^{m_{Q'}}
\overline U_{\alpha,Q}^{\top}
\overline U_{\beta,Q'}
\left(
    \frac1{m_{Q'}}-p_{\beta,Q'}
\right).
\label{eq:mesh-score-dynamics}
\end{equation}
Since
\begin{equation}
p_{\alpha,Q}
=
\frac{\exp(\bar z_{\alpha,Q}-b_{\alpha,Q})}
{\sum_{\lambda=1}^{m_Q}
\exp(\bar z_{\lambda,Q}-b_{\lambda,Q})},
\end{equation}
and $b_{\alpha,Q}$ is fixed, differentiating the softmax gives
\begin{equation}
\Delta p_{\alpha,Q}
=
p_{\alpha,Q}
\left(
\Delta{\bar z}_{\alpha,Q}
-
\sum_{\lambda=1}^{m_Q}
p_{\lambda,Q}\Delta{\bar z}_{\lambda,Q}
\right).
\label{eq:mesh-p-dynamics}
\end{equation}
Substituting the score dynamics above,
\begin{equation}
\begin{aligned}
\Delta p_{\alpha,Q}
=
p_{\alpha,Q}\Bigg[
&\eta
\left(
f_{\alpha,Q}
-
\sum_{\lambda=1}^{m_Q}
p_{\lambda,Q}f_{\lambda,Q}
\right)\\
&+
\frac{\eta\mu}{|\mathcal D'|}
\sum_{Q'\in\mathcal D'}
\sum_{\beta=1}^{m_{Q'}}
\left(
\overline U_{\alpha,Q}
-
\sum_{\lambda=1}^{m_Q}
p_{\lambda,Q}\overline U_{\lambda,Q}
\right)^{\!\top}
\overline U_{\beta,Q'}
\left(
\frac1{m_{Q'}}-p_{\beta,Q'}
\right)
\Bigg].
\end{aligned}
\label{eq:mesh-p-dynamics-expanded}
\end{equation}

Given the existence of a full-support equilibrium $p^*$, we show that
Strategy-Balancing Regularization induces a global attraction. Consider
\begin{equation}
\mathcal V(p)
\triangleq
\frac1{|\mathcal D'|}
\sum_{Q\in\mathcal D'}
D_{\mathrm{KL}}(p_Q^*\|p_Q).
\label{eq:mesh-lyapunov}
\end{equation}
Using Eq.~\eqref{eq:mesh-p-dynamics}, the change of $\mathcal V$ is
\begin{align}
\Delta\mathcal V
&=
-\frac1{|\mathcal D'|}
\sum_{Q,\alpha}
p_{\alpha,Q}^*
\frac{\Delta p_{\alpha,Q}}{p_{\alpha,Q}}
\nonumber\\
&=
-\frac1{|\mathcal D'|}
\sum_{Q,\alpha}
p_{\alpha,Q}^*
\left(
\Delta\bar z_{\alpha,Q}
-
\sum_{\lambda=1}^{m_Q}
p_{\lambda,Q}\Delta\bar z_{\lambda,Q}
\right)
\nonumber\\
&=
\frac1{|\mathcal D'|}
\sum_{Q,\alpha}
\left(
p_{\alpha,Q}-p_{\alpha,Q}^*
\right)
\Delta\bar z_{\alpha,Q}.
\label{eq:mesh-lyapunov-difference}
\end{align}

Because $p^*$ is a full-support equilibrium, its score update is common
across all strategies of the same prompt: for each $Q$, there exists
$c_Q$ such that
\begin{equation}
\Delta\bar z_{\alpha,Q}^*=c_Q,
\qquad
\forall \alpha\in\{1,\ldots,m_Q\}.
\end{equation}
Moreover,
\begin{equation}
\sum_{\alpha=1}^{m_Q}
\left(
p_{\alpha,Q}-p_{\alpha,Q}^*
\right)=0.
\end{equation}
Hence
\begin{equation}
\sum_{\alpha=1}^{m_Q}
\left(
p_{\alpha,Q}-p_{\alpha,Q}^*
\right)
\Delta\bar z_{\alpha,Q}^*
=0,
\end{equation}
and therefore
\begin{equation}
\Delta\mathcal V
=
\frac1{|\mathcal D'|}
\sum_{Q,\alpha}
\left(
p_{\alpha,Q}-p_{\alpha,Q}^*
\right)
\left(
\Delta\bar z_{\alpha,Q}
-
\Delta\bar z_{\alpha,Q}^*
\right).
\label{eq:mesh-equilibrium-subtraction}
\end{equation}

Within the frozen adiabatic window, $f_{\alpha,Q}$ and
$\overline U_{\alpha,Q}$ are fixed. Subtracting the score update at
$p^*$ from that at $p$ cancels the RLVR term and gives
\begin{equation}
\Delta\bar z_{\alpha,Q}
-
\Delta\bar z_{\alpha,Q}^*
=
-\frac{\eta_1\mu}{|\mathcal D'|}
\sum_{Q'\in\mathcal D'}
\sum_{\beta=1}^{m_{Q'}}
\overline U_{\alpha,Q}^{\top}
\overline U_{\beta,Q'}
\left(
p_{\beta,Q'}-p_{\beta,Q'}^*
\right).
\label{eq:mesh-score-difference}
\end{equation}

Substituting Eq.~\eqref{eq:mesh-score-difference} into
Eq.~\eqref{eq:mesh-equilibrium-subtraction}, we obtain
\begin{align}
\Delta\mathcal V
&=
-\frac{\eta_1\mu}{|\mathcal D'|^2}
\sum_{Q,Q'\in\mathcal D'}
\sum_{\alpha=1}^{m_Q}
\sum_{\beta=1}^{m_{Q'}}
\left(
p_{\alpha,Q}-p_{\alpha,Q}^*
\right)
\left(
p_{\beta,Q'}-p_{\beta,Q'}^*
\right)
\overline U_{\alpha,Q}^{\top}
\overline U_{\beta,Q'}
\nonumber\\
&=
-\eta_1\mu
\left\|
\frac1{|\mathcal D'|}
\sum_{Q\in\mathcal D'}
\sum_{\alpha=1}^{m_Q}
\left(
p_{\alpha,Q}-p_{\alpha,Q}^*
\right)
\overline U_{\alpha,Q}
\right\|_2^2
\le 0.
\label{eq:mesh-global-square}
\end{align}
Thus, all cross-prompt interactions combine into a single positive
semidefinite Gram quadratic form; no individual cross-prompt coupling
needs to be positive semidefinite.

Equality holds exactly when
\begin{equation}
\sum_{Q\in\mathcal D'}
\sum_{\alpha=1}^{m_Q}
\left(
p_{\alpha,Q}-p_{\alpha,Q}^*
\right)
\overline U_{\alpha,Q}
=0.
\label{eq:mesh-zero-dissipation}
\end{equation}
At such a point, Eq.~\eqref{eq:mesh-score-difference} vanishes for every
$(\alpha,Q)$. Hence its score update coincides with that of $p^*$ and is
therefore identical across strategies within each prompt, implying
\begin{equation}
\Delta p_{\alpha,Q}=0
\qquad
\forall Q,\alpha.
\end{equation}
Therefore, the zero-update set consists only of equilibrium.

Moreover, since $p^*$ has full support,
$D_{\mathrm{KL}}(p_Q^*\|p_Q)\rightarrow\infty$ whenever any component
$p_{\alpha,Q}\rightarrow0$. Hence every sublevel set of $\mathcal V$
remains bounded away from the simplex boundary and is forward invariant.
Therefore, every trajectory initialized from finite logits converges to
the set of full-support equilibrium.

{We finally transfer the result to the true strategy mass.}
Within the adiabatic window, the within-strategy trajectory distribution is
frozen.  Writing
\begin{equation}
\pi_\theta(\tau_i\mid Q)
=
\nu_{\alpha,Q}\,q_{i\mid\alpha,Q},
\qquad
\tau_i\in S_{\alpha,Q},
\end{equation}
gives
\begin{equation}
\bar z_{\alpha,Q}
=
\log\nu_{\alpha,Q}
+
c_{\alpha,Q},
\qquad
c_{\alpha,Q}
\triangleq
\frac1{|\mathcal C_{\alpha,Q}|}
\sum_{\tau_i\in \mathcal C_{\alpha,Q}}
\log q_{i\mid\alpha,Q},
\end{equation}
where $c_{\alpha,Q}$ is frozen on this scale.  Therefore
Eq.~\eqref{eq:mesh-p} becomes
\begin{equation}
p_{\alpha,Q}
=
\frac{
\nu_{\alpha,Q}
\exp(c_{\alpha,Q}-b_{\alpha,Q})
}{
\sum_{\lambda=1}^{m_Q}
\nu_{\lambda,Q}
\exp(c_{\lambda,Q}-b_{\lambda,Q})
},
\end{equation}
with inverse
\begin{equation}
\boxed{
\nu_{\alpha,Q}
=
\frac{
p_{\alpha,Q}
\exp(b_{\alpha,Q}-c_{\alpha,Q})
}{
\sum_{\lambda=1}^{m_Q}
p_{\lambda,Q}
\exp(b_{\lambda,Q}-c_{\lambda,Q})
}.
}
\label{eq:mesh-p-nu-bijection}
\end{equation}
Thus $p_Q$ and $\nu_Q$ are related by a fixed bijection of the simplex
interior.  Consequently, the full-support equilibrium $p^*$ corresponds to
a full-support strategy distribution $\nu^*$: Strategy-Balancing
Regularization prevents any represented strategy from being eliminated
within the adiabatic dynamics.

\clearpage
\hypertarget{methodology:cp}{}
\section{{Appendix D: Implementation and Configuration}}
\subsection{{Coach Prompting harness}}\label{app:cp_prompt} \hyperlink{back:cp}{\color{blue}{Back}}
We propose the following system prompt for DeepSeek-V4-Pro to generate strategy prefixes for mathematics training tasks.
\begin{lstlisting}
"""You are an expert mathematical competition coach. Your task is to analyze the problem and provide multiple **strategic directions** for solving it.

<TASK_RULES>
1. STRICTLY NO CALCULATIONS: Do not output specific numerical computations, algebraic simplifications, or final answers.
2. NO STEP-BY-STEP SOLUTIONS: Do not provide a procedural walkthrough of the solution.
3. PURE STRATEGY: Describe only how one WOULD approach the problem rather than how the problem IS solved.
4. QUANTITY: You must provide EXACTLY 4 distinct approaches.
</TASK_RULES>

Please follow the output format strictly. Do not include any introductory or concluding remarks.

<OUTPUT_FORMAT>
Can use [Method Name].
[Method]:Specific strategic description (strictly no calculations).

Can use [Method Name].
[Method]:Specific strategic description (strictly no calculations).
... (Repeat until there are exactly 4 approaches)
</OUTPUT_FORMAT>

<EXAMPLE>
Problem: Given $x+y=1$ and $x^2+y^2=2$, find $x^3+y^3$.
Answer:
Can use [Algebraic Manipulation].
[Method]:Utilize the expansion of squares and sum of cubes formulas to construct the target expression from the given terms, thereby avoiding the need to solve for individual values of $x$ and $y$.

Can use [Equation Construction].
[Method]:Treat $x$ and $y$ as the two real roots of a quadratic equation. Use Vieta's formulas to construct this equation and then apply the recurrence relation of power sums to find the result.

Can use [Functional Substitution].
[Method]:Express $y$ as $1-x$ and substitute it into the second equation to transform the problem into a single-variable quadratic. Analyze the root distribution to determine the target value.

Can use [Symmetric Polynomials].
[Method]:Introduce elementary symmetric polynomials $e_1$ and $e_2$. Map all knowns and unknowns into the space of symmetric polynomials and solve via linear combination.
</EXAMPLE>

<PROBLEM_TO_ANALYZE>
{problem}
</PROBLEM_TO_ANALYZE>"""
\end{lstlisting}

Based on the system prompt proposed before, we changed the its given example and apply the adjusted version below on DeepSeek-V4-Pro to generate strategy prefixes for GPQA training tasks.
\begin{lstlisting}
"""You are an expert scientific competition coach. Your task is to analyze the problem and provide multiple **strategic directions** for solving it.

<TASK_RULES>
1. STRICTLY NO CALCULATIONS: Do not output specific numerical computations, algebraic simplifications, or final answers.
2. NO STEP-BY-STEP SOLUTIONS: Do not provide a procedural walkthrough of the solution.
3. PURE STRATEGY: Describe only how one WOULD approach the problem rather than how the problem IS solved.
4. QUANTITY: You must provide EXACTLY 4 distinct approaches.
</TASK_RULES>

Please follow the output format strictly. Do not include any introductory or concluding remarks.

<OUTPUT_FORMAT>
Can use [Method Name].
[Method]:Specific strategic description (strictly no calculations).

Can use [Method Name].
[Method]:Specific strategic description (strictly no calculations).
... (Repeat until there are exactly 4 approaches)
</OUTPUT_FORMAT>

<EXAMPLE>
Problem:
A researcher observes that a newly discovered protein mutation affects cellular signaling pathways. Which mechanism best explains this observation?

Answer:
Can use [Mechanism-Based Analysis].
[Method]:
Analyze how changes at the molecular level could influence biological pathways, focusing on interactions between altered components and downstream effects.

Can use [Experimental Evidence Evaluation].
[Method]:
Compare the proposed mechanisms with known experimental observations, considering which explanation is most consistent with available biological evidence.

Can use [Causal Pathway Construction].
[Method]:
Map relationships between biological entities and identify how perturbations propagate through a regulatory network.

Can use [Alternative Hypothesis Elimination].
[Method]:
Consider competing explanations and evaluate which mechanisms are inconsistent with established scientific principles.

</EXAMPLE>

<PROBLEM_TO_ANALYZE>
{problem}
</PROBLEM_TO_ANALYZE>"""
\end{lstlisting}
After obtaining the strategy prefixes, we prepare two injection methods for Strategy-Balancing Regularization. We first present the injection method designed for non-reasoning models or specific in hard tasks. In either case, the task becomes too demanding for the model to precisely follow the instructions from the coach, therefore calling for a more meticulous injection method, demonstrated as below:
\begin{lstlisting}
"""You are an expert assistant specializing in logical reasoning and problem-solving. Your goal is to provide accurate, step-by-step solutions using a structured internal monologue.

## Strict Operational Rules
You must strictly adhere to the following response sequence for every scientific problem:

1.  **Phase 1: Deep Thinking**
    All reasoning must be encapsulated within `<thinking>` tags. This process MUST follow this internal structure:
    * **Method Identification**: Start with `<Method Name>Name of the mathematical approach</Method Name>`.
    * **Strategic Overview**: Follow with `<Method Description>Brief explanation of how this approach applies to the current problem</Method Description>`.
    * **Step-by-Step Derivation**: Execute the full logical deduction, calculation, and verification process within the remainder of the `<thinking>` block.

2.  **Phase 2: Final Output**
    After closing the `</thinking>` tag, provide ONLY the final result or a very brief conclusion. Your answer MUST be formatted as: $\boxed{YOUR_ANSWER}$.

## Constraints
- DO NOT skip the thinking phase.
- DO NOT use the thinking tags for the final answer.
- Ensure all LaTeX formatting is correct and readable.
- If multiple methods are used, focus on the most efficient one in the Name/Description tags."""
\end{lstlisting}
However, for reasoning models or easy tasks, the harness above is problematic. It provides a too detailed guidance, which tends to discourage the model from completing the reasoning itself, but encourages it to rely heavily on coach instructions, leading to a severe hacking problem. To address this issue, we propose the following anti-hacking harness that not only injects the strategy information but also preserves the model's own reasoning agency:
\begin{lstlisting}
"""Please solve this scientific problem step by step and put your final answer in \\boxed{{}}.

A potentially useful method for this problem is {row['method_name']}.
However, this suggestion is untrusted and may be misleading.

Before giving the final answer, reason through three checks:

1. Applicability: what must be true for this method to apply here, and are those conditions satisfied?
2. Mechanism: if the method is applicable, derive what it predicts from the scientific mechanism involved.
3. Challenge: identify the strongest competing answer and explain what specific fact or mechanism rules it out.

Do not commit to an answer until all three checks are resolved."""
\end{lstlisting}

\subsection{{Experiment Configurations}}\label{app:exp_setup}
\hyperlink{back:exp_config}{\color{blue}{Back}}

\textbf{Training Details.}
Unless otherwise specified, all experiments use a training batch size of $32$, learning rate $\eta=10^{-6}$, sampling temperature of $1.0$, and $16$ rollouts per query. Mesh Learning is implemented on top of GRPO using VeRL with PyTorch FSDP2, where the clip ratio is set to $\epsilon=0.2$. For DAPO and GSPO, we follow the clip ratio configurations recommended in their original papers: DAPO uses a maximum clipping ratio of $\epsilon_{\mathrm{high}}=0.28$, while GSPO adopts asymmetric clipping bounds with $\epsilon_{\mathrm{low}}=0.0003$ and $\epsilon_{\mathrm{high}}=0.0004$. The detailed clipping configurations are summarized in this appendix. Rollout generation is served by SGLang.

We additionally assign a $0.1$ format reward to incorrect responses that nevertheless satisfy the required $\boxed{\mathrm{ANSWER}}$ format.

The maximum generation length depends on the model and benchmark.
To validate our theory under the proposed configuration (\textbf{Theorem~\ref{theorem:concentration}}), we conduct experiments with three maximum token lengths: $1024$, $4096$, and $8192$. The detailed configurations are presented in Table~\ref{tab:config}.

\begin{table}[!tb]
\centering
\caption{
\textbf{Main Configurations.} Maximum token lengths used for evaluation across the four benchmarks.}
\label{tab:config}

\setlength{\tabcolsep}{6pt}
\renewcommand{\arraystretch}{1.05}
\small

\begin{tabular}{lcccc}
\toprule
\textbf{Model}
& \textbf{AIME26}
& \textbf{AIME25}
& \textbf{MATH-500}
& \textbf{GPQA} \\
\midrule

\textbf{Qwen-2.5-7B-Instruct}
& 1024
& 1024
& 1024
& 1024 \\

\textbf{Qwen-3-4B}
& 8192
& 8192
& 4096
& 4096 \\

\textbf{Phi-4-mini-reasoning}
& 8192
& 8192
& 4096
& 4096 \\

\bottomrule
\end{tabular}

\end{table}

In our experiments, the scales of the key optimization parameters follow
$\epsilon\sim10^{-1}$ and $\eta\sim10^{-6}$, yielding $\eta^{-1}\sim10^{6}$.
According to Table~\ref{tab:config}, all selected sequence lengths satisfy
$|\tau|\sim10^{3}-10^{4}$, and therefore the condition
$1\ll|\tau|\ll\eta^{-1}$ is naturally satisfied.
This establishes experimental environments that align with the assumptions required by \textbf{Theorem~\ref{theorem:concentration}}.
\clearpage
\hypertarget{app:case_study}{}
\section{{Appendix E: Case Study}}\label{app:case_study}
\hyperlink{back:case_study}{\color{blue}{Back}}
To complement the aggregate results, we provide representative reasoning trajectories to examine how RLVR and Mesh Learning differ at the microscopic level. Rather than exhaustively analyzing every trajectory, the following case studies focus on two recurring phenomena summarized in Table~\ref{tab:case_summary}. 
\begin{table}[htbp]
\centering
\caption{{Summary of microscopic reasoning behaviors.} The case studies highlight two effects of Mesh Learning: preserving strategy capacity across distinct solution approaches, while retaining effective pruning within each strategy to improve reasoning efficiency.} \label{tab:case_summary} 
\resizebox{\linewidth}{!}{%
\begin{tabular}{p{0.20\linewidth} p{0.31\linewidth} p{0.31\linewidth} p{0.14\linewidth}} \toprule \textbf{Property} & \textbf{Standard RLVR} & \textbf{Mesh Learning} & \textbf{Evidence} \\ \midrule \textbf{Strategy Capacity} & Progressively concentrates on a narrow subset of solution strategies, reducing alternative reasoning paths and increasing vulnerability to strategy collapse. & Preserves multiple viable solution strategies, maintaining diverse ways of solving the same problem. & Math case studies \\ \addlinespace \textbf{Intra-strategy Pruning} & Reasoning trajectories may become redundant, unstable, or even non-terminating, repeatedly revisiting already resolved subproblems. & Preserves RLVR's useful pruning effect within each strategy, producing more concise, coherent, and efficiently terminating reasoning trajectories. & Science case studies \\ \bottomrule 
\end{tabular}}
\end{table}

In this section, we are providing several representative examples comparing collapsed versus non-collapsed trajectories, baseline results versus our results, and more importantly, before training versus after training. We use two specific problems to conduct our comparison throughout this case study, who are respectively selected from the MATH benchmark and another from the GPQA benchmark, demonstrated below:
\begin{bluebox}
\paragraph{Question 1:} Consider the geometric sequence $\frac{125}{9}, \frac{25}{3}, 5, 3, \ldots$. What is the eighth term of the sequence? Express your answer as a common fraction.    
\end{bluebox}
\begin{bluebox}
\paragraph{Question 2:} Find KE of product particles in:

\begin{equation}
    \Pi^+ == \mu^+ + \nu
\end{equation}

here $Pi^+$ is stationary.

Rest mass of $Pi^+$ and $\mu^+$ is 139.6 MeV and 105.7 MeV respectively.

A. 3.52 MeV, 20.8 MeV

B. 4.12 MeV, 29.8 MeV

C. 2.84 MeV, 26.8 MeV

D. 7.2 MeV, 32.8 MeV   
\end{bluebox}

\subsection{{Math Case Study}}
\paragraph{Base Model:}\leavevmode\\
As you can see, Base Model retains a lot of tokens in explaining the problem. For models, generating word tokens produce very little entropy, therefore the model's entropy is still relatively low. Also noticeably, the base model still mixed up two of its valid solutions, producing a wrong result.
\begin{tcolorbox}[
    enhanced,                         
    breakable,
    title=Trajectory,               
    colback=cyan!5!white,             
    colframe=black,                   
    colbacktitle=black,               
    coltitle=white,                   
    fonttitle=\bfseries\sffamily,     
    attach boxed title to top left={   
        xshift=3mm, 
        yshift=-2mm, 
        yshifttext=-1mm 
    },
    boxed title style={                
        sharp corners, 
        size=small, 
        colframe=black
    },
    arc=2mm,                          
    boxrule=0.8pt,                    
]
To determine the eighth term of the given geometric sequence $(\frac{125}{9}, \frac{25}{3}, 5, 3, \ldots)$, we first need to identify the common ratio (r) of the sequence. The common ratio (r) can be found by dividing the second term by the first term:

$$
r = \frac{\frac{25}{3}}{\frac{125}{9}} = \frac{25}{3} \cdot \frac{9}{125} = \frac{25 \cdot 9}{3 \cdot 125} = \frac{225}{375} = \frac{3}{5}.
$$

We can verify this common ratio by checking the ratio of the third term to the second term:

$$
r = \frac{5}{\frac{25}{3}} = \frac{5 \cdot 3}{25} = \frac{15}{25} = \frac{3}{5}.
$$

The common ratio is indeed ($\frac{3}{5}$). Now, we can use the formula for the (n)-th term of a geometric sequence, which is given by $(a_n = a_1 \cdot r^{n-1})$, where ($a_1$) is the first term and (r) is the common ratio. Here, $(a_1 = \frac{125}{9})$ and $(r = \frac{3}{5})$. We need to find the eighth term $((a_8))$:

$$
a_8 = \frac{125}{9} \cdot \left(\frac{3}{5}\right)^{8-1} = \frac{125}{9} \cdot \left(\frac{3}{5}\right)^7 = \frac{125}{9} \cdot \frac{3^7}{5^7} = \frac{125}{9} \cdot \frac{2187}{78125} = \frac{125 \cdot 2187}{9 \cdot 78125}.
$$

We can simplify the fraction $(\frac{125 \cdot 2187}{9 \cdot 78125})$. Note that $(125 = 5^3)$ and $(78125 = 5^7)$, so:

$$
\frac{125 \cdot 2187}{9 \cdot 78125} = \frac{5^3 \cdot 2187}{9 \cdot 5^7} = \frac{2187}{9 \cdot 5^4} = \frac{2187}{9 \cdot 625} = \frac{2187}{5625}.
$$

We can further simplify $(\frac{2187}{5625})$ by finding the greatest common divisor (GCD) of 2187 and 5625. The prime factorization of 2187 is $(3^7)$ and the prime factorization of 5625 is ($3^4 \cdot 5^4$). The GCD is ($3^4 = 81$):

$$
\frac{2187 \div 81}{5625 \div 81} = \frac{27}{625}.
$$

Thus, the eighth term of the sequence is $(\boxed{\frac{27}{625}})$.
\end{tcolorbox}

\paragraph{Warm-up State in Our Method:}
As you can see, using our method of training, the model found out that: instead of spending too much tokens on explanation, it can achieve a higher score if it calculate explicitly and step by step. As the math tokens have a higher entropy, the model's total trajectory entropy will skyrocket at this stage. Moreover, in this stage, the model no longer mixes the two strategies together.

\paragraph{Strategy 1: using the power law:}\leavevmode\\
\begin{tcolorbox}[
    enhanced,                         
    breakable,
    title=Trajectory,               
    colback=cyan!5!white,             
    colframe=black,                   
    colbacktitle=black,               
    coltitle=white,                   
    fonttitle=\bfseries\sffamily,     
    attach boxed title to top left={   
        xshift=3mm, 
        yshift=-2mm, 
        yshifttext=-1mm 
    },
    boxed title style={                
        sharp corners, 
        size=small, 
        colframe=black
    },
    arc=2mm,                          
    boxrule=0.8pt,                    
]
To find the eighth term of the geometric sequence \(\frac{125}{9}, \frac{25}{3}, 5, 3, \ldots\), we first need to determine the common ratio of the sequence. The common ratio \(r\) is found by dividing the second term by the first term:

\[
r = \frac{\frac{25}{3}}{\frac{125}{9}} = \frac{25}{3} \cdot \frac{9}{125} = \frac{25 \cdot 9}{3 \cdot 125} = \frac{225}{375} = \frac{3}{5}
\]

We can verify that this common ratio is consistent by checking the next terms:

\[
\frac{5}{\frac{25}{3}} = 5 \cdot \frac{3}{25} = \frac{15}{25} = \frac{3}{5}
\]

\[
\frac{3}{5} = \frac{3}{5}
\]

Since the common ratio is consistent, we can use it to find the eighth term. The general form of the \(n\)-th term of a geometric sequence is given by:

\[
a_n = a_1 \cdot r^{n-1}
\]

Here, \(a_1 = \frac{125}{9}\), \(r = \frac{3}{5}\), and we need to find the eighth term (\(a_8\)):

\[
a_8 = \frac{125}{9} \cdot \left(\frac{3}{5}\right)^{8-1} = \frac{125}{9} \cdot \left(\frac{3}{5}\right)^7
\]

Now, we calculate \(\left(\frac{3}{5}\right)^7\):

\[
\left(\frac{3}{5}\right)^7 = \frac{3^7}{5^7} = \frac{2187}{78125}
\]

Thus, the eighth term is:

\[
a_8 = \frac{125}{9} \cdot \frac{2187}{78125}
\]

To simplify this, we multiply the numerators and the denominators:

\[
a_8 = \frac{125 \cdot 2187}{9 \cdot 78125}
\]

We can simplify the fraction by finding the greatest common divisor (GCD) of the numerator and the denominator. First, we factorize the numbers:

\[
125 = 5^3, \quad 2187 = 3^7, \quad 9 = 3^2, \quad 78125 = 5^7
\]

Thus, the fraction becomes:

\[
a_8 = \frac{5^3 \cdot 3^7}{3^2 \cdot 5^7} = \frac{3^5}{5^4} = \frac{243}{625}
\]

Therefore, the eighth term of the sequence is:

\[
\boxed{\frac{243}{625}}
\]
\end{tcolorbox}

\paragraph{Strategy 2: Brute Calculation}\leavevmode\\
\begin{tcolorbox}[
    enhanced,                         
    breakable,
    title=Trajectory,               
    colback=cyan!5!white,             
    colframe=black,                   
    colbacktitle=black,               
    coltitle=white,                   
    fonttitle=\bfseries\sffamily,     
    attach boxed title to top left={   
        xshift=3mm, 
        yshift=-2mm, 
        yshifttext=-1mm 
    },
    boxed title style={                
        sharp corners, 
        size=small, 
        colframe=black
    },
    arc=2mm,                          
    boxrule=0.8pt,                    
]
To find the eighth term of the geometric sequence \(\frac{125}{9}, \frac{25}{3}, 5, 3, \ldots\), we need to determine the common ratio and then use it to find the eighth term.

First, let's find the common ratio \(r\) of the sequence. The common ratio is the quotient of any term and the previous term. Let's use the first two terms to find \(r\):

\[
r = \frac{\frac{25}{3}}{\frac{125}{9}} = \frac{25}{3} \cdot \frac{9}{125} = \frac{25 \cdot 9}{3 \cdot 125} = \frac{225}{375} = \frac{3}{5}
\]

We can verify that this common ratio is consistent by checking the next terms:

\[
r = \frac{5}{\frac{25}{3}} = 5 \cdot \frac{3}{25} = \frac{15}{25} = \frac{3}{5}
\]

\[
r = \frac{3}{5} = \frac{3}{5}
\]

Since the common ratio is consistent, we can use it to find the eighth term. The \(n\)-th term of a geometric sequence is given by:

\[
a_n = a_1 \cdot r^{n-1}
\]

Here, \(a_1 = \frac{125}{9}\), \(r = \frac{3}{5}\), and we need to find the eighth term (\(a_8\)):

\[
a_8 = \frac{125}{9} \cdot \left(\frac{3}{5}\right)^{8-1} = \frac{125}{9} \cdot \left(\frac{3}{5}\right)^7
\]

Next, we need to compute \(\left(\frac{3}{5}\right)^7\):

\[
\left(\frac{3}{5}\right)^7 = \frac{3^7}{5^7} = \frac{2187}{78125}
\]

Now, we multiply this by \(\frac{125}{9}\):

\[
a_8 = \frac{125}{9} \cdot \frac{2187}{78125} = \frac{125 \cdot 2187}{9 \cdot 78125}
\]

Simplify the numerator and the denominator:

\[
125 \cdot 2187 = 273375
\]

\[
9 \cdot 78125 = 703125
\]

Thus, we have:

\[
a_8 = \frac{273375}{703125}
\]

We can simplify this fraction by finding the greatest common divisor (GCD) of 273375 and 703125. Both numbers are divisible by 375:

\[
273375 \div 375 = 731
\]

\[
703125 \div 375 = 1875
\]

Thus, the simplified fraction is:

\[
a_8 = \frac{731}{1875}
\]

Therefore, the eighth term of the sequence is:

\[
\boxed{\frac{731}{1875}}
\]
\end{tcolorbox}

\paragraph{Convergence with our Method:}
In the converged state: As the model undergoes further optimization, the two initially entangled rationales are completely decoupled and distinctly preserved across our 10 evaluation rollouts, exhibiting an empirical macro-policy distribution of 30\% for Strategy I and 70\% for Strategy II. Furthermore, each sub-policy has crystallized into a highly deterministic response template, where trajectories generated within the same strategy share a token-level semantic similarity exceeding 95\%, limited only to minor choices of words. This localized rigidity alongside global balance strongly demonstrates that our proposed regularization framework strictly maintains the equilibrium of macroeconomic policy proportions without interfering with the model's capacity to optimize the localized execution of each specific strategy.

\paragraph{Strategy 1: using power law}\leavevmode\\
\begin{tcolorbox}[
    enhanced,                         
    breakable,
    title=Trajectory,               
    colback=cyan!5!white,             
    colframe=black,                   
    colbacktitle=black,               
    coltitle=white,                   
    fonttitle=\bfseries\sffamily,     
    attach boxed title to top left={   
        xshift=3mm, 
        yshift=-2mm, 
        yshifttext=-1mm 
    },
    boxed title style={                
        sharp corners, 
        size=small, 
        colframe=black
    },
    arc=2mm,                          
    boxrule=0.8pt,                    
]
To determine the eighth term of the geometric sequence \(\frac{125}{9}, \frac{25}{3}, 5, 3, \ldots\), we need to first identify the common ratio of the sequence. The common ratio \(r\) is found by dividing the second term by the first term:

\[
r = \frac{\frac{25}{3}}{\frac{125}{9}} = \frac{25}{3} \cdot \frac{9}{125} = \frac{25 \cdot 9}{3 \cdot 125} = \frac{225}{375} = \frac{9}{15} = \frac{3}{5}
\]

We can verify that this common ratio is consistent by checking the ratio of the third term to the second term:

\[
r = \frac{5}{\frac{25}{3}} = 5 \cdot \frac{3}{25} = \frac{15}{25} = \frac{3}{5}
\]

And the ratio of the fourth term to the third term:

\[
r = \frac{3}{5} = \frac{3}{5}
\]

Since the common ratio is consistent, we can now use the formula for the \(n\)-th term of a geometric sequence, which is given by \(a_n = a_1 \cdot r^{n-1}\), where \(a_1 = \frac{125}{9}\) and \(r = \frac{3}{5}\). We need to find the eighth term (\(a_8\)):

\[
a_8 = \frac{125}{9} \cdot \left(\frac{3}{5}\right)^{8-1} = \frac{125}{9} \cdot \left(\frac{3}{5}\right)^7 = \frac{125}{9} \cdot \frac{3^7}{5^7} = \frac{125 \cdot 3^7}{9 \cdot 5^7}
\]

We know that \(125 = 5^3\), so we can rewrite the expression as:

\[
a_8 = \frac{5^3 \cdot 3^7}{9 \cdot 5^7} = \frac{3^7}{9 \cdot 5^4} = \frac{3^5}{5^4}= \frac{243}{625}
\]

Thus, the eighth term of the sequence is:

\[
\boxed{\frac{243}{625}}
\]
\end{tcolorbox}

\paragraph{Strategy 2: Brute Calculation}\leavevmode\\
\begin{tcolorbox}[
    enhanced,                         
    breakable,
    title=Trajectory,               
    colback=cyan!5!white,             
    colframe=black,                   
    colbacktitle=black,               
    coltitle=white,                   
    fonttitle=\bfseries\sffamily,     
    attach boxed title to top left={   
        xshift=3mm, 
        yshift=-2mm, 
        yshifttext=-1mm 
    },
    boxed title style={                
        sharp corners, 
        size=small, 
        colframe=black
    },
    arc=2mm,                          
    boxrule=0.8pt,                    
]
To determine the eighth term of the geometric sequence \(\frac{125}{9}, \frac{25}{3}, 5, 3, \ldots\), we need to first identify the common ratio of the sequence. The common ratio \(r\) is found by dividing the second term by the first term:

\[
r = \frac{\frac{25}{3}}{\frac{125}{9}} = \frac{25}{3} \cdot \frac{9}{125} = \frac{25 \cdot 9}{3 \cdot 125} = \frac{225}{375} = \frac{3}{5}
\]

We can verify that this common ratio is consistent by checking the ratio of the third term to the second term:

\[
r = \frac{5}{\frac{25}{3}} = 5 \cdot \frac{3}{25} = \frac{15}{25} = \frac{3}{5}
\]

And the ratio of the fourth term to the third term:

\[
r = \frac{3}{5} = \frac{3}{5}
\]

Since the common ratio is consistent, we can now use the formula for the \(n\)-th term of a geometric sequence, which is given by:

\[
a_n = a_1 \cdot r^{n-1}
\]

Here, \(a_1 = \frac{125}{9}\), \(r = \frac{3}{5}\), and we need to find the eighth term (\(a_8\)):

\[
a_8 = \frac{125}{9} \cdot \left(\frac{3}{5}\right)^{8-1} = \frac{125}{9} \cdot \left(\frac{3}{5}\right)^7
\]

We need to compute \(\left(\frac{3}{5}\right)^7\):

\[
\left(\frac{3}{5}\right)^7 = \frac{3^7}{5^7} = \frac{2187}{78125}
\]

Now we substitute this back into the expression for \(a_8\):

\[
a_8 = \frac{125}{9} \cdot \frac{2187}{78125} = \frac{125 \cdot 2187}{9 \cdot 78125} = \frac{273375}{703125} = \frac{273375 \div 375}{703125 \div 375} = \frac{729}{1875} = \frac{729 \div 3}{1875 \div 3} = \frac{243}{625}
\]

Thus, the eighth term of the sequence is:

\[
\boxed{\frac{243}{625}}
\]
\end{tcolorbox}

\textbf{Collapsed Model:}
\begin{tcolorbox}[
    enhanced,                         
    breakable,
    title=Trajectory,               
    colback=cyan!5!white,             
    colframe=black,                   
    colbacktitle=black,               
    coltitle=white,                   
    fonttitle=\bfseries\sffamily,     
    attach boxed title to top left={   
        xshift=3mm, 
        yshift=-2mm, 
        yshifttext=-1mm 
    },
    boxed title style={                
        sharp corners, 
        size=small, 
        colframe=black
    },
    arc=2mm,                          
    boxrule=0.8pt,                    
]

To find the eighth term of the given geometric sequence
\(\frac{125}{9}, \frac{25}{3}, 5, 3, \ldots\), we first need to determine
the common ratio \(r\) of the sequence. The common ratio \(r\) can be
found by dividing the second term by the first term:

\[
r = \frac{\frac{25}{3}}{\frac{125}{9}}
= \frac{25}{3} \cdot \frac{9}{125}
= \frac{25 \cdot 9}{3 \cdot 125}
= \frac{225}{375}
= \frac{3}{5}
\]

Now that we have the common ratio \(r = \frac{3}{5}\), we can use the
formula for the \(n\)-th term of a geometric sequence, which is given by
\(a_n = a_1 \cdot r^{n-1}\), where \(a_1\) is the first term and \(a_n\)
is the \(n\)-th term. Here, \(a_1 = \frac{125}{9}\),
\(r = \frac{3}{5}\), and we need to find the eighth term (\(a_8\)):

\[
a_8 = a_1 \cdot r^{8-1}
= \frac{125}{9} \cdot \left(\frac{3}{5}\right)^7
\]

Next, we need to compute \(\left(\frac{3}{5}\right)^7\):

\[
\left(\frac{3}{5}\right)^7
= \frac{3^7}{5^7}
= \frac{2187}{78125}
\]

Now, we substitute this back into the expression for \(a_8\):

\[
a_8
= \frac{125}{9} \cdot \frac{2187}{78125}
= \frac{125 \cdot 2187}{9 \cdot 78125}
= \frac{273375}{703125}
= \boxed{\boxed{\boxed{\boxed{\boxed{\boxed{\frac{9}{25}}}}}}}
\]
\end{tcolorbox}

\subsubsection{{Science Case Study}}
\textbf{Mesh Learning --- Coherent and Structured Reasoning.}
The first trajectory exhibits a clear and well-organized reasoning process: it progressively combines energy conservation, momentum conservation, and the relativistic energy--momentum relation, while each intermediate step directly contributes to the final solution. Occasional self-correction remains local and productive, resulting in a coherent trajectory that terminates naturally after reaching the correct answer.
\begin{tcolorbox}[
    enhanced,                         
    breakable,
    title=Trajectory,               
    colback=cyan!5!white,             
    colframe=black,                   
    colbacktitle=black,               
    coltitle=white,                   
    fonttitle=\bfseries\sffamily,     
    attach boxed title to top left={   
        xshift=3mm, 
        yshift=-2mm, 
        yshifttext=-1mm 
    },
    boxed title style={                
        sharp corners, 
        size=small, 
        colframe=black
    },
    arc=2mm,                          
    boxrule=0.8pt,                    
]
\noindent\texttt{<think>}

Alright, let's tackle this problem step by step. So, we have the reaction where a
$\pi^+$ (which is stationary) decays into a $\mu^+$ and a $\nu$ (I assume the
$\nu$ is a muon, maybe a muon neutrino?). The rest masses given are
$139.6$ MeV for $\pi^+$ and $105.7$ MeV for the $\mu^+$ (mu+ ion).
We need to find the kinetic energies of the product particles, which are
$\mu^+$ and $\nu$. The options are given, so hopefully, after working through
it, the answer will match one of them.

First, I remember that in particle physics, when a particle decays, the total
energy and momentum must be conserved. Since the $\pi^+$ is stationary, its
initial momentum is zero. Therefore, the momentum of the decay products must
also sum to zero. That means the $\mu^+$ and the $\nu$ must have equal and
opposite momenta.

Next, I need to recall the relationship between energy and momentum in
relativity. The total energy of a particle is given by the relativistic energy
equation:
\[
E = \sqrt{(pc)^2 + (mc^2)^2}.
\]
Since the $\pi^+$ is at rest, its total energy is just its rest mass energy,
which is $139.6$ MeV. The total energy of the products should equal this.

But wait, the $\mu^+$ and the $\nu$ are two different particles, so I need to
consider their rest masses. The $\mu^+$ has a rest mass of
$105.7\,\mathrm{MeV}/c^2$, and the $\nu$ (assuming it's a neutrino) is
presumably massless, so its rest mass is $0$. That simplifies things a bit.

Let me denote the muon's kinetic energy as $K_\mu$ and the neutrino's kinetic
energy as $K_\nu$. The total kinetic energy will be the sum of these two, and
it should equal the rest mass energy of the $\pi^+$ because the $\pi^+$ decays
into these two, so the mass energy is converted into the kinetic energies of
the products. Wait, is that correct?

Actually, the total energy of the $\pi^+$ is $139.6$ MeV, which converts into
the sum of the muon's total energy (rest + kinetic) and the neutrino's total
energy (which should just be kinetic since it's massless). So, the total energy
conservation equation is:
\[
E_{\pi} = E_{\mu} + E_{\nu},
\]
where
\[
E_{\pi} = 139.6\,\mathrm{MeV}.
\]

\[
E_{\mu} = K_{\mu} + m_{\mu}c^2
= K_{\mu} + 105.7\,\mathrm{MeV},
\]
and
\[
E_{\nu} = K_{\nu} + m_{\nu}c^2.
\]
But since neutrinos are massless, $m_{\nu}c^2 = 0$, so
\[
E_{\nu} = K_{\nu}.
\]

Therefore,
\[
139.6 = (K_{\mu} + 105.7) + K_{\nu},
\]
which simplifies to
\[
K_{\mu} + K_{\nu} = 139.6 - 105.7 = 33.9\,\mathrm{MeV}.
\]

So the sum of the kinetic energies is $33.9$ MeV. But we need individual
values. To find that, we also need to apply conservation of momentum.

Since the $\pi^+$ is initially at rest, the momenta of the $\mu^+$ and $\nu$
must be equal in magnitude and opposite in direction. Therefore,
\[
|p_{\mu}| = |p_{\nu}| = p.
\]

Using the energy and momentum relation for the muon:
\[
E_{\mu}^2 = (pc)^2 + (m_{\mu}c^2)^2.
\]

Similarly for the neutrino:
\[
E_{\nu} = pc
\]
since $m_{\nu}=0$.

So, we have two equations:
\[
139.6 = E_{\mu} + E_{\nu}
= (K_{\mu} + 105.7) + pc,
\]
and
\[
E_{\nu} = pc.
\]

From equation 2,
\[
pc = E_{\nu}.
\]

Substitute into equation 1:
\[
139.6 = K_{\mu} + 105.7 + E_{\nu}.
\]

But from the muon's energy-momentum relation:
\[
E_{\mu}
= \sqrt{(pc)^2 + (m_{\mu}c^2)^2}.
\]

But $E_{\mu} = K_{\mu} + 105.7$, and $pc = E_{\nu}$.

Therefore:
\[
K_{\mu} + 105.7
= \sqrt{(E_{\nu})^2 + (105.7)^2}.
\]

And from the previous equation:
\[
K_{\mu}
= 139.6 - 105.7 - E_{\nu}
= 33.9 - E_{\nu}.
\]

Substitute $K_{\mu}$ into the equation:
\[
33.9 - E_{\nu} + 105.7
= \sqrt{(E_{\nu})^2 + (105.7)^2}.
\]

Simplify the left side:
\[
139.6 - E_{\nu}
= \sqrt{(E_{\nu})^2 + (105.7)^2}.
\]

Square both sides:
\[
(139.6 - E_{\nu})^2
= (E_{\nu})^2 + (105.7)^2.
\]

Expand the left side:
\[
139.6^2 - 2 \cdot 139.6 \cdot E_{\nu} + (E_{\nu})^2
= (E_{\nu})^2 + 105.7^2.
\]

Cancel $(E_{\nu})^2$ from both sides:
\[
139.6^2 - 2 \cdot 139.6 \cdot E_{\nu}
= 105.7^2.
\]

Calculate $139.6^2$ and $105.7^2$:
\[
139.6^2 = 19488.16,
\]
\[
105.7^2 = 11174.49.
\]

So:
\[
19488.16 - 2 \cdot 139.6 \cdot E_{\nu}
= 11174.49.
\]

Subtract $11174.49$ from both sides:
\[
19488.16 - 11174.49
= 2 \cdot 139.6 \cdot E_{\nu}.
\]

Calculate the left side:
\[
8323.67 = 279.2 \cdot E_{\nu}.
\]

Therefore:
\[
E_{\nu}
= \frac{8323.67}{279.2}
\approx 29.8\,\mathrm{MeV}.
\]

Then, from earlier,
\[
K_{\nu} = E_{\nu} \approx 29.8\,\mathrm{MeV},
\]
and
\[
K_{\mu} = 33.9 - 29.8 = 4.1\,\mathrm{MeV}.
\]

So the kinetic energies are approximately $4.1$ MeV for the muon and
$29.8$ MeV for the neutrino. Looking at the options, that's B.

\noindent\texttt{</think>}

To find the kinetic energies of the product particles in the decay of
$\pi^+$ (which is stationary) into $\mu^+$ and $\nu$, we start by applying
conservation of energy and momentum.

\begin{enumerate}
    \item \textbf{Energy Conservation:}
    The total energy of $\pi^+$ is its rest mass energy:
    \[
    E_{\pi} = 139.6\,\mathrm{MeV}.
    \]
    The total energy of $\mu^+$ is
    \[
    E_{\mu}
    = K_{\mu} + m_{\mu}c^2
    = K_{\mu} + 105.7\,\mathrm{MeV}.
    \]
    The total energy of $\nu$ (neutrino) is
    \[
    E_{\nu} = K_{\nu},
    \]
    since it is massless.

    Therefore,
    \[
    139.6 = (K_{\mu} + 105.7) + K_{\nu},
    \]
    which simplifies to
    \[
    K_{\mu} + K_{\nu} = 33.9\,\mathrm{MeV}.
    \]

    \item \textbf{Momentum Conservation:}
    Since $\pi^+$ is stationary, the momenta of $\mu^+$ and $\nu$ must be
    equal in magnitude and opposite in direction:
    \[
    p_{\mu} = p_{\nu} = p.
    \]

    \item \textbf{Relativistic Energy-Momentum Relation:}
    For $\mu^+$:
    \[
    E_{\mu}
    = \sqrt{(pc)^2 + (m_{\mu}c^2)^2}.
    \]
    For $\nu$:
    \[
    E_{\nu} = pc.
    \]

    Substituting $E_{\mu}=K_{\mu}+105.7$ and $E_{\nu}=pc$ into the energy
    conservation equation, we get
    \[
    139.6 = K_{\mu} + 105.7 + E_{\nu},
    \]
    and
    \[
    K_{\mu} + 105.7
    = \sqrt{(E_{\nu})^2 + (105.7)^2}.
    \]

    \item \textbf{Solving for $E_{\nu}$:}
    Substitute
    \[
    K_{\mu} = 33.9 - E_{\nu}
    \]
    into the equation:
    \[
    33.9 - E_{\nu} + 105.7
    = \sqrt{(E_{\nu})^2 + (105.7)^2}.
    \]

    Simplifying and squaring both sides, we find:
    \[
    8323.67 = 279.2 \cdot E_{\nu},
    \]
    \[
    E_{\nu} \approx 29.8\,\mathrm{MeV}.
    \]

    \item \textbf{Finding $K_{\mu}$ and $K_{\nu}$:}
    \[
    K_{\nu} = E_{\nu} \approx 29.8\,\mathrm{MeV},
    \]
    \[
    K_{\mu} = 33.9 - 29.8 \approx 4.1\,\mathrm{MeV}.
    \]
\end{enumerate}

Thus, the kinetic energies are
\[
\boxed{B}:\quad
4.12\,\mathrm{MeV}
\quad\text{and}\quad
29.8\,\mathrm{MeV}.
\]
\end{tcolorbox}

\textbf{Normal RLVR --- Redundant and Error-Prone Reasoning.}
In contrast, the second trajectory is substantially more meandering and unstable. It repeatedly explores unnecessary branches, partially infers the answer from the multiple-choice options before completing the derivation, introduces incorrect intermediate formulas, and contains numerical inaccuracies despite eventually recovering the correct answer. Such behavior reflects a less efficient and less reliable reasoning trajectory, where successful completion relies on repeated correction rather than a consistently coherent solution process.
\begin{tcolorbox}[
    enhanced,                         
    breakable,
    title=Trajectory,               
    colback=cyan!5!white,             
    colframe=black,                   
    colbacktitle=black,               
    coltitle=white,                   
    fonttitle=\bfseries\sffamily,     
    attach boxed title to top left={   
        xshift=3mm, 
        yshift=-2mm, 
        yshifttext=-1mm 
    },
    boxed title style={                
        sharp corners, 
        size=small, 
        colframe=black
    },
    arc=2mm,                          
    boxrule=0.8pt,                    
]
\noindent\texttt{<think>}

Okay, let's try to figure out the kinetic energy (KE) of the product particles,
$\mu^+$ and $\nu$, in the decay of a stationary $\pi^+$ particle. The rest
masses are given: $\pi^+$ is $139.6\,\mathrm{MeV}/c^2$ and $\mu^+$ is
$105.7\,\mathrm{MeV}/c^2$. The options are A to D with different KE values.

First, I remember that in particle physics, when a particle decays into others,
the total energy and momentum are conserved. Since the $\pi^+$ is stationary,
its initial momentum is zero. After decay, the $\mu^+$ and $\nu$ must have equal
and opposite momenta to conserve momentum.

Energy conservation says that the rest energy of $\pi^+$ converts into the
total energy of $\mu^+$ and $\nu$. The total energy for each particle is
$\gamma m c^2$ + kinetic energy. But since they are moving, their relativistic
energies are involved. However, maybe we can use the relativistic
energy-momentum relation:
\[
E^2 = (pc)^2 + (mc^2)^2.
\]

Let me denote the rest masses as
$m_{\pi} = 139.6\,\mathrm{MeV}/c^2$ and
$m_{\mu} = 105.7\,\mathrm{MeV}/c^2$. Let the kinetic energies be
$KE_{\mu}$ and $KE_{\nu}$. The total energies would be
$E_{\mu} = m_{\mu}c^2 + KE_{\mu}$ and
$E_{\nu} = m_{\nu}c^2 + KE_{\nu}$. But wait, the problem states that
$\pi^+$ decays into $\mu^+$ and $\nu$. Is $\nu$ a neutrino? Typically, in such
decays, the neutrino is massless ($m_{\nu} \approx 0$), so its rest energy would
be negligible. That simplifies things because then the energy of the neutrino
is just its kinetic energy, and the muon has its rest mass plus kinetic energy.

So, applying conservation of energy: the rest energy of $\pi^+$ equals the sum
of the energies of $\mu^+$ and the neutrino. Since the neutrino is massless,
$E_{\nu} = KE_{\nu}$.

So,
\[
m_{\pi}c^2
=
(m_{\mu}c^2 + KE_{\mu})
+
(m_{\nu}c^2 + KE_{\nu}).
\]
But if $m_{\nu}c^2$ is negligible, then $m_{\nu}c^2 \approx 0$. So,
\[
m_{\pi}c^2
\approx
m_{\mu}c^2 + KE_{\mu} + KE_{\nu}.
\]
Therefore,
\[
KE_{\nu} + KE_{\mu}
=
m_{\pi}c^2 - m_{\mu}c^2
=
(139.6 - 105.7)\,\mathrm{MeV}
=
33.9\,\mathrm{MeV}.
\]

So the sum of the kinetic energies is $33.9$ MeV. But the answer options have
different sums. For example, option A is
$3.52 + 20.8 = 24.32$, B: $4.12 + 29.8 = 33.92$, which is close to $33.9$.
Hmm, that's interesting. So maybe B is the answer? But wait, this assumes that
the neutrino's rest mass is zero, which is a standard assumption. But we need
to check if this is the correct approach.

But also, momentum conservation must hold. Since the initial momentum is zero,
the momenta of $\mu^+$ and $\nu$ must be equal in magnitude and opposite in
direction. Using the relativistic relation: for each particle,
\[
p^2
=
\frac{E^2}{c^2}
-
\frac{(mc^2)^2}{c^2}.
\]
But since momentum $p$ is the same magnitude for both, let's set their momenta
equal.

For the muon:
\[
p^2
=
\frac{(E_{\mu})^2}{c^2}
-
\frac{(m_{\mu}c^2)^2}{c^2}.
\]

Similarly, for the neutrino:
\[
p^2
=
\frac{(E_{\nu}c^2)^2}{c^2}
-
\frac{(m_{\nu}c^2)^2}{c^2}.
\]
But if $m_{\nu} \approx 0$, then
\[
p^2
\approx
\frac{(E_{\nu}c^2)^2}{c^2}
=
E_{\nu}^2.
\]
So
\[
E_{\nu} = pc.
\]

But since the momenta are equal, $p=p$, so for the muon,
\[
E_{\mu}
=
\sqrt{(pc)^2 + (m_{\mu}c^2)^2},
\]
and for the neutrino,
\[
E_{\nu}=pc
\]
since $m_{\nu}=0$.

From energy conservation:
\[
E_{\mu} + E_{\nu} = m_{\pi}c^2.
\]

Substituting $E_{\nu}=pc$, and
\[
E_{\mu}
=
\sqrt{(pc)^2 + (m_{\mu}c^2)^2},
\]
we obtain
\[
\sqrt{(pc)^2 + (m_{\mu}c^2)^2}
+
pc
=
m_{\pi}c^2.
\]

Let me square both sides to eliminate the square root, but maybe it's easier
to let $x=pc$. Then:
\[
\sqrt{x^2 + (m_{\mu}c^2)^2} + x = m_{\pi}c^2.
\]

Let's denote $m_{\mu}c^2$ as $a$ and $m_{\pi}c^2$ as $M$. Then the equation
becomes:
\[
\sqrt{x^2+a^2}+x=M.
\]

Let's solve for $x$:
\[
\sqrt{x^2+a^2}=M-x.
\]

Square both sides:
\[
x^2+a^2=M^2-2Mx+x^2.
\]

Cancel $x^2$:
\[
a^2=M^2-2Mx.
\]

Then:
\[
2Mx=M^2-a^2,
\]
and
\[
x=\frac{M^2-a^2}{2M}.
\]

But $x$ is $pc$, so:
\[
pc=\frac{M^2-a^2}{2M}.
\]

Therefore, the kinetic energy of the neutrino is
\[
E_{\nu}=\frac{(pc)^2}{2m_{\nu}c^2}
\]
but since $m_{\nu}\approx0$, this would be problematic. Wait, no. For a
massless particle, the energy is $E_{\nu}=pc$. So
\[
KE_{\nu}
=
E_{\nu}-m_{\nu}c^2
\approx
E_{\nu}
=
pc.
\]

But for the muon,
\[
KE_{\mu}
=
E_{\mu}-m_{\mu}c^2
=
\sqrt{(pc)^2+(m_{\mu}c^2)^2}-m_{\mu}c^2.
\]

We can substitute $pc$ from above:
\[
pc=\frac{M^2-a^2}{2M}.
\]

So
\[
E_{\mu}
=
\sqrt{
\left(
\frac{M^2-a^2}{2M}
\right)^2
+
a^2
}.
\]

Let me compute that:
\[
E_{\mu}
=
\sqrt{
\frac{(M^2-a^2)^2}{4M^2}
+
a^2
}.
\]

\[
=
\sqrt{
\frac{
M^4-2M^2a^2+a^4+4M^2a^2
}{
4M^2
}
}.
\]

\[
=
\sqrt{
\frac{
M^4+2M^2a^2+a^4
}{
4M^2
}
}.
\]

\[
=
\sqrt{
\frac{(M^2+a^2)^2}{4M^2}
}.
\]

\[
=
\frac{M^2+a^2}{2M}.
\]

Therefore,
\[
E_{\mu}=\frac{M^2+a^2}{2M}.
\]

So
\[
KE_{\mu}
=
E_{\mu}-a
=
\frac{M^2+a^2}{2M}-a
=
\frac{M^2+a^2-2Ma}{2M}
=
\frac{(M-a)^2}{2M}.
\]

Similarly,
\[
KE_{\nu}
=
pc
=
\frac{M^2-a^2}{2M}.
\]

So let's compute these values.

Given
\[
M=139.6\,\mathrm{MeV}/c^2,
\qquad
a=105.7\,\mathrm{MeV}/c^2.
\]

First compute $KE_{\nu}$:
\[
KE_{\nu}
=
\frac{M^2-a^2}{2M}
=
\frac{(139.6)^2-(105.7)^2}{2\times139.6}.
\]

Calculate numerator:
\[
139.6^2=19488.16,
\]
\[
105.7^2=11174.49.
\]

Difference:
\[
19488.16-11174.49=8313.67.
\]

Then
\[
KE_{\nu}
=
\frac{8313.67}{2\times139.6}
=
\frac{8313.67}{279.2}
\approx
29.8\,\mathrm{MeV}.
\]

Then
\[
KE_{\mu}
=
\frac{(M-a)^2}{2M}
=
\frac{(139.6-105.7)^2}{2\times139.6}.
\]

Compute
\[
139.6-105.7=33.9.
\]

Square of $33.9$:
\[
33.9^2=1149.21.
\]

So
\[
KE_{\mu}
=
\frac{1149.21}{279.2}
\approx
4.12\,\mathrm{MeV}.
\]

So $KE_{\mu}\approx4.12$ MeV and $KE_{\nu}\approx29.8$ MeV, which
corresponds to option B.

Therefore, the answer should be B.

\noindent\texttt{</think>}

\textbf{Step-by-Step Explanation:}

\begin{enumerate}
    \item \textbf{Conservation Laws.}
    Energy conservation implies that the rest energy of the stationary
    $\pi^+$ decays into the total energy (rest + kinetic) of $\mu^+$ and
    $\nu$. Momentum conservation implies that the momenta of $\mu^+$ and
    $\nu$ must cancel, with equal magnitudes and opposite directions.

    \item \textbf{Assumptions.}
    The neutrino $\nu$ is massless ($m_{\nu}\approx0$), so its energy equals
    its kinetic energy:
    \[
    KE_{\nu}=E_{\nu}=pc.
    \]

    \item \textbf{Relativistic Energy-Momentum Relation.}
    For $\mu^+$:
    \[
    E_{\mu}
    =
    \sqrt{(pc)^2+(m_{\mu}c^2)^2}.
    \]
    For $\nu$:
    \[
    E_{\nu}=pc.
    \]

    \item \textbf{Energy Conservation Equation.}
    \[
    m_{\pi}c^2
    =
    E_{\mu}+E_{\nu}
    =
    \sqrt{(pc)^2+(m_{\mu}c^2)^2}
    +
    pc.
    \]

    Let $x=pc$, solve for $x$:
    \[
    \sqrt{x^2+(m_{\mu}c^2)^2}
    +
    x
    =
    m_{\pi}c^2.
    \]

    Squaring and simplifying yields:
    \[
    x
    =
    \frac{
    (m_{\pi}c^2)^2-(m_{\mu}c^2)^2
    }{
    2m_{\pi}c^2
    }.
    \]

    \item \textbf{Kinetic Energy Calculations.}

    Neutrino KE:
    \[
    KE_{\nu}
    =
    x
    =
    \frac{139.6^2-105.7^2}{2\times139.6}
    \approx
    29.8\,\mathrm{MeV}.
    \]

    Muon KE:
    \[
    KE_{\mu}
    =
    \frac{(139.6-105.7)^2}{2\times139.6}
    \approx
    4.12\,\mathrm{MeV}.
    \]
\end{enumerate}

\textbf{Final Answer:}
\[
\boxed{B}
\qquad
4.12\,\mathrm{MeV},\ 29.8\,\mathrm{MeV}.
\]
\end{tcolorbox}

\textbf{Collapsed --- Tragic and Chaotic Reasoning} For collapsed models during the RLVR training, although the model repeatedly reaches the correct solution, it fails to terminate and continuously reopens already resolved subproblems. The trajectory cycles through redundant re-derivations, irrelevant alternatives, and increasingly peripheral considerations (e.g., helicity, neutrino mass, and weak-interaction details) without contributing new information to the solution. This severe loss of termination and reasoning focus illustrates a characteristic late-stage degeneration: the model retains local problem-solving ability, yet can no longer maintain a stable global reasoning trajectory.

\begin{tcolorbox}[
    enhanced,                         
    breakable,
    title=Trajectory,               
    colback=cyan!5!white,             
    colframe=black,                   
    colbacktitle=black,               
    coltitle=white,                   
    fonttitle=\bfseries\sffamily,     
    attach boxed title to top left={   
        xshift=3mm, 
        yshift=-2mm, 
        yshifttext=-1mm 
    },
    boxed title style={                
        sharp corners, 
        size=small, 
        colframe=black
    },
    arc=2mm,                          
    boxrule=0.8pt,                    
]
\noindent\texttt{<think>}

Okay, so I need to find the kinetic energy (KE) of the product particles,
$\mu^+$ and $\nu$, from the decay of a stationary $\pi^+$ particle.
The rest masses given are $139.6$ MeV for $\pi^+$ and $105.7$ MeV for
$\mu^+$. The options are A to D with different KE values for each particle.
Let me think step by step.

First, I remember that in nuclear or particle physics reactions, conservation
laws play a big role. Since the $\pi^+$ is stationary, the total energy and
momentum before and after the decay must be conserved. The rest mass energy of
$\pi^+$ will convert into the rest mass energies of $\mu^+$ and $\nu$, plus
their kinetic energies.

But wait, the problem mentions ``rest mass of Pi+ \& mu+ is 139.6 MeV \&
105.7 MeV respectively.'' Does that mean the rest mass of $\nu$ is not given?
Hmm, maybe the decay is
\[
\pi^+ \rightarrow \mu^+ + \nu.
\]
So, this must be a two-body decay because it's specified as product particles.
If it's a two-body decay, then the momentum of both $\mu^+$ and $\nu$ should
be equal in magnitude but opposite in direction to conserve momentum. Since
the parent is stationary, the total momentum before decay is zero, so the
momenta of the two products must cancel each other.

In a two-body decay, each particle's kinetic energy can be found using the
conservation of energy and momentum. The total energy released (the Q-value)
would be the difference between the rest mass of the parent and the sum of the
rest masses of the products. Wait, but the rest mass of $\nu$ isn't given
here. That's a problem. Wait, maybe $\nu$ is a known particle? Like in the
case of the pion decay, which is usually $\mu^+$ and a neutrino
($\nu_\mu$). So in reality, the neutrino has a very small mass, but in this
problem, maybe they're giving the rest masses of the $\mu^+$ and $\pi^+$ but
not $\nu$? That can't be right because energy must be conserved.

Wait, perhaps there's a typo or misunderstanding. Let me check the problem
again. It says: ``Find KE of product particles in, Pi(+) = mu(+) + nu here
Pi(+) is stationary. Rest mass of Pi(+)\& mu(+)=139.6 MeV \& 105.7 MeV
respectively.'' So the rest mass of $\nu$ isn't provided. That's confusing because
we need to know the rest mass of all products to calculate the energy available
for kinetic energy.

Wait, maybe in this problem, the $\nu$ is assumed to be massless, like a
neutrino, so its rest mass is negligible. But even so, in reality, the
neutrino's mass is very small, but if we take it as zero, then the entire
Q-value would go into the kinetic energies of $\mu^+$ and $\nu$. Let me try
that approach.

So, the rest mass energy of $\pi^+$ is $139.6$ MeV. If $\nu$ is massless, its
rest mass is $0$, so the total energy available for kinetic energy would be
$139.6$ MeV. Then, the sum of the kinetic energies of $\mu^+$ and $\nu$ would
be $139.6$ MeV. But in a two-body decay, even if one particle is massless, the
energies can be calculated using relativistic energy-momentum conservation.

Wait, but in reality, if $\nu$ is massless, then its energy would be equal to
its momentum times the speed of light, since
\[
E=pc
\]
for massless particles. However, for the $\mu^+$, which has mass, its energy
would be
\[
E=\gamma m_\mu c^2,
\]
where $\gamma$ is the Lorentz factor.

But in a two-body decay, even if one particle is massless, the calculation can
be done using the Lorentz boost. Let's recall the formula for two-body decays.
In general, for a decay
\[
A\rightarrow B+C,
\]
the energies of $B$ and $C$ can be found using
\[
E_B=
\frac{m_A^2+m_B^2-m_C^2}{2m_A}c^2,
\]
and
\[
E_C=
\frac{m_A^2+m_C^2-m_B^2}{2m_A}c^2.
\]

But this is when both $B$ and $C$ have rest masses. However, if $C$ is
massless, then $m_C=0$, so the formula simplifies. Let me check.

If $\nu$ is massless ($m_C=0$), then
\[
E_B=
\frac{m_A^2+m_B^2}{2m_A}c^2,
\]
and
\[
E_C=
\frac{m_A^2-m_B^2}{2m_A}c^2.
\]

But since $m_C$ is massless, its energy is equal to its momentum,
\[
E_C=p_Cc.
\]
Also, the magnitudes of the momenta of $B$ and $C$ must be equal because the
parent was at rest.

But the kinetic energy of $B$ would be
\[
E_B-m_Bc^2,
\]
and the kinetic energy of $C$ would be
\[
E_C-m_Cc^2=E_C.
\]

So, let's compute $E_B$ and $E_C$ using the formulas.

Given:
\[
m_A=139.6\,\mathrm{MeV}/c^2,
\]
\[
m_B=105.7\,\mathrm{MeV}/c^2,
\]
\[
m_C=0.
\]

Compute $E_B$:
\[
E_B=
\frac{m_A^2+m_B^2-m_C^2}{2m_A}c^2.
\]

But since $m_C=0$,
\[
E_B=
\frac{m_A^2+m_B^2}{2m_A}c^2.
\]

Similarly,
\[
E_C=
\frac{m_A^2-m_B^2}{2m_A}c^2.
\]

But let's plug the numbers.

First, compute $m_A^2$:
\[
m_A^2=(139.6)^2.
\]

$139.6^2$ is approximately $(140)^2=19600$, but more accurately:

\[
139.6\times139.6.
\]

$140\times140=19600$, subtract
\[
0.4\times140+0.4\times139.6
=
56+55.84.
\]
Wait, maybe better to compute directly:

\[
139.6^2
=
(130+9.6)^2
=
130^2+2(130)(9.6)+9.6^2
=
16900+2496+92.16
=
19488.16.
\]

Similarly,
\[
m_B^2=(105.7)^2.
\]

\[
105^2=11025,
\qquad
0.7^2=0.49,
\qquad
2(105)(0.7)=147,
\]
so
\[
(105+0.7)^2
=
11025+147+0.49
=
11172.49.
\]

Now compute $E_B$:
\[
E_B
=
\frac{19488.16+11172.49}{2(139.6)}
=
\frac{30660.65}{279.2}
\approx109.94\,\mathrm{MeV}.
\]

Then KE of $\mu^+$ is
\[
E_B-m_B
=
109.94-105.7
\approx4.24\,\mathrm{MeV}.
\]

Now compute $E_C$:
\[
E_C
=
\frac{19488.16-11172.49}{2(139.6)}
=
\frac{8315.67}{279.2}
\approx29.84\,\mathrm{MeV}.
\]

Since $E_C$ is the energy of the neutrino, which is kinetic energy because its
rest mass is zero.

So the KE of $\mu^+$ is approx $4.24$ MeV and KE of $\nu$ is approx
$29.84$ MeV. Looking at the options:

\[
\text{B.}\qquad
4.12\,\mathrm{MeV},\quad29.8\,\mathrm{MeV}.
\]

That's very close. The slight difference might be due to rounding during
calculations. Let me check the exact calculation.

Wait, alternatively, maybe I should not subtract the rest mass yet, because
the formula for $E_B$ already gives the energy in the center-of-mass frame,
but the pion is stationary, so this is also the lab frame. Yes, so I should
subtract it. But wait, if the neutrino is massless, then perhaps the muon
kinetic energy can instead be obtained directly from momentum conservation.

Let me recompute.

We have
\[
m_A^2-m_B^2
=
19488.16-11172.49
=
8315.67,
\]
so the neutrino energy should be
\[
E_C=\frac{8315.67}{279.2}.
\]

This is about
\[
29.8\,\mathrm{MeV}.
\]

Then the muon's total energy should be
\[
139.6-29.8=109.8\,\mathrm{MeV},
\]
so its kinetic energy is
\[
109.8-105.7=4.1\,\mathrm{MeV}.
\]

So that seems to give B.

Wait, alternatively, if I use the total released energy first,
\[
139.6-105.7=33.9\,\mathrm{MeV},
\]
and distribute it according to momentum, because the two particles have equal
momentum but different masses, the neutrino should receive more energy. This
is consistent.

But let me verify whether
\[
29.8+4.12
\]
is exactly $33.9$.

\[
29.8+4.12=33.92.
\]

That's slightly larger than $33.9$. Maybe this is rounding. Wait, but if it is
larger, perhaps one of the numbers is not kinetic energy but total energy? No,
$29.8$ is neutrino energy, which is kinetic energy. $4.12$ is muon kinetic
energy.

Alternatively, perhaps the rest mass values are rounded, which would explain
the $0.02$ MeV discrepancy.

Wait, let me calculate
\[
\frac{8315.67}{279.2}
\]
more exactly.

\[
279.2\times29.8=8320.16.
\]

But $8320.16$ is larger than $8315.67$ by $4.49$. So actually $29.8$ is not
exactly correct.

Maybe it is $29.78$.

\[
279.2\times29.78
=
279.2\times29
+
279.2\times0.78
\]
\[
=
8096.8+217.776
=
8314.576.
\]

So about $29.784$.

Then the muon KE would be
\[
33.9-29.784=4.116\,\mathrm{MeV}.
\]

Yes, that gives B.

Wait, alternatively, maybe I made an error because the kinetic energies don't
have to sum to $33.9$ if momentum also contributes separately. But momentum is
already included in kinetic energy, so they should.

Wait.

Let me derive it another way to make sure.

The muon momentum is equal to the neutrino momentum:
\[
p_\mu=p_\nu.
\]

For the neutrino,
\[
E_\nu=pc.
\]

For the muon,
\[
KE_\mu
=
\sqrt{p^2c^2+m_\mu^2c^4}
-
m_\mu c^2.
\]

Then
\[
KE_\mu+E_\nu=33.9.
\]

So
\[
\sqrt{p^2c^2+m_\mu^2c^4}
-
m_\mu c^2
+
pc
=
33.9.
\]

Adding $m_\mu c^2$,
\[
\sqrt{p^2c^2+m_\mu^2c^4}+pc=139.6.
\]

This is the same equation as before.

So B.

Wait, alternatively, the problem may expect a nonrelativistic approximation
for the muon because its kinetic energy is only about $4$ MeV compared with
$105.7$ MeV rest energy. Then
\[
KE_\mu=\frac{p^2}{2m_\mu},
\]
while for the neutrino
\[
E_\nu=pc.
\]

If I combine these, perhaps I get slightly different values.

Let
\[
x=pc.
\]

Then
\[
KE_\mu
=
\frac{x^2}{2m_\mu c^2},
\]
and
\[
x+\frac{x^2}{2(105.7)}=33.9.
\]

This is a quadratic.

Wait, but if I solve this approximation, it should also give around $29$ MeV
for $x$.

Let me see:
\[
x^2+211.4x-2(105.7)(33.9)=0.
\]

$2(105.7)(33.9)$ is approximately $7166$?

\[
105.7\times33.9
=
105.7(34-0.1)
=
3593.8-10.57
=
3583.23.
\]

Twice is
\[
7166.46.
\]

So
\[
x^2+211.4x-7166.46=0.
\]

Then
\[
x=
\frac{-211.4+\sqrt{211.4^2+4(7166.46)}}{2}.
\]

\[
211.4^2\approx44689.96.
\]

\[
4(7166.46)=28665.84.
\]

Sum:
\[
73355.8.
\]

The square root is about
\[
270.8.
\]

Then
\[
x=\frac{59.4}{2}=29.7.
\]

Again, near B.

Wait, alternatively, if the neutrino has a tiny but nonzero rest mass, then
strictly $E_\nu$ is not $pc$. But the correction is negligible. Unless the
question is testing that? No, the options differ by several MeV, so probably
not.

But wait, the wording says ``rest mass of Pi+ \& mu+ is 139.6 MeV \& 105.7
MeV respectively.'' They give masses in MeV rather than $\mathrm{MeV}/c^2$,
which usually means natural units $c=1$. Then perhaps I should consistently
set $c=1$.

Then
\[
E_\mu^2=p^2+m_\mu^2,
\]
and
\[
E_\nu=p.
\]

Then
\[
\sqrt{p^2+105.7^2}+p=139.6.
\]

We already solved this.

Wait, alternatively, I can isolate $p$:
\[
\sqrt{p^2+105.7^2}=139.6-p.
\]

Squaring:
\[
p^2+105.7^2
=
139.6^2-279.2p+p^2.
\]

So
\[
279.2p=139.6^2-105.7^2.
\]

Yes.

But wait, I previously got $8315.67$ for this difference. Let me check again
because maybe this is where the discrepancy comes from.

\[
139.6^2=19488.16.
\]

\[
105.7^2=11172.49.
\]

Subtract:
\[
19488.16-11172.49=8315.67.
\]

Yes.

Then
\[
p=\frac{8315.67}{279.2}.
\]

Wait, perhaps I can simplify algebraically:
\[
M^2-m^2=(M-m)(M+m)
=
33.9\times245.3.
\]

\[
33.9\times245.3.
\]

\[
245.3\times30=7359.
\]

\[
245.3\times3=735.9.
\]

\[
245.3\times0.9=220.77.
\]

Total:
\[
8315.67.
\]

Yes.

Then divide by $279.2$.

Wait,
\[
279.2=2\times139.6.
\]

So
\[
p=\frac{33.9\times245.3}{279.2}.
\]

The ratio
\[
\frac{245.3}{279.2}
\]
is about $0.8786$.

\[
33.9\times0.8786
\]
is about $29.78$.

So yes.

Wait, alternatively, maybe the kinetic energy of the muon is not simply
$33.9-29.78$ because the $33.9$ itself was obtained by subtracting rest
energies, but yes, that is exactly what kinetic energy means.

So it is B.

Wait.

But maybe I should verify option C too in case the ordering is reversed.

\[
\text{C.}\qquad29.8\,\mathrm{MeV},\quad4.12\,\mathrm{MeV}.
\]

If the first number corresponds to $\mu^+$ and the second to $\nu$, then C
would assign most energy to the massive muon. But equal momentum means the
lighter particle receives larger kinetic energy, so B is correct.

Wait, alternatively, for equal momentum, the massive particle has larger
total energy because of rest mass, but smaller kinetic energy. So yes, B.

Wait, but if the question asks ``KE of product particles, mu+ and nu'' perhaps
the order isn't explicitly stated? It probably means $\mu^+$ first and $\nu$
second.

Wait, alternatively, perhaps the decay notation
\[
\pi^+=\mu^++\nu
\]
implies the first KE belongs to $\pi$? No, $\pi$ is the parent.

Let me make sure.

$\pi^+$ stationary.

\[
\pi^+\rightarrow\mu^++\nu.
\]

Momentum:
\[
0=p_\mu+p_\nu.
\]

So
\[
p_\mu=-p_\nu.
\]

Energy:
\[
139.6=\sqrt{p^2+105.7^2}+p.
\]

Then
\[
p=29.78.
\]

Then mu's energy:
\[
109.82.
\]

Then
\[
KE_\mu=4.12.
\]

So B.

Wait, alternatively, perhaps
\[
109.82-105.7=4.12,
\]
yes.

But wait,
\[
109.82+29.78=139.60.
\]

Yes.

So that closes.

Wait, no, alternatively, there may be another subtlety: the neutrino is
actually an antineutrino in the decay
\[
\pi^+\rightarrow\mu^++\nu_\mu?
\]
Actually charge and lepton number conservation would require $\nu_\mu$ or
maybe $\nu$---wait, $\pi^+$ has no lepton number, $\mu^+$ has lepton number
$-1$, so the neutrino should have $+1$, so $\nu_\mu$. Yes.

But this doesn't affect the energy.

Wait, unless helicity constraints affect the kinematics. No, two-body
kinematics fixes the energies independently of spin.

Alternatively, perhaps because the pion has spin zero and the neutrino is
left-handed, the angular momentum configuration constrains the muon helicity,
but again not the kinetic energy.

Wait, but if the neutrino is exactly massless, helicity is fixed, while the
antimuon---

Wait, I might be overcomplicating this. The answer is B.

But alternatively, let me verify the pion decay formula once more because
maybe I swapped the signs:
\[
E_\mu
=
\frac{m_\pi^2+m_\mu^2}{2m_\pi}.
\]

\[
E_\nu
=
\frac{m_\pi^2-m_\mu^2}{2m_\pi}.
\]

Yes.

Then
\[
KE_\mu=E_\mu-m_\mu.
\]

But wait, if I rewrite $KE_\mu$,
\[
KE_\mu
=
\frac{m_\pi^2+m_\mu^2-2m_\pi m_\mu}{2m_\pi}
=
\frac{(m_\pi-m_\mu)^2}{2m_\pi}.
\]

Yes.

So
\[
\frac{33.9^2}{279.2}.
\]

\[
33.9^2=1149.21.
\]

\[
\frac{1149.21}{279.2}.
\]

\[
279.2\times4=1116.8.
\]

Remainder:
\[
32.41.
\]

\[
\frac{32.41}{279.2}=0.1161.
\]

So
\[
4.1161.
\]

Yes.

Wait, alternatively,
\[
4.1161+29.7839=33.9000.
\]

So exactly.

Wait.

Wait, alternatively, if I retain the neutrino mass as $\epsilon$, then the
expressions become
\[
E_\mu
=
\frac{m_\pi^2+m_\mu^2-m_\nu^2}{2m_\pi},
\]
and
\[
E_\nu
=
\frac{m_\pi^2+m_\nu^2-m_\mu^2}{2m_\pi}.
\]

Then, in the limit
\[
m_\nu\rightarrow0
\]
we recover this.

But what if the problem intentionally omitted the neutrino mass because we're
supposed to derive it from the answer choices?

No, that doesn't make sense.

Wait, alternatively, maybe ``nu'' means something else, not neutrino? The
symbol $\nu$ could be frequency $\nu$, but in a particle reaction it must be
neutrino.

Unless this is not a decay but some equation involving frequency?

\[
\pi^+=\mu^++\nu.
\]

No, definitely particle decay.

Wait.

The option values: $4.12$ and $29.8$.

Their sum is
\[
33.92,
\]
matching the Q-value.

So B.

Wait, alternatively---

If the parent is stationary, maybe there is also recoil energy of the parent?

No, the parent no longer exists after decay.

Wait, alternatively, perhaps the ``stationary'' condition means zero kinetic
energy but not zero momentum due to quantum uncertainty?

No, not relevant.

Wait, alternatively, perhaps I should use classical momentum conservation:
\[
p_\mu=p_\nu,
\]
\[
\sqrt{2m_\mu KE_\mu}
=
\frac{E_\nu}{c},
\]
then solve.

But that gives the approximation we already did.

Wait, but if I use the classical relation for the muon, I get something close
to $4.2$, which could perhaps explain option B.

Wait.

Alternatively, perhaps the intended answer is obtained simply by conservation
of energy and choosing the option whose sum is $33.9$, without momentum
conservation.

Then B is uniquely consistent.

But that would be too easy. Unless the problem is multiple choice and expects
that shortcut.

Wait, if only B sums to $33.9$, yes.

But perhaps D also approximately sums to $33.9$?

I don't have the exact options in front of me except A to D with different
values. Wait, I do have B. I don't remember C and D now.

Maybe I should reconstruct them.

No, that's unnecessary.

Wait, alternatively, if another option also sums to $33.9$, momentum
conservation distinguishes them.

But we already know B.

Wait.

Alternatively---

The neutrino energy being $29.8$ MeV corresponds to a momentum of
$29.8\,\mathrm{MeV}/c$. The muon's kinetic energy using the nonrelativistic
formula would then be
\[
\frac{p^2}{2m}
=
\frac{29.8^2}{2(105.7)}
=
\frac{888.04}{211.4}
\approx4.20\,\mathrm{MeV},
\]
not $4.12$.

The difference is a relativistic correction.

Wait, that's interesting.

Then perhaps option $4.12$ specifically signals that the fully relativistic
calculation is expected.

Yes.

Wait, alternatively, if I take $p=29.8$ exactly, the relativistic muon kinetic
energy is
\[
\sqrt{29.8^2+105.7^2}-105.7.
\]

\[
29.8^2=888.04.
\]

\[
105.7^2=11172.49.
\]

Sum:
\[
12060.53.
\]

The square root is maybe
\[
109.82.
\]

Minus $105.7$:
\[
4.12.
\]

Yes.

Wait, but
\[
109.82^2
\]
is approximately $12060.4$.

So fine.

Wait, alternatively, perhaps I should calculate it one more way.

Using $\gamma$:
\[
\gamma
=
\frac{E_\mu}{m_\mu}
=
\frac{109.82}{105.7}
\approx1.039.
\]

Then
\[
KE
=
(\gamma-1)m_\mu
\approx
0.039(105.7)
\approx4.12.
\]

So again.

Wait.

Alternatively,
\[
\beta
=
\frac{p}{E}
=
\frac{29.8}{109.82}
\approx0.271.
\]

Then
\[
\gamma
=
\frac{1}{\sqrt{1-\beta^2}}
\approx
\frac{1}{\sqrt{0.9266}}
\approx1.0388.
\]

Again.

Wait, alternatively, maybe the neutrino energy should be $29.79$, and then
$\beta$---

Wait.

The question only asks for kinetic energies.

So perhaps I should stop.

Wait, alternatively, there is still the issue that neutrinos have nonzero
mass. If the neutrino mass is less than $1$ eV, then relative to
$29.8$ MeV, the correction is of order
\[
10^{-15},
\]
so irrelevant.

Wait, alternatively, if the neutrino mass were not negligible, then---

Wait.

Maybe the use of ``rest mass 139.6 MeV'' is technically rest energy, not rest
mass. But in natural units.

Wait, alternatively---

\[
139.6.
\]

\[
105.7.
\]

Difference:
\[
33.9.
\]

Ratio.

Momentum.

Wait.

Could there be a factor of $c$ somewhere?

If masses are in MeV, they really mean $\mathrm{MeV}/c^2$.

Then in the formula
\[
E_B
=
\frac{m_A^2+m_B^2}{2m_A}c^2.
\]

But if I plug numerical masses in $\mathrm{MeV}/c^2$, units become
\[
\frac{\mathrm{MeV}^2/c^4}{\mathrm{MeV}/c^2}c^2
=
\mathrm{MeV}.
\]

Yes.

Wait, alternatively, if they are already giving rest energies in MeV, then
$c=1$.

Same result.

Wait.

Alternatively, perhaps the muon is positive, so it is an antimuon, but its
mass is the same.

Wait.

Alternatively, perhaps electromagnetic corrections---

Wait.

No.

Wait, alternatively, the pion itself is a pseudoscalar meson, and the decay is
weak. Does weak decay change the energy relation?

No.

Wait, alternatively---

The neutrino momentum is fixed.

The muon momentum is fixed.

The energies are fixed.

But perhaps kinetic energy for a massless particle isn't conventionally called
kinetic energy because it has no rest frame.

Wait, that's subtle. For massless particles, total energy equals what would
effectively be called kinetic energy, but strictly the decomposition
\[
E=mc^2+K
\]
doesn't work in the same Newtonian sense because there is no rest frame.

Wait, alternatively, perhaps the exam uses ``KE of product particles''
loosely.

Then B.

Wait.

Alternatively---

If the kinetic energy of the neutrino is not defined, then the question is
ill-posed.

But clearly the options treat it as energy.

Wait.

Maybe the intended answer is not B because of terminology?

No, impossible.

Wait.

Alternatively, perhaps ``nu'' could be a neutron? Usually neutron is $n$, not
$\nu$.

If it were a neutron, the mass would be $939$ MeV, impossible because the parent pionis 
only $139.6$ MeV.

So neutrino.

Wait.

Alternatively, perhaps $\nu$ is an antineutrino and---

Wait.

The momentum equation.

I should write it again:
\[
0=p_\mu+p_\nu.
\]

Therefore
\[
|p_\mu|=|p_\nu|.
\]

Energy:
\[
m_\pi
=
\sqrt{p^2+m_\mu^2}
+
\sqrt{p^2+m_\nu^2}.
\]

With
\[
m_\nu=0,
\]
\[
m_\pi
=
\sqrt{p^2+m_\mu^2}
+
p.
\]

Then
\[
p
=
\frac{m_\pi^2-m_\mu^2}{2m_\pi}.
\]

Yes.

Wait.

Alternatively, square again:
\[
m_\pi-p
=
\sqrt{p^2+m_\mu^2}.
\]

\[
m_\pi^2-2m_\pi p+p^2
=
p^2+m_\mu^2.
\]

\[
m_\pi^2-m_\mu^2
=
2m_\pi p.
\]

Yes.

Wait.

Alternatively---

Maybe there is a way to derive kinetic energy directly from invariant mass:
\[
(p_\pi-p_\mu)^2=0.
\]

Then
\[
p_\pi^2+p_\mu^2-2p_\pi\cdot p_\mu=0.
\]

\[
m_\pi^2+m_\mu^2-2m_\pi E_\mu=0.
\]

Then
\[
E_\mu
=
\frac{m_\pi^2+m_\mu^2}{2m_\pi}.
\]

Same thing.

Wait.

Alternatively, use four-vectors:
\[
p_\pi^\mu=(m_\pi,0).
\]

\[
p_\mu^\mu=(E_\mu,p).
\]

\[
p_\nu^\mu=(p,-p).
\]

Wait, if $p_\nu$ is $(E,-p)$,
\[
E=p.
\]

Then conservation.

Yes.

Wait.

Alternatively---

Maybe I should calculate the exact decimal again because the options say
$29.8$.

\[
\frac{8315.67}{279.2}.
\]

Let's long
\end{tcolorbox}

\end{document}